\documentclass[11pt]{article}

\usepackage{float}
\usepackage[margin=1in]{geometry}
\usepackage{microtype}
\usepackage{amsmath}
\usepackage{amssymb}
\usepackage{amsfonts}
\usepackage{mathtools}
\usepackage{amsthm}
\usepackage{graphicx}
\usepackage{subcaption}
\usepackage{booktabs}
\usepackage[normalem]{ulem}
\usepackage[sort&compress,numbers]{natbib}
\setcitestyle{numbers,square}
\usepackage[colorlinks,
            linkcolor=red,
            anchorcolor=blue,
            citecolor=magenta]{hyperref}
\usepackage[capitalize,noabbrev]{cleveref}

\graphicspath{{./}}

\newcommand{\R}{\mathbb{R}}
\newcommand{\E}{\mathbb{E}}

\newcommand{\dd}{\mathrm{d}}

\newcommand{\cS}{\mathcal{S}}
\newcommand{\cT}{\mathcal{T}}

\newcommand{\cD}{\mathcal{D}}

\newcommand{\KL}{\mathrm{KL}}

\newcommand{\cF}{\mathcal{F}}

\newcommand{\cH}{\mathcal{H}}
\newcommand{\N}{\mathbb{N}}

\newcommand{\PP}{\mathbb{P}}

\newcommand{\argmin}{\operatorname*{arg\,min}}

\DeclareMathOperator{\supp}{supp}

\theoremstyle{plain}
\newtheorem{theorem}{Theorem}[section]
\newtheorem{proposition}[theorem]{Proposition}
\newtheorem{lemma}[theorem]{Lemma}

\theoremstyle{definition}

\newtheorem{assumption}[theorem]{Assumption}
\theoremstyle{remark}
\newtheorem{remark}[theorem]{Remark}

\title{Sobolev Regularized Score Difference Estimation in Diffusion Models}
\author{Chenghan Xie\thanks{Department of Management Science and Engineering, Stanford University. Emails: \texttt{\{xxcchan, jose.blanchet, renyuanxu\}@stanford.edu}.}
\and
Jose Blanchet\footnotemark[1]
\and
Renyuan Xu\footnotemark[1]
}
\date{\today}
\newcommand{\blfootnotetext}[1]{%
  \begingroup
  \renewcommand{\thefootnote}{}\footnotetext{#1}%
  \endgroup
}
\newcommand{\proofstep}[1]{\par\medskip\noindent\uline{#1}}

\begin{document}
\maketitle
\allowdisplaybreaks
\blfootnotetext{Code is available at the \href{https://github.com/chenghands-on/Sobolev-Regularized-Score-Difference-Estimator}{GitHub repo}. Paper is accepted at ICML 2026.}

\begin{abstract}
Estimating the difference of two Stein's score functions is a fundamental problem in generative modeling. In particular, score differences arise naturally in transfer learning, where the score difference provides the mechanism for adapting a pre-trained model to a new target distribution, and in diffusion model-based post-training methods such as discriminator guidance. Existing estimators for score differences in these settings either lack of statistical consistency or are difficult to scale up in high-dimensions. We propose a statistically consistent and scalable estimator for score differences based on Sobolev regularization, which plays a crucial role in ensuring consistency and stablizing the training in the small-sample regime. Mathematically, we establish a convergence rate of $\tilde{\mathcal{O}}(n^{-\frac{s-1}{d+2s-2}})$ where $d$ is the  dimension and $s$ denotes the smoothness of the underlying densities, and provide a minimax lower bound of $\tilde{\Omega}(n^{-\frac{2(s-1)}{d+2s}})$ (in mean-squared error). Empirically, our estimator exhibits significantly improved stability in small-sample regimes compared to existing methods. We demonstrate its effectiveness on real-world tasks, including transfer learning for ECG signal generation, where it substantially outperforms non-regularized score difference estimators in downstream classification performance.
\end{abstract}
\section{Introduction}
The difference of Stein's score functions is defined as the gradient of the log-density ratio $\nabla \log q(\cdot) - \nabla \log p(\cdot)$ between a target distribution $q$ and a source distribution $p$. Conceptually, this difference represents the driving force required to transport samples from a source distribution $p$ to a target distribution $q$. As a result, the score difference emerges as a fundamental primitive in transfer learning for modern generative modeling \cite{liu2023minimizing,ouyang2024transfer,wang2024bridging}.

This estimation problem is central in post-training of diffusion models, which adapts pre-trained generative models to align with human preferences, structural constraints, or downstream tasks.
A wide range of approaches have been proposed, including RLHF
\cite{black2023training,fan2023reinforcement}, stochastic control–based formulations \cite{tang2024fine,han2024stochastic,uehara2024fine}, and classifier-guided or conditioning-based methods.
Notably, most of these can be unified as add-on mechanisms to pre-trained dynamics:
\begin{equation}
    \mathrm{d} Y_t = s_\theta(t,Y_t)\mathrm{d} t - h_\eta(t,Y_t)\mathrm{d} t + \sigma(t)\mathrm{d} W_t,
    \label{eq:controlled_dynamics}
\end{equation}
where $s_\theta(t,\cdot)$ is the pre-trained score function approximating $\nabla \log p_t(\cdot)$, and $h_\eta(t,\cdot)$ is an additive control term.
In frameworks such as discriminator guidance and conditional generation, via Doob's $h$-transform, the term $h_\eta(t,\cdot)$ approximates $\nabla \log q_t(\cdot)$, where $q_t$ is a target distribution carrying information from constraints, classifiers, or preference signals~\cite{DEFT,DP24,pidstrigach2025conditioning,howard2025control}. In diffusion models, post-training is a form of transfer learning, transferring pre-training knowledge to new data-generation tasks. For simplicity, we use “transfer learning” to refer to this broader setting throughout the paper.

In transfer learning, the target task typically has far fewer samples than the source task, rendering direct estimation of the target score $\nabla q_t$ unstable. Fortunately, because the target and source tasks are closely related, the density ratio $\frac{q_t}{p_t}$ often exhibits exploitable structure. This motivates directly estimating the density ratio and its gradient, since
\begin{eqnarray}
    \nabla \log q_t(\cdot)- \nabla \log p_t(\cdot) =  \frac{\nabla ({q_t}/{p_t})(\cdot)}{{q_t}/{p_t}(\cdot)}.
    \label{eq:score_diff_ratio}
\end{eqnarray}
Unlike estimating and differencing the two scores separately, this approach exploits shared geometry and is more structure-aware and data-efficient in low-sample regimes (see Figure~\ref{fig:velocity_field}).

Despite arising naturally in many settings, obtaining a reliable estimator for this score difference remains challenging.
The most prevalent large-scale approaches~\cite{ouyang2024transfer,kim2022refining} adopt a ``classify-then-differentiate'' pipeline: first estimating the log-density ratio via binary classification, and then differentiating the fitted model.
However, this approach lacks theoretical guarantees on gradient convergence.
As illustrated in Figure~\ref{fig:oscillating_gradient}, standard classifiers tend to overfit high-frequency but low-amplitude noise components in the data.
This overfitting yields small ratio errors (left) but catastrophic oscillations after differentiation (right).

We address this challenge via Sobolev regularization, which controls the smoothness of the log-density ratio gradient and prevents overfitting to spurious high-frequency noise.
With $n$ samples from both the source $p$ and target $q$, we prove a convergence rate of $\tilde{\mathcal{O}}(n^{-\frac{s-1}{d+2s-2}})$ for the score difference error (Theorem \ref{thm:classification-oracle}, \ref{thm: unbounded_convergence_nohop}), where $d$ is the dimension and $s$ is the smoothness parameter.
We also provide a minimax lower bound of $\tilde{\Omega}(n^{-\frac{2(s-1)}{d+2s}})$ (Theorem \ref{thm: minimax lower bound}), showing near-optimality.
Experiments verify these findings across diverse tasks, including joint domain adaptation via Wasserstein gradient flow (Sec \ref{subsec:WGF}) and transfer learning for diffusion models (Sec \ref{subsec:diffusion}), confirming that Sobolev regularization yields stable and consistent score difference estimators.
\begin{figure}[t]
    \centering
    \includegraphics[width=1.0\linewidth]{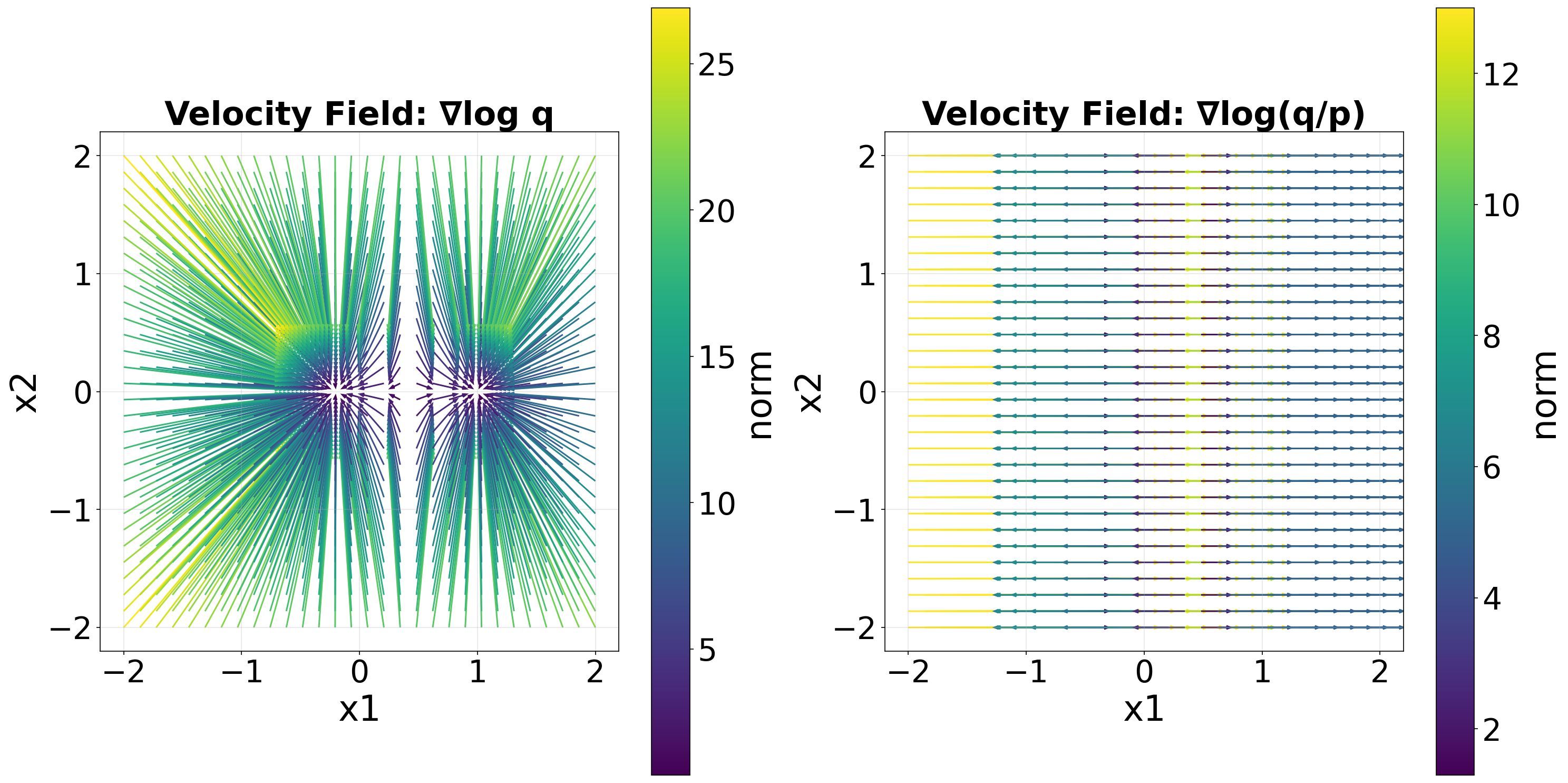}
    \caption{
    \textbf{Visualization of the score functions}. The score difference field (right) is markedly smoother and less variable than the absolute target score (left), demonstrating the structural benefit of directly estimating the score difference. }
    \label{fig:velocity_field} 
\end{figure}

\begin{figure}[t]
    \centering
    \includegraphics[width=1.0\linewidth]{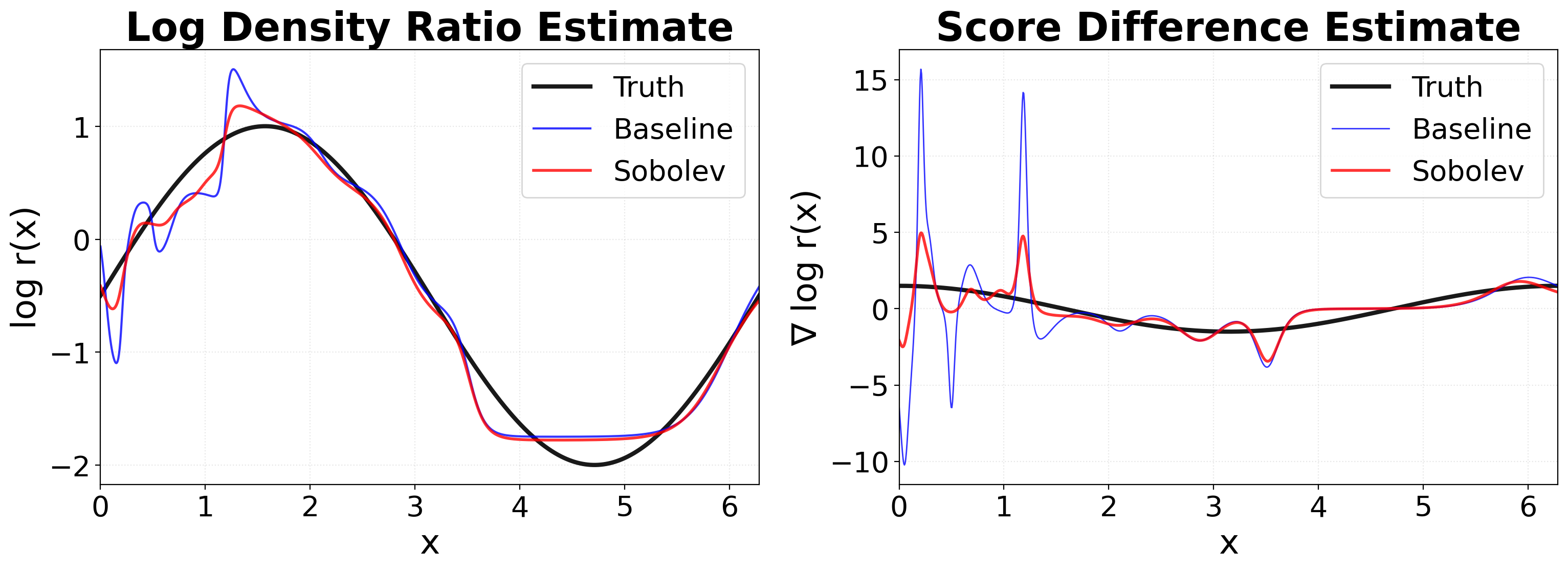}
    \caption{
    {\textbf{Impact of high-frequency noise.}
    Comparison of $\log r$ (left) and $\nabla \log r$ (right) estimates.
    Standard classification (blue) captures the function value but fails on the gradient due to overfitting.
    Our Sobolev-regularized method (red) enforces smoothness, recovering the gradient accurately by filtering out noise.}
}
    \label{fig:oscillating_gradient}

\end{figure}

\paragraph{Closely related literature.}
Our work is closely related to two lines of literature: score difference estimation and Sobolev regularization in machine learning.

Score difference estimation underlies a broad class of applications related to diffusion models. For example, \cite{ouyang2024transfer} propose using score difference estimation for transfer learning tasks in diffusion models, which is later refined by \cite{wang2024bridging} with an additional residual fine-tuning step. For post-training problems, the discriminator guidance framework of \cite{kim2022refining} highlights the central role of score difference estimation. In this approach, a discriminator is trained to distinguish samples from the target distribution and the pre-trained model, and its gradient provides an estimate of the score difference $\nabla \log q_t - \nabla \log p_t$, which is then injected as a control term into the diffusion dynamics. This framework was further extended to noisy data settings in \cite{cong2025guiding}. However, the above papers primarily emphasize methodological development and empirical validation, while theoretical convergence guarantees remain unestablished. On the other hand, several alternative approaches provide convergence guarantees but may still suffer from practical limitations such as scalability and numerical stability.
For example, assuming that the ground-truth score difference lies in a reproducing kernel Hilbert space (RKHS), \cite{srikanthaccelerated} derive a closed-form kernel-based solution. However, this approach incurs prohibitive computational costs, as it requires evaluating kernel distances against the entire training dataset at each inference step, rendering it impractical at scale. Similarly, \cite{liu2023minimizing} propose a consistent score difference estimator for Wasserstein gradient flow (WGF) methods, but their local kernel approach faces the same scalability challenges.  More recently, \cite{verine2025improving} address objective bias in discriminator guidance by adopting a score-matching objective in which the regression target is the residual error of a pre-trained score. However, without explicit regularity constraints, the resulting estimator can be dominated by variance, particularly in low-density regions or when the pre-trained score is already relatively accurate. In contrast to these approaches, our Sobolev-regularized score difference estimator retains the computational efficiency of classification-based frameworks, while offering enhanced consistency and robustness via the introduction of a Sobolev regularization term.

Highlighting the central role of  regularization, \cite{husaingeneralization} point out that discriminators in guidance frameworks must be well regularized to ensure generalization, a perspective that closely aligns with our work. Sobolev-type regularization, often realized through gradient penalties, has been widely studied in the literature on generative adversarial networks (GANs) \cite{lin2025diffusion, roth2017stabilizing, mescheder2018training}, whereas its role in diffusion models remains comparatively underexplored.  In these settings, the discriminator can be interpreted as estimating a time-varying density ratio. Moreover, the underlying motivations differ substantially. In those works, gradient penalties are introduced to promote dynamical stability by shifting the eigenvalues of the Jacobian and preventing oscillations around local Nash equilibria. In contrast, our motivation is rooted in statistical estimation: we employ Sobolev regularization to constrain the hypothesis space, ensuring that the estimated score difference remains accurate and robust to overfitting to high-frequency noise. Finally, beyond generative modeling, \cite{ding2025semi} study Sobolev-penalized deep semi-supervised regression for the joint estimation of a regression function and its gradient, leveraging a large set of unlabeled covariates to approximate the Sobolev penalty.

\section{Problem setup and proposed method}
\label{sec:setup-OU-tilde}
Our goal is to formalize transfer learning from a source   $p$ to a target
$q$, by estimating the score difference in a way that is structure-aware, data-efficient, and stable.

To focus on the essential building block of our framework, we first consider a bounded state space. The extension to unbounded domains and the inclusion of a time component, tailored to diffusion models, are deferred to Section~\ref{sec:generalization}, where the generalization is straightforward.

Let $[0,1]^d \subseteq\Omega \subseteq \R^d$, $d\in\N$ be a bounded and connected domain with sufficiently smooth boundary $\partial \Omega$. Let $P,Q \in \mathcal{P}(\Omega) $ be absolute continuous with respect to Lebesgue measure with densities $p,q$.
For any $x\in \Omega$ let
\[
  \rho(x):= \tfrac12 p(x) + \tfrac12 q(x), \qquad
  \mu(\dd x):= \rho(x)\,\dd x,
\]
and define the log-density ratio
\begin{eqnarray}\label{eq:f_star}
    f^\star(x):= \log \frac{q(x)}{p(x)}.
\end{eqnarray}
For a multi-index $\alpha=(\alpha_1,\ldots,\alpha_d)\in\mathbb{N}^{d}$,
denote $|\alpha|:=\alpha_1+\cdots+\alpha_d$ and
$
D^\alpha f:=
\frac{\partial^{|\alpha|} f}{\partial x_1^{\alpha_1}\cdots\partial x_d^{\alpha_d}}$,
interpreted in the weak (distributional) sense. For an integer $s\ge 0$, the weighted Sobolev space $H^s(\mu)$ is  given by
\[
H^s(\mu):=
\left\{
\begin{aligned}
f \in L^2(\mu)
\;\big|\;
& D^\alpha f \in L^2(\mu), \text{for all}\\
& \text{multi-indices } \alpha
\text{ with } |\alpha|\le s
\end{aligned}
\right\}.
\]
For the supremum norm, we denote the $L^\infty(\mu)$ norm as the essential supremum with respect to the measure $\mu$:
\[
\|f\|_{L^\infty(\mu)}:
= \inf \bigl\{ C \ge 0: \mu(\{x \in \Omega: |f(x)| > C\}) = 0 \bigr\}.
\]

Let $\mu_0$ be the uniform distribution on $\Omega$. We assume the following outstanding assumptions.

\begin{assumption}[Density functions]\label{assm: density}
Assume  $p$ and $q$ satisfy the following conditions:
\begin{enumerate}
\item $p,q\in C^2(\overline\Omega)$, and there exist constants \(0<c<C<\infty\) such that \(c\le p(x),q(x)\le C\) for all \(x\in\overline\Omega\).
\item The following Neumann boundary condition holds:
\[
\rho(x)\,\partial_n f^\star(x)=0
\qquad \text{for all }x\in\partial\Omega,
\]
where \(n(x)\) is the outward unit normal to \(\partial\Omega\).
\item Fix $s\geq2$. We assume $f^\star \in H^s(\mu_0)$.
\end{enumerate}
\end{assumption}
Under Assumption~\ref{assm: density}, we have $f^\star \in C^2(\Omega)$. Hence, there exists $M^\star>0$ such that
$$    \max \{\|f^\star\|_{L^\infty(\mu_0)},\| \nabla f^\star\|_{L^\infty(\mu_0)}\}\leq M^\star.$$
Motivated by the boundedness of the true density ratio
$f^\star$, we restrict the function class to ensure stability in the optimization. Specially, fix a constant $M \geq M^\star$, we define a bounded Sobolev subset of $H^1(\mu)$  as follows:
\begin{eqnarray}\label{eq:HM}
    \mathcal{H}_M:= \left\{ f \in H^1(\mu): \|f\|_{L^\infty(\mu)} \leq M \right\}.
\end{eqnarray}
By construction, this set is closed and convex in $H^1(\mu)$, and it contains the ground truth $f^\star$.
Since \(\mu\) and \(\mu_0\) are equivalent under Assumption~\ref{assm: density}, we keep the notation
\(\mathcal H_M\) when risks and norms are evaluated with respect to \(\mu\).

\paragraph{Learning objective.}
Our objective is to estimate the score difference
\[
   \nabla_x f^\star(x):= \nabla_x \log \frac{q(x)}{p(x)},
\]
given access only to finite samples from both source and target datasets:
\begin{eqnarray}\label{eq:data_set}
     \cS=\{X_{i}^{p}\}_{i=1}^{n}
      \stackrel{\text{i.i.d.}}{\sim} p,
\,\,
  \cT=\{X_{i}^{q}\}_{i=1}^{n}
      \stackrel{\text{i.i.d.}}{\sim} q.
\end{eqnarray}
When $p$ and $q$ share similar geometric structure, the vector field
$
   \nabla\log\frac{q(x)}{p(x)}
$
often concentrates on a low-dimensional manifold, making it easier to learn than the full score of $q$.  The standard approach first estimates the log-density ratio and then differentiates it directly \cite{ouyang2024transfer,kim2022refining}. Usually, the estimation of log-density ratio is trained via classification loss.

\paragraph{Binary classification reformulation.}
We introduce a label-augmented mixture model,
offering a convenient probabilistic representation for learning the log-density ratio. Specifically, denote by $\mu_{p,q}$ the joint law of $(X,Y)$ with
\begin{eqnarray}\label{eq: label ovservation model}
\begin{aligned}
    &  Y \sim\mathrm{Bernoulli}(1/2),\\
  &     X\mid(Y=1)\sim q,\,
  X\mid(Y=0)\sim p.
\end{aligned}
\end{eqnarray}
In such setting, the
posterior classification probability is, for $x\in \Omega$,
$
  \eta(x):= \PP(Y=1\mid X=x)
  = \frac{q(x)}{p(x)+q(x)}
$
The corresponding Bayes logit is $ \log\frac{\eta(x)}{1-\eta(x)}
  = \log\frac{q(x)}{p(x)} = f^\star(x).$
Therefore, if we let $\sigma(u):=\frac{1}{1+e^{-u}}$ and define the cross-entropy loss
\[
  \ell_{\rm CE}(y,u):= -y\log\sigma(u)-(1-y)\log(1-\sigma(u)),
\]
then the log-density ratio \(f^\star\) in \eqref{eq:f_star} minimizes the population cross-entropy risk
\[
  L_{\rm CE}(f):= \E_{(X,Y)\sim\mu_{p,q}}
     \big[\,\ell_{\rm CE}(Y,f(X))\,\big]
\]
among all measurable functions $f:\Omega\to\R$.
Using the fixed source--target design in \eqref{eq:data_set}, we augment
deterministic labels and write
$
\mathcal D:=
\{(X_i^p,0)\}_{i=1}^n
\cup
\{(X_i^q,1)\}_{i=1}^n
$.
The corresponding fixed source--target
empirical cross-entropy risk is
\[
  \widehat L_{\rm CE,\mathcal D}(f):=
  \frac{1}{2n}\sum_{i=1}^{n}\ell_{\rm CE}(0,f(X_i^p))
  +
  \frac{1}{2n}\sum_{i=1}^{n}\ell_{\rm CE}(1,f(X_i^q)).
\]

\paragraph{Functional class: Function class: clipped sparse neural networks with bounded gradients.}
We consider a sparse neural network class
\cite{schmidt2020nonparametric,suzuki2018adaptivity,farrell2021deep}
with the $\mathrm{ReLU}^3$ activation $\eta_3(x):=\max\{x^3,0\}$, applied componentwise for vector inputs. This activation is useful for representing spline-type approximants and is supported by the corresponding approximation theory \cite{lu2021machine}. Fix $M>M^\star$. Let $T_M:\mathbb R\to[-M,M]$ be a nondecreasing $C^1$ truncation map such that $|T_M(u)|\le M$ and $|T_M'(u)|\le1$ for all $u\in\mathbb R$, and such that $T_M(u)=u$ whenever $|u|\le M_0$, for some $M_0\in(M^\star,M)$. Thus $T_M$ clips only outside a range strictly containing the target range.

Given the depth $L$, width $W$, sparsity level $S$, and parameter bound $B$, define
\begin{eqnarray}\label{eq:network_class}
\begin{aligned}
    &\mathcal{F}_M(L,W,S,B):=
    \Bigg\{
    g(x)=T_M\Big[
    \big(\mathcal{W}^{(L)}\eta_3(\cdot)+b^{(L)}\big)
    \circ\cdots\circ
    \big(\mathcal{W}^{(1)}x+b^{(1)}\big)
    \Big],\\
    &\qquad \qquad
    \mathcal{W}^{(1)}\in\mathbb R^{W\times d},
    \mathcal{W}^{(2,\dots,L-1)}\in\mathbb R^{W\times W},
    \mathcal{W}^{(L)}\in\mathbb R^{1\times W},\, \sum_{l=1}^L
    \left(
    \|\mathcal W^{(l)}\|_0+\|b^{(l)}\|_0
    \right)
    \le S,\\
    &\qquad \qquad
    \, \max_l
    \|\mathcal W^{(l)}\|_\infty
    \vee
    \|b^{(l)}\|_\infty
    \le B,
    \,
    \|g\|_{L^\infty(\mu)}\le M,
    \|\nabla g\|_{L^\infty(\mu)}\le M
    \Bigg\}.
\end{aligned}
\end{eqnarray}
Here $\circ$ denotes function composition, $\|\cdot\|_0$ counts nonzero entries, and $\|\cdot\|_\infty$ denotes the entrywise maximum norm.

\paragraph{Proposed learning method.}
To control both function values and gradients, we define the Sobolev seminorm
\[
  \mathcal{R}(f):= \|\nabla f\|_{L^2(\mu)}^2,
\]
for any $ s\geq 1$ and $f\in H^s(\mu)$.  For $\lambda>0$, the population Sobolev-regularized risk is defined as
\[
  J_\lambda(f):= L_{\rm CE}(f) + \lambda\,\mathcal{R}(f).
\]
The existence of the unique minimizer $f_\lambda$ is provided in Lemma \ref{lm: uniq and exist of pop}.
We estimate this population penalty using the same fixed source--target design. Accordingly, define the fixed source--target empirical Sobolev penalizer:
\[
  \widehat{\mathcal{R}}_{\mathcal D}(f):=
  \frac{1}{2n}\sum_{j=1}^{n} \|\nabla f(X_j^p)\|_2^2
  +
  \frac{1}{2n}\sum_{j=1}^{n} \|\nabla f(X_j^q)\|_2^2.
\]
The fixed source--target empirical energy functional is
\[
  \widehat J_{\lambda,\mathcal D}(f):= \widehat L_{{\rm CE}, \mathcal D}(f)
     + \lambda\,\widehat{\mathcal{R}}_{\mathcal D}(f).
\]
Our method proceeds by first estimating the log–density ratio as the minimizer of an empirical energy functional,
\begin{eqnarray}\label{eq:empirical_loss}
    \widehat f_{\lambda, \mathcal  F}
  \in \arg\min_{f\in \mathcal  F} \widehat J_{\lambda,\mathcal D}(f),
\end{eqnarray}
and then taking its gradient to obtain the desired score difference.
We assume that the minimum is attained in $\mathcal  F$.

\section{Upper bound analysis}\label{sec: Theory}

In this section, we provide a statistical convergence rate for our Soblev regularized estimator, in terms of the number of samples used in optimization, for which the proof is deferred to Appendix \ref{sec: proof of bounded case}.

\begin{theorem}[Convergence rate for Sobolev-regularized logistic empirical risk minimization (ERM)]
\label{thm:classification-oracle}
Suppose Assumption~\ref{assm: density} hold and the smoothness index
satisfies \(s\le4\). Consider neural network $\mathcal{F}_M(L, W, S, B)$  with
$N_n\asymp n^{\frac{d}{d+2s-2}}$,
 $L=\mathcal{O}(1)$,
$W,S,B\asymp N_n,
\lambda\asymp N_n^{-\frac{s-1}{d}}$.
Then the soblev-penalized  estimator $\widehat f_{\lambda,\cF}$ in \eqref{eq:empirical_loss} satisfies the following upper bound with probability $1-2n^{-2}$:
\begin{equation*}
     \|\widehat f_{\lambda,\cF}-f^\star\|_{H^1(\mu)}^2 \lesssim n^{-\frac{s-1}{d+2s-2}}\log n.
\end{equation*}
\end{theorem}
\begin{remark}
(a). The restriction \(s\le 4\) is technical and follows from the cubic quasi-interpolant and ReLU\(^3\) approximation used in the proof, following the construction in \cite{lu2021machine}. Higher smoothness can be handled by higher-order quasi-interpolants and ReLU\(^r\) networks with spline order \(r+1\ge s\).
\vspace{5pt}

\noindent  (b). Since $\mu$ admits a $C^2$ density $\rho$  on $\Omega$, it is equivalent to the uniform distribution $\mu_0$. Consequently, the above bound can also be  equivalently expressed as
\begin{equation*}
    \|\widehat f_{\lambda,\cF}-f^\star\|_{H^1(\mu_0)}^2 \lesssim n^{-\frac{s-1}{d+2s-2}}\log n.
\end{equation*}

\noindent (c). Moreover, since $J_\lambda$ is strongly convex over $\cH_M$, a local Rademacher complexity argument yields a sharper rate than the $\mathcal{O}(n^{-\frac{s}{d+4s}})$ rate obtained for the Sobolev-regularized estimators in \cite{ding2025semi}, where they used global Rademacher complexity for regression problems.

\end{remark}

\paragraph{Proof outline.} The proof proceeds in four steps. First, we establish key analytic properties of the population energy functional $J_\lambda$, including strong convexity and smoothness (see \cref{lem:J-strong-smooth}). Second, we quantify the regularization-induced bias at the population level (see \cref{lem:bias}). Third, we relate the generalized error to the excess energy via localized Rademacher complexity (see \cref{thm: generalized error}). Finally, combining these ingredients yields the desired rate.

For any functional $ J:H^s(\mu) \rightarrow \R$, let  $D J(g)$ be the Fr\'echet derivative of $J$ at $g$. We have the following result.

\begin{lemma}[Strong convexity and smoothness of the penalized risk]
\label{lem:J-strong-smooth}
Define
$
  c_{\min}:= \frac{1}{4\cosh^2(M/2)}.
$
Then the penalized risk $J_\lambda(f):= L_{\rm CE}(f) + \lambda\,\mathcal{R}(f)$
is Fr\'echet differentiable on $\cH_M$ and satisfies, for all $f,g\in \cH_M$,
\begin{align}
J_\lambda(f)-J_\lambda(g)-DJ_\lambda(g)[f-g]
&\ge
\frac{c_{\min}}{2}\|f-g\|_{L^2(\mu)}^2
+\lambda\|\nabla(f-g)\|_{L^2(\mu)}^2,
\label{eq:J-strong-general}\\
J_\lambda(f)-J_\lambda(g)-DJ_\lambda(g)[f-g]
&\le
\|f-g\|_{L^2(\mu)}^2
+\lambda\|\nabla(f-g)\|_{L^2(\mu)}^2.
\label{eq:J-smooth}
\end{align}
\end{lemma}
Note that the parameter  $M>0$ used in $c_{\min}$ is the uniform logit bound, i.e.\ $|f(x)|\le M$ for all
$f\in \cF \cup f^\star$ and $\mu$-a.e.~$x$.

Our argument follows the general outline of \cite{ding2025semi}, with an important modification: because the cross-entropy loss fails to be globally strongly convex, uniqueness cannot be obtained directly. Instead, we establish conditional strong convexity restricted to $\cH_M$. Our proof proceeds by (i) deriving strictly positive pointwise curvature bounds for the logistic loss within $[-M, M]$, (ii) lifting the curvature bounds to the functional space via Taylor expansions, and (iii) combining the bounds with the quadratic structure of the Sobolev regularizer.

Denote $\Delta$ as the Laplacian operator
$\Delta f(x):= \sum_{j=1}^d \partial_{x_j x_j} f(x)$, interpreted in the weak sense for $f\in H^1(\mu)$.
 Define the weighted elliptic operator
$
  \mathcal K h:= \Delta h + \nabla h\cdot\nabla\log \rho.
$
Then \(\mathcal K f^\star\in L^2(\mu)\), and define
\[
\beta:=
  64\,\,\|{\mathcal K f^\star}\|^2_{L^2(\mu)}\,\cosh^4\!\big(\tfrac{M}{2}\big).
\]
Next we quantify the regularization-induced bias.
\begin{lemma}[Bias of the population Sobolev solution]
\label{lem:bias}
For all $\lambda>0$, the unique minimizer $f_\lambda$
of $J_\lambda$ satisfies
\begin{eqnarray*}
     \|f_\lambda-f^\star\|_{L^2(\mu)}^2
  \le \beta\,\lambda^2,\,\, \|\nabla(f_\lambda-f^\star)\|_{L^2(\mu)}^2
  \le \beta\,\lambda.
\end{eqnarray*}
\end{lemma}
The proof tests the first-order variational inequality for the constrained minimizer $f_\lambda$ in the direction $f-f_\lambda$, and then uses the weighted Green identity to relate the regularization part in EL equation to the Laplacian of $f^\star$, which controls the bias magnitude.

\begin{theorem}[Shifted oracle inequality for the clipped sieve]
\label{thm: generalized error}
Fix \(0<\lambda<1\). Let $\mathcal F:=\mathcal F_M(L,W,S,B)$ be the clipped and gradient-bounded neural-network sieve, where
\(L=\mathcal{O}(1)\), \(W=\mathcal{O}(N)\), \(S=\mathcal{O}(N)\), and \(B=\mathcal{O}(N)\).
Let \(f_0\in\mathcal F\) be any fixed comparator independent of the data, and  \(f_\lambda\) as the population minimizer of \(J_\lambda\). Let
\[
\widehat f_{\lambda,\mathcal F}
\in
\arg\min_{f\in\mathcal F}
\widehat J_{\lambda,\mathcal D}(f).
\]
Then for any \(t>0\), with probability at least \(1-e^{-t}\),
\begin{equation}
\label{eq:shifted-oracle}
J_\lambda(\widehat f_{\lambda,\mathcal F})
-
J_\lambda(f_\lambda)
\lesssim
J_\lambda(f_0)
-
J_\lambda(f_\lambda)
+
r^\star
+
\frac{t}{n},
\end{equation}
where \(r^\star\) is the critical radius of the sub-root function
\[
\phi(r):=
{C_0}\left[
\frac{1}{n}
+
\sqrt{
\frac{S\,3^L r}{n}\log(BWn)
}
\right].
\]
Here \({C_0}\) depends only on the fixed envelope \(M\) and \(c_{\min}^{-1}\),
but not on \(n,N,W,S,B\).
\end{theorem}

The proof is deferred to Appendix \ref{sec: proof of generalzied error}. The key property facilitating our fast rate analysis is the structural constraint of the hypothesis space $\cH_M$. Since optimization is performed over a bounded domain, the regularized loss $J_\lambda$ is strongly convex and the estimator remains bounded.

Strong convexity and boundedness ensure that the variance of the error diminishes as the estimator approaches the optimum. This motivates the use of localized, rather than global, complexity measures. In the proof, we formalize this intuition using the framework of \cite{lu2021machine}.

\section{Minimax lower bound}
This section provides  a minimax lower bound for the ratio estimation, for which the proof is deferred to Appendix \ref{sec:proof_minimax}.

Let \(\mathcal C_{\rm pair}\) be the class of pairs \((p,q)\) satisfying Assumption~\ref{assm: density}, and write \(f_{p,q}:=\log(q/p)\). Let \(\mathbb E_{p,q}\) denote expectation under the fixed source--target design \(X_i^p\overset{\mathrm{i.i.d.}}{\sim}p\), \(X_i^q\overset{\mathrm{i.i.d.}}{\sim}q\).

\begin{theorem}[Minimax lower bound for score difference estimation]
\label{thm: minimax lower bound}
For any estimator $\psi:\left(\R^d\right)^n\times\left(\R^d\right)^n \to H^1(\mu_0)$, we have the minimax lower bound
\begin{align}
  &{\inf}_{\psi}{\sup}_{(p,q)\in\mathcal C_{\rm pair}}
  \E_{p,q}\Big[
    \big\|\psi(\cS,\cT)-f_{p,q}\big\|_{H^1(\mu_{0})}^2
  \Big]\nonumber\;\gtrsim\;
  n^{-2(s-1)/(2s+d)}.  \label{eq:minimax-lower-H1}
\end{align}
\end{theorem}
{\bf Technical novelties.} Our lower-bound analysis builds on the Local Fano method, but requires several new ingredients to address three structural obstacles specific to our setting.
First, we confront a regularity mismatch between the observation and the target quantity: unlike \cite{lu2021machine}, where the observation involves a differential operator that smooths the estimation problem, here we must recover high-order
$H^1$ information from lower-order
$L^2$ observations. Second, we characterize the information-theoretic limits under a heteroscedastic observation model induced by the joint sampling of the source and target densities. Finally, we handle the global normalization constraint intrinsic to density estimation, showing that local perturbations remain statistically indistinguishable even after enforcing normalization through global partition functions.

\paragraph{Discussion on the optimality gap.}
Comparing Theorem~\ref{thm:classification-oracle} with Theorem~\ref{thm: minimax lower bound}, we observe a gap between the achievable upper rate $\tilde{\mathcal{O}}(n^{-\frac{s-1}{d+2s-2}})$ and the minimax lower bound $\tilde{\Omega}(n^{-\frac{2(s-1)}{d+2s}})$.
We conjecture that this suboptimality is an artifact of the saturation phenomenon associated with single-step Tikhonov-type regularization \cite{bauer2007regularization, engl1996regularization}. Within our framework, the Sobolev penalty $\lambda \|\nabla f\|^2$  induces a bias–variance trade-off in which the regularization parameter $\lambda$ simultaneously governs approximation error and statistical stability.  In particular, enforcing conditional strong convexity requires $\lambda$ to be sufficiently large to control the stochastic error--manifested through the $1/\lambda$ factor in the gradient stability bounds. This constraint limits how rapidly the bias can decay, thereby preventing the estimator from attaining the optimal nonparametric rate.

Theoretically, the gap between the upper and lower bounds could be closed using iterative regularization schemes, such as iterated Tikhonov regularization \cite{engl1996regularization}, which are known to attain optimal convergence rates. However, implementing such methods in a deep learning setting would require training a sequence of neural networks where each subsequent network relies on the previous one, leading to prohibitive computational and memory costs for large-scale applications such as diffusion model and transfer learning. Consequently, while our single-step estimator is not minimax optimal, it is computationally efficient and practically scalable, and already exhibits strong empirical performance and stability in the small-sample regime (see Table~\ref{tab:grouped_results_bold}).

\section{Generalization to diffusion models: time dependence and unbounded domains }\label{sec:generalization}

In this section, we extend our framework to transfer learning for diffusion models. Compared to the static setting in Section \ref{sec:setup-OU-tilde}, two additional challenges arise: (i) the density ratio becomes time dependent; and (ii) Gaussian perturbations render the data support unbounded. We address both issues by reformulating the problem on an augmented time-space domain, applying a truncation argument and proposing a projected-rescaled algorithm.

{\bf Augmented time-space formulation under the VP forward model.}
To formalize the diffusion-model setting, we focus on the variance-preserving
(VP) forward noising model. Let \(X_0\sim p_0\) and \(Y_0\sim q_0\) denote
the source and target initial variables. For \(t\in[0,T]\), assume
\[
X_t=\alpha_t X_0+\sigma_t \xi,\,\, Y_t=\alpha_t Y_0+\sigma_t \xi', \,\,\xi,\xi'\sim\mathcal N(0,I_d),
\]
where \(\alpha_t\in\mathbb R\) and \(\sigma_t>0\) are deterministic scalar
VP coefficients.

Fix $t_0\in(0,T)$ as the early-stopping time, commonly used in the diffusion model literature \cite{han2024neural}. In this context, the goal is to estimate the score difference $\nabla_x f^\star(t, x):=\nabla_x \log (q_t(x)/p_t(x))$ for $t \in [t_0, T]$. We consider the corresponding learning problem on the augmented space:
\begin{equation*}
\tilde{\Omega}:= [t_0, T] \times \mathbb{R}^d, \quad \tilde{z}:= (t, x) \in \tilde{\Omega}.
\end{equation*}
We write \(\rho_t(x):=\frac12(p_t(x)+q_t(x))\), and use \(\mu_t(dx):=\rho_t(x)\,dx\) as the reference measure on each time slice.

The objective is to learn a function $f: \tilde{\Omega} \to \mathbb{R}$ that minimizes a time-averaged risk. Since the temporal domain $[t_0, T]$ is compact, the main difficulty arises from the unbounded spatial domain $\mathbb{R}^d$. In particular, strong convexity fails globally on $\mathbb{R}^d$. To address this issue, we impose suitable assumptions on the tail behavior of the data distribution, which is a common and reasonable assumption to assume in the diffusion model literature \cite{han2024neural,li2023generalization,kong2024diffusion}.

\begin{assumption}[Compact initial support and VP coefficients]
\label{assm:unbounded_diffusion}
The initial source and target distributions are compactly supported: there
exists \({R_0}<\infty\) such that
\[
\supp(p_0)\cup\supp(q_0)\subseteq B_{{R_0}}.
\]
Moreover, the VP coefficients satisfy
$
\alpha,\sigma\in C^s([t_0,T]),
$
and there exist constants
$
0<\underline\sigma\le \overline\sigma<\infty, {\omega_0}<\infty$, such that for all \(t\in[t_0,T]\),
\[
\underline\sigma\le \sigma_t\le \overline\sigma,
\qquad
\max_{0\le a\le s}
\left(
|\partial_t^a\alpha_t|
+
|\partial_t^a\sigma_t|
\right)
\le {\omega_0}.
\]
\end{assumption}
\paragraph{Truncation-dependent clipped network class.}
For \(R>0\), the network is trained on the rescaled domain
\([t_0,T]\times B_1\) with input \((t,\bar x)\in\mathbb R^{d+1}\). We use
\[
\mathcal F_R:=
\mathcal F^{(d+1)}_{M_R}(L,W,S,B),
\qquad
M_R:=C_M(1+R),
\]
where \(\mathcal F^{(d+1)}_{M_R}\) is the clipped sparse ReLU\(^3\) class
from Section~\ref{sec: Theory}, now with input dimension \(d+1\). The constant
\(C_M\) is fixed independently of \(R,n,N,\lambda\) and is chosen large enough
so that the envelope requirements in Appendix~\ref{subsec:uniform-bound} hold.
For functions on \([t_0,T]\times B_{2R}\), write
\[
(\pi_R^{-1}h)(t,\bar x):=h(t,2R\bar x),
\qquad (t,\bar x)\in [t_0,T]\times B_1.
\]
In particular, for every \(R\ge1\),
\[
\|\pi_R^{-1}f^\star\|_{L^\infty([t_0,T]\times B_1)}
+
\|\nabla_{\bar x}(\pi_R^{-1}f^\star)\|_{L^\infty([t_0,T]\times B_1)}
\le M_R.
\]
This follows from the growth bounds under
Assumption~\ref{assm:unbounded_diffusion}.

\paragraph{Projected-rescaled empirical objective.}
Define the Euclidean projection onto \(B_{2R}\) by
\[
{\operatorname{Proj}_{2R}}(x):=\frac{x}{\max\{1,\|x\|/(2R)\}}.\]
Then $\bar x:=\frac{{\operatorname{Proj}_{2R}}(x)}{2R}\in B_1.$
To approximate the time-integrated source--target classification risk, we
augment the source and target samples with independent time draws,
\[
\mathcal D_{\rm ext}:=
\{(t_i^p,X_i^p,0)\}_{i=1}^n
\cup
\{(t_i^q,X_i^q,1)\}_{i=1}^n,
\]
where $t_i^p,t_i^q\stackrel{\rm i.i.d.}{\sim}{\rm Unif}([t_0,T])$, $X_i^p\mid t_i^p \sim p_{t_i^p}$,
$X_i^q\mid t_i^q \sim q_{t_i^q},$
with all variables independent across the \(p\)-sample and \(q\)-sample
blocks and across indices, and $\bar X_i^p:=\frac{{\operatorname{Proj}_{2R}}(X_i^p)}{2R}, \bar X_i^q:=\frac{{\operatorname{Proj}_{2R}}(X_i^q)}{2R}.$ We estimate the rescaled log-density ratio by
\[
\widehat f_R
\in
\arg\min_{f\in\mathcal F_R} \widehat J^{\rm proj}_{\lambda,\mathcal D_{\rm ext},2R}(f),
\]
\begin{equation}\label{eq:empirical_objective_proj}
\begin{aligned}
  &\widehat J^{\rm proj}_{\lambda,\mathcal D_{\rm ext},2R}(f):= \frac1{2n}\sum_{i=1}^n
\left[
\ell_{\rm CE}\big(0,f(t_i^p,\bar X_i^p)\big)
+
\frac{\lambda}{4R^2}
\|\nabla_{\bar x}f(t_i^p,\bar X_i^p)\|_2^2
\right]\\
&\quad+
\frac1{2n}\sum_{i=1}^n
\left[
\ell_{\rm CE}\big(1,f(t_i^q,\bar X_i^q)\big)
+
\frac{\lambda}{4R^2}
\|\nabla_{\bar x}f(t_i^q,\bar X_i^q)\|_2^2
\right],
\end{aligned}
\end{equation}
The factor \(1/(4R^2)\) accounts for the spatial rescaling
\(x=2R\bar x\).

{\bf Cutoff extension and core-tail decomposition.}
Let \(\chi_R:\mathbb R^d\to[0,1]\) be a smooth cutoff satisfying
\[
\chi_R\equiv1\text{ on }B_R, \chi_R\equiv0\text{ on }\mathbb R^d\setminus B_{2R}, \|\nabla\chi_R\|_\infty\lesssim R^{-1}.
\]
Using the optimizer \(\widehat f_R\) from
\eqref{eq:empirical_objective_proj}, define the global estimator on
\(\widetilde\Omega=[t_0,T]\times\mathbb R^d\) by
\begin{equation}\label{eq:global_cutoff_estimator}
\widetilde f^{(R)}(t,x):=
\chi_R(x)\,
\widehat f_R\!\left(t,\frac{{\operatorname{Proj}_{2R}}(x)}{2R}\right).
\end{equation}

Define the local loss density associated with  \(H^1(\mu_t)\) by
\begin{equation}\label{eq:norm}
\mathcal L(f,g):=
|f-g|^2+\|\nabla_x f-\nabla_x g\|_2^2,
\end{equation}
where all terms are evaluated pointwise. We decompose the time-averaged
global error as
\begin{equation}
\label{eq:core_tail_decomp}
\begin{aligned}
\int_{t_0}^T
\|\widetilde f_t^{(R)}-f_t^\star\|_{H^1(\mu_t)}^2\,dt
&=
\underbrace{
\int_{t_0}^T
\int_{B_R}
\mathcal L(\widetilde f^{(R)},f^\star)\,d\mu_t\,dt
}_{\text{\textbf{Main error}}}\\
&+
\underbrace{
\int_{t_0}^T
\int_{B_R^c}
\mathcal L(\widetilde f^{(R)},f^\star)\,d\mu_t\,dt
}_{\text{\textbf{Tail error}}}.
\end{aligned}
\end{equation}
The main error is controlled by applying the compact-domain analysis to the
pulled-back estimator on \([t_0,T]\times B_{2R}\). The tail error is controlled by the
sub-Gaussian decay of \(\mu_t\) under Assumption~\ref{assm:unbounded_diffusion},
together with the growth bounds for \(f^\star\) implied by the VP structure.
Choosing \(R=R(n)\) then balances these two errors and yields
Theorem~\ref{thm: unbounded_convergence_nohop}.

\begin{theorem}[Convergence for VP diffusion models on unbounded domains]
\label{thm: unbounded_convergence_nohop}
Suppose Assumption~\ref{assm:unbounded_diffusion} holds and the smoothness index satisfies \(s\le4\).
For each \(R>0\), consider the \(R\)-dependent clipped and
spatial-gradient-bounded time-space class $\mathcal F_R
= \mathcal F^{(d+1)}_{M_R}(L,W,S,B)$ with $M_R=C_M(1+R).$
Let $N_n^{\rm ext}\asymp n^{\frac{d+1}{d+1+2s-2}}, L=\mathcal{O}(1),
W,S,B\asymp N_n^{\rm ext},\lambda\asymp (N_n^{\rm ext})^{-\frac{s-1}{d+1}}$. Choose
$R=A\sqrt{\log n}$, where \(A>0\) is sufficiently large depending only on
\(t_0,T,R_0,\omega_0,\underline\sigma,\overline\sigma\). Then, with
probability at least \(1-3n^{-2}\),
\[
\int_{t_0}^T
\|\widetilde f_t^{(R)}-f_t^\star\|_{H^1(\mu_t)}^2\,dt
\le
C e^{C\sqrt{\log n}}(\log n)^K n^{-\frac{s-1}{d+1+2s-2}},
\]
where $\widetilde f_t^{(R)}$ is defined in \eqref{eq:global_cutoff_estimator} and \(C,K<\infty\) are independent of \(n\). As a consequence, for every
\(\varepsilon>0\), there exists \(C_\varepsilon<\infty\) such that
\[
\int_{t_0}^T
\|\widetilde f_t^{(R)}-f_t^\star\|_{H^1(\mu_t)}^2\,dt
\le
C_\varepsilon n^{-\frac{s-1}{d+1+2s-2}+\varepsilon}.
\]
\end{theorem}

\section{Experiments}
This section presents experiments on both synthetic environments and real-world datasets, demonstrating the promising performance of our proposed method.
\subsection{Estimation error in simulation environment}
\begin{figure}[tbp]
    \centering

    \includegraphics[width=0.95\linewidth]{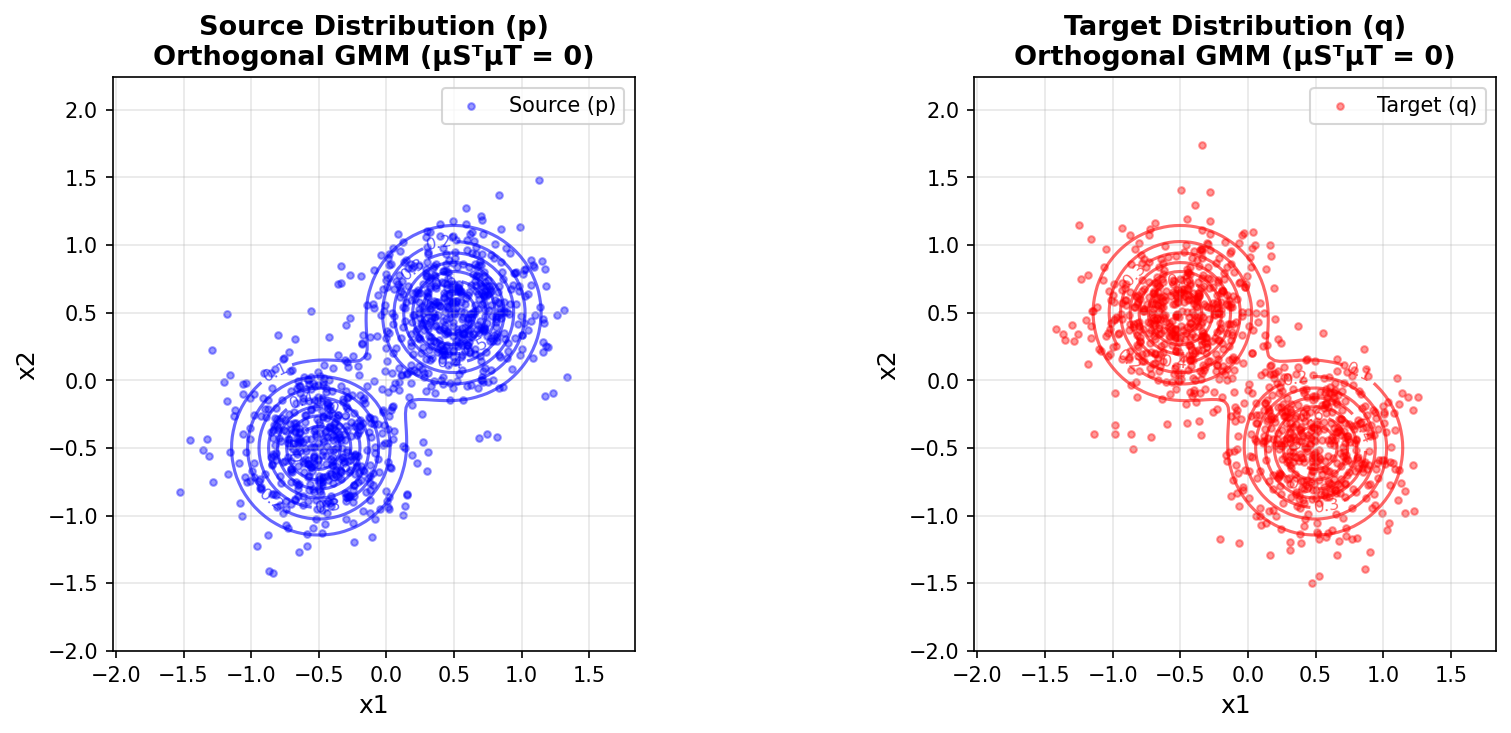}

    \caption{
    One example of the simulation distributions (All distribution pairs can be found in Appendix \ref{subsec:Synthetic source--target pairs}).
    }
    \label{fig:one-subfig}

\end{figure}
In this simulation study, we show the effectiveness of our regularization, in terms of the accuracy of estimating the gradient of the log density ratio
\(\nabla \log \tfrac{q}{p}\).

We construct three distinct pairs of two–dimensional source and target
distributions \((p,q)\), for which the quantity
\(\nabla \log \tfrac{q}{p}\) can be computed in closed form. Detailed descriptions of the three distribution pairs are provided in Appendix~\ref{subsec:Synthetic source--target pairs}.
For each pair we train a three–layer MLP \(f_\theta(x)\) to approximate
the log density ratio.

We consider two training objectives: a standard
classification-based objective,
and our Sobolev-penalized classification objective.
For each \((p,q)\) pair we train the network with
\(N \in \{10, 100, 1000\}\) samples drawn independently from both
\(p\) and \(q\).
After training, we set the gradient estimator directly as
\({g}_\theta(x):= \nabla_x f_\theta(x)\),
and estimate its mean squared error using \(1000\) new samples from both \(p\) and \(q\).

The errors are reported in Table~\ref{tab:grouped_results_bold}.
We first observe that the Sobolev-penalized estimator (\textit{Cls w sob}) consistently outperforms the naive classification baseline (\textit{Cls w/o sob}) across all training sizes.
Notably, this relative improvement is most pronounced in the small-sample regime
({\it e.g.}, more than a $50\%$ improvement for $N=20$), confirming that the Sobolev penalty effectively stabilizes the estimator when data are scarce.

\begin{table}[htbp]
\centering
\caption{Accuracy for gradient of log-density ratio estimation. Best results are marked in bold. Improvement indicates the relative reduction in error from Cls w/o sob to Cls w sob.}
\label{tab:grouped_results_bold}
\begin{tabular}{lrrrr}
\toprule
Train Size & Cls w/o sob & Cls w sob & Kernel & Improvement \\
           & (A)           & (B)     &        & $(A-B)/A$ [\%] \\
\midrule
\multicolumn{5}{l}{\textbf{ROTATED\_RIDGE}} \\
20   & $61.94 \pm 7.98$ & $30.96 \pm 6.92$ & $\mathbf{7.65 \pm 1.24}$ & 50.02 \\
200  & $5.91 \pm 0.92$  & $\mathbf{4.29 \pm 0.61}$  & $6.44 \pm 0.43$ & 27.36 \\
2000 & $1.68 \pm 0.21$  & $\mathbf{0.96 \pm 0.12}$  & $5.80 \pm 0.06$ & 42.96 \\
\midrule
\multicolumn{5}{l}{\textbf{ORTHOGONAL\_GMM}} \\
20   & $24.10 \pm 2.73$ & $\mathbf{11.89 \pm 0.80}$ & $25.87 \pm 1.42$ & 50.66 \\
200  & $8.90 \pm 2.74$  & $\mathbf{6.31 \pm 0.44}$  & $23.78 \pm 0.63$ & 29.19 \\
2000 & $5.03 \pm 0.34$  & $\mathbf{3.83 \pm 0.17}$  & $22.80 \pm 0.21$ & 23.86 \\
\midrule
\multicolumn{5}{l}{\textbf{BOUNDED}} \\
20   & $103.46 \pm 7.16$ & $48.77 \pm 2.62$ & $\mathbf{7.73 \pm 0.59}$ & 52.86 \\
200  & $15.68 \pm 2.45$  & $9.86 \pm 1.19$  & $\mathbf{3.68 \pm 0.09}$ & 37.12 \\
2000 & $2.42 \pm 0.17$   & $\mathbf{1.96 \pm 0.13}$  & $3.55 \pm 0.01$ & 18.88 \\
\bottomrule
\end{tabular}
\end{table}
\begin{table}[tbp]
\centering \small
\caption{Accuracy for gradient of log-density ratio estimation.
Values are reported as mean (standard deviation).
Best results are marked in bold.}
\label{tab:grouped_results_compact}
\begin{tabular}{lrrr}
\toprule
Train & Cls w/o sob & Cls w sob & Improve \\
 Size          & (A)         & (B)       & $(A-B)/A$ [\%] \\
\midrule
\multicolumn{4}{l}{\textbf{ROTATED\_RIDGE}} \\
20   & $61.94(7.98)$ & $\mathbf{30.96(6.92)}$ & 50.02 \\
200  & $5.91(0.92)$  & $\mathbf{4.29(0.61)}$  & 27.36 \\
2000 & $1.68(0.21)$  & $\mathbf{0.96(0.12)}$  & 42.96 \\
\midrule
\multicolumn{4}{l}{\textbf{ORTHOGONAL\_GMM}} \\
20   & $24.10(2.73)$ & $\mathbf{11.89(0.80)}$ & 50.66 \\
200  & $8.90(2.74)$  & $\mathbf{6.31(0.44)}$  & 29.19 \\
2000 & $5.03(0.34)$  & $\mathbf{3.83(0.17)}$  & 23.86 \\
\midrule
\multicolumn{4}{l}{\textbf{BOUNDED}} \\
20   & $103.46(7.16)$ & $\mathbf{48.77(2.62)}$ & 52.86 \\
200  & $15.68(2.45)$  & $\mathbf{9.86(1.19)}$  & 37.12 \\
2000 & $2.42(0.17)$   & $\mathbf{1.96(0.13)}$  & 18.88 \\
\bottomrule
\end{tabular}
\end{table}
\subsection{Transfer learning with real-world datasets}
In transfer learning, the goal is to leverage labeled samples from a source distribution to improve prediction on a related target distribution. We observe source data $\mathcal{D}_p:=\{(x_p^{(i)}, z_p^{(i)})\}_{i=1}^{n_p}$ drawn from a joint distribution $P$ and target samples $\mathcal{D}_q:=\{(x_q^{(j)}, z_q^{(j)})\}_{j=1}^{n_q}$ drawn from a different joint distribution $Q$.
In practice, labeled target samples are often scarce, making direct training unreliable. Our objective is therefore to infer a labeling function for
$\mathcal{D}_q$ by transferring information from $\mathcal{D}_p$.

\subsubsection{WGF Methods}\label{subsec:WGF}
When many unlabeled target samples are available, following prior optimal-transport-based adaptation methods\cite{liu2023minimizing}, we use WGF to transport labeled source particles toward the target distribution and then train a classifier on the transported samples.

WGF requires the score difference $\nabla \log q - \nabla \log p_t$.
We therefore evaluate the proposed estimators by implementing WGF on the Office–Caltech-10 benchmark \cite{saenko2010adapting}, which contains four visual datasets: Amazon, Caltech, DSLR, and Webcam. Following standard practice, all samples are projected onto a 100-dimensional PCA subspace. We compare target classification accuracy across four settings: a source-only RBF SVM baseline, and RBF SVMs trained on WGF-transported source samples using (i) kernel-based estimators (LL; \cite{liu2023minimizing}), (ii) unregularized classification-based estimators, and (iii) our Sobolev-regularized classification estimator.

The results (Table~\ref{table: domain adaptation}) show that directly reusing source classifiers can lead to severe performance degradation (e.g., Amazon $\to$ DSLR), whereas joint distribution–based adaptation substantially alleviates this issue. Moreover, Sobolev regularization enables the classification-based method to achieve the strongest overall performance.

While kernel-based estimators attain accuracy comparable to our Sobolev-regularized approach in low-dimensional settings, they incur substantially higher computational cost. The kernel-based WGF baseline is intrinsically local, requiring a new gradient estimator to be trained at each WGF iteration through kernel optimization, which repeatedly evaluates interactions with the training samples (see Table~\ref{tab:gradient_times}). In contrast, the classification-based estimator amortizes gradient evaluation: once trained, the score difference at a new point is obtained via a single forward and backward pass of a fixed network. This amortization accounts for the significant speedup and enables scalability to large-scale settings, including diffusion-based generation.

\begin{table}[htbp]
    \centering\small
    \caption{Comparison of classification accuracy on Office-Caltech-10 Dataset.}\label{table: domain adaptation}
    \label{tab:results}
    \begin{tabular}{lcccc}
        \toprule
        \textbf{$\cD_p\rightarrow\cD_q$} & \textbf{Base} & \textbf{Cls w/o sob} & \textbf{Kernel} & \textbf{Cls w sob} \\
        \midrule
        amz.$\to$cal. & 0.7115 & 0.7863 & 0.8379 & \textbf{0.8504} \\
        amz.$\to$dslr    & 0.2675 & 0.7580 & \textbf{0.7962} & \textbf{0.7962} \\
        amz.$\to$web.  & 0.3932 & 0.7390 & \textbf{0.8678} & 0.8203 \\
        cal.$\to$amz. & \textbf{0.9081} & 0.8716 & \textbf{0.9081} & 0.8977 \\
        cal.$\to$dslr   & 0.2420 & 0.8535 & 0.8344 & \textbf{0.8623} \\
        cal.$\to$web. & 0.3797 & 0.7763 & \textbf{0.8203} & 0.8034 \\
        dslr$\to$amz.    & 0.7035 & 0.8017 & \textbf{0.8716} & 0.8006 \\
        dslr$\to$cal.   & 0.6572 & 0.7489 & \textbf{0.8094} & 0.7427 \\
        dslr$\to$web.    & 0.9492 & 0.9492 & 0.9492 & \textbf{0.9525} \\
        web.$\to$amz.  & 0.6294 & 0.6002 & 0.5877 & \textbf{0.7046} \\
        web.$\to$cal. & 0.3954 & 0.7070 & \textbf{0.7640} & 0.6456 \\
        web.$\to$dslr    & 0.8535 & 0.9618 & 0.9554 & \textbf{0.9682} \\
        \bottomrule
    \end{tabular}
\end{table}
\begin{table}[htbp]
    \centering\small
    \caption{Average gradient computation time across the 12 transfer tasks shown in Table \ref{table: domain adaptation}.}\label{tab: gradient cal time}
    \label{tab:gradient_times}
    \begin{tabular}{lc}
        \toprule
        \textbf{Method} & \textbf{Gradient Calculation Time} \\
        \midrule
        Baseline & N/A \\
        Kernel & 2.0617 s \\
        Cls w/o sob & \textbf{0.0004 s} \\
        Cls w sob  & \textbf{0.0004 s} \\
        \bottomrule
    \end{tabular}
\end{table}
\subsubsection{ Diffusion Models}\label{subsec:diffusion}

When the amount of unlabeled target data is extremely limited, directly training a generative model on the target task is often unreliable.
A more effective alternative is to leverage a large source domain and adapt a pre-trained diffusion model to the target distribution via transfer.
In particular, recent work \cite{ouyang2024transfer} proposes to guide a source-trained diffusion model using an estimated score difference between the source and target distributions.

To evaluate the effectiveness of such transfer-based diffusion model methods under limited target data, we consider a benchmark task in electrocardiogram (ECG) generation. Following \cite{ouyang2024transfer}, we use the PTB-XL dataset \cite{wagner2020ptb} as the source task and the ICBEB2018 dataset \cite{liu2018open} as the target task, whose details can be found in Appendix \ref{subsec: ecg results detail}.

\paragraph{Evaluation protocol.} We assess different methods through their ability to generate target-task samples for downstream classification. (More results on computing time and generation quality can be found in Appendix \ref{subsec: ecg results detail}) Specifically, each method is used to generate a sufficient number of synthetic ECG samples, which are then combined with the limited target samples to train a classifier. Performance is evaluated on the target test set. We compare four approaches:
(1) \emph{Vanilla Diffusion}, which trains a diffusion model directly on limited target data;
(2) \emph{Finetune Generator}, which adapts a source-trained diffusion model to generate target-label samples;
(3) \emph{TGDP} \cite{ouyang2024transfer}, which trains the guidance function using standard classification objectives; and
(4) \emph{TGDP-SoB}, which incorporates Sobolev regularization into the TGDP framework.

Following the ECG benchmark protocol in \cite{strodthoff2020deep}, we report the Macro-averaged area under the ROC curve (AUC), macro-averaged $F_\beta$-score with $\beta=2$, and macro-averaged $G_\beta$-score with $\beta=2$; where $F_\beta = \frac{(1+\beta^2) \cdot \text{TP}}{(1+\beta^2) \cdot \text{TP} + \beta^2 \cdot \text{FN} + \text{FP}}$, $ G_\beta = \frac{\text{TP}}{\text{TP} + \text{FP} + \beta \cdot \text{FN}}$. As shown in Table~\ref{table: downstream}, TGDP substantially outperforms the baseline methods across all evaluation metrics. In addition, the Sobolev-regularized variant (TGDP-SoB) consistently yields further improvements, highlighting the benefit of regularized score difference estimation for diffusion-based transfer under limited target data.

\begin{table}[htbp]
\centering\small
\begin{tabular}{l|c|c|c}
\hline
\textbf{Method} & \textbf{AUC} & $\mathbf{F_{\beta=2}}$ & $\mathbf{G_{\beta=2}}$ \\
\hline
\hline
Vanilla Diffusion    & 0.844(07) & 0.590(09) & 0.331(08) \\
Finetune Generator   & 0.862(05) & 0.604(09) & 0.351(09) \\
TGDP        & 0.905(04) & 0.662(10) & 0.436(12) \\
\textbf{TGDP-SoB} & \textbf{0.915(05)}& \textbf{0.693(11)} & \textbf{0.453(11)} \\
\hline
\end{tabular}
\caption{Results on ECG benchmark for  downstream classification task. (90\% confidence intervals are provided via empirical bootstrapping  \cite{strodthoff2020deep}; 0.915(04) stands for 0.915 ± 0.004.)}\label{table: downstream}
\end{table}

\section*{Impact Statement}
This paper presents work whose goal is to advance the field of machine learning. There are many potential societal consequences of our work, none of which we feel must be specifically highlighted here.

\section*{Acknowledgments}
J. Blanchet gratefully acknowledges support from ONR under award N00014-24-1-2655, and the National Science Foundation (NSF) under grants 2312204 and 2403007. R. Xu gratefully acknowledges support from the NSF under grants DMS-2602037 and CAREER Award DMS-2614933.

\bibliographystyle{plainnat}
\bibliography{reference}

\newpage
\appendix
\paragraph{Appendix roadmap.}
Appendix~\ref{sec: proof of bounded case} proves the bounded-domain theory
from Section~\ref{sec: Theory}. It establishes the variational properties of
the Sobolev-regularized population objective, the clipped-network
approximation result, the regularization bias bound, the shifted oracle
inequality, and the resulting upper rate. Appendix~\ref{sec:proof_minimax}
proves the minimax lower bound in Theorem~\ref{thm: minimax lower bound} by a
local packing construction and Fano's method. Appendix~\ref{sec:proof_unbounded}
proves the diffusion-model extension in
Theorem~\ref{thm: unbounded_convergence_nohop}; the proof combines
truncation, rescaling, compact-cylinder bias and oracle bounds, and the
core-tail decomposition. Appendix~\ref{sec:supplementary-experiments}
contains the supplementary experimental details for the synthetic, WGF, and
ECG studies.

\section{Proof for Bounded Case (Section \ref{sec: Theory})}\label{sec: proof of bounded case}

\paragraph{Appendix A roadmap.}
The bounded-case proof is organized as follows. After fixing the two-group
average notation, Lemma~\ref{lm: uniq and exist of pop} proves existence and
uniqueness of the population minimizer. We then build the approximation
tools: Lemma~\ref{lem:spline-stability} gives the local B-spline
quasi-interpolant, and Proposition~\ref{prop: approximate of clipped NN}
realizes the resulting approximant in the clipped, gradient-bounded
neural-network sieve. With these deterministic ingredients in place,
Lemma~\ref{lem:J-strong-smooth} proves strong convexity and smoothness of
\(J_\lambda\), while Lemma~\ref{lem:bias} gives the regularization bias.
Finally, Theorem~\ref{thm: generalized error} gives the shifted localized
oracle inequality, and Appendix~\ref{sec: proof of upper bound} combines it
with the approximation and bias bounds to obtain
Theorem~\ref{thm:classification-oracle}.

\paragraph{Notation.}

Recall the bounded-domain fixed source--target design in \eqref{eq:data_set}. For any measurable
function \(\varphi:\Omega\times\{0,1\}\to\mathbb R\), define the avg operator on two-group:
\[
\overline P\varphi:=
\frac12\E_{X\sim p}[\varphi(X,0)]
+
\frac12\E_{X\sim q}[\varphi(X,1)],
\]
and
\[
\overline P_n\varphi:=
\frac{1}{2n}\sum_{i=1}^n\varphi(X_i^p,0)
+
\frac{1}{2n}\sum_{i=1}^n\varphi(X_i^q,1).
\]

With this notation,
$
  \widehat L_{\rm CE,\mathcal D}(f)
  =
  \overline P_n\{\ell_{\rm CE}(y,f(x))\}.
$.

We first show that the minimizer for our energy functional exists and is unique.
\begin{lemma}[Existence and uniqueness of the population risk minimizer in $\cH_M$]\label{lm: uniq and exist of pop}
Recall $L_{\rm CE}(f):= \E_{(X,Y)\sim\mu_{p,q}}
     \big[\,\ell_{\rm CE}(Y,f(X))\,\big]$. For any regularization parameter $\lambda > 0$, define the Sobolev-regularized population risk
\begin{equation*}
    J_{\lambda}(f):= L_{\rm CE}(f) + \lambda \mathcal{R}(f).
\end{equation*}
Then $J_{\lambda}$ admits a unique minimizer $f_{\lambda} \in \cH_M$:
\begin{equation*}
    f_{\lambda}:= \operatorname*{arg\,min}_{f \in \cH_M }J_{\lambda}(f).
\end{equation*}
\end{lemma}
\begin{proof}
 We proceed by establishing uniqueness via strong convexity and existence via the properties of the constraint set $\cH_M$.

From Lemma \ref{lem:J-strong-smooth}, the functional $J_{\lambda}$ satisfies the following strong convexity inequality for any $f, g \in \cH_M$:
\begin{equation} \label{eq:strong_convexity}
    J_{\lambda}(f) - J_{\lambda}(g) - DJ_{\lambda}(g)[f-g] \ge \frac{c_{\min}}{2}\|f-g\|_{L^{2}(\mu)}^{2} + \lambda\|\nabla(f-g)\|_{L^{2}(\mu)}^{2}.
\end{equation}
Since $\lambda > 0$ and $c_{\min} > 0$, the right-hand side is strictly positive for any $f \neq g$ (in the $H^1(\mu)$ norm sense).
Suppose there exist two distinct minimizers $f_1, f_2 \in \cH_M$ with $J_{\lambda}(f_1) = J_{\lambda}(f_2) = \inf_{f \in \cH_M} J_{\lambda}(f) = m$. By the strict convexity, for any $t \in (0, 1)$, we have:
\begin{equation*}
    J_{\lambda}(t f_1 + (1-t)f_2) < t J_{\lambda}(f_1) + (1-t)J_{\lambda}(f_2) = m.
\end{equation*}
This implies we have found an element with a risk strictly lower than the infimum $m$, which is a contradiction. Thus, the minimizer must be unique.

To apply the Direct Method in the Calculus of Variations
\cite{evans2022partial} on the constrained set, we first establish that the
feasible set $\cH_M$ is a weakly closed subset of $H^1(\mu)$.

First, $\cH_M$ is convex. For any $f, g \in \cH_M$ and $t \in [0, 1]$, by the triangle inequality:
\begin{equation*}
    \|t f + (1-t)g\|_{L^\infty} \le t\|f\|_{L^\infty} + (1-t)\|g\|_{L^\infty} \le tM + (1-t)M = M.
\end{equation*}
Thus, the convex combination remains in $\cH_M$.
Second, $\cH_M$ is closed in the strong topology of $H^1(\mu)$. Let $\{f_n\} \subset \cH_M$ be a sequence converging to $f$ in $H^1(\mu)$. Convergence in $H^1$ implies convergence in $L^2$, which in turn implies the existence of a subsequence converging pointwise almost everywhere. Since $|f_n(x)| \le M$ a.e., the pointwise limit must satisfy $|f(x)| \le M$ a.e. Thus, $f \in \cH_M$.
Since $\cH_M$ is a closed and convex subset of a Banach space, it is weakly closed.

Next we show coercivity.
Although functions in $\mathcal{H}_M$ are bounded in $L^\infty$, they are not a priori bounded in $H^1(\mu)$ (as gradients can be arbitrarily large). However, fixing a reference $g=0 \in \cH_M$, the inequality \eqref{eq:strong_convexity} implies:
\begin{equation*}
    J_{\lambda}(f) \ge J_{\lambda}(0) + DJ_{\lambda}(0)[f] + C(\lambda) \|f\|_{H^1(\mu)}^2.
\end{equation*}
Since the quadratic term dominates the linear functional as $\|f\|_{H^1} \to \infty$, $J_{\lambda}$ is coercive. This ensures that any minimizing sequence $\{f_n\}_{n=1}^\infty \subset \cH_M$ such that $J_{\lambda}(f_n) \to \inf_{f \in \cH_M} J_{\lambda}(f)$ is bounded in the $H^1(\mu)$ norm.

Finaly, we  take limit via weak lower semicontinuity.
Since the minimizing sequence $\{f_n\} \subset \cH_M$ is bounded in the reflexive space $H^1(\mu)$, by the Banach-Alaoglu theorem, there exists a subsequence $\{f_{n_k}\}$ that converges weakly to some limit $f_\lambda \in H^1(\mu)$.
Crucially, because $\cH_M$ is weakly closed, the limit must satisfy the constraints: $f_\lambda \in \cH_M$.
Finally, since $J_{\lambda}$ is continuous and convex, it is weakly lower semicontinuous. Therefore:
\begin{equation*}
    J_{\lambda}(f_\lambda) \le \liminf_{k \to \infty} J_{\lambda}(f_{n_k}) = \inf_{f \in \cH_M} J_{\lambda}(f).
\end{equation*}
Thus, the minimum is attained by $f_\lambda$ within the set $\cH_M$.
\end{proof}

\begin{lemma}[Tensor-product B-spline quasi-interpolation]
\label{lem:spline-stability}
Fix integers \(k,l\in\mathbb N\). There exist
\[
Q_{k,l}:L^1([0,1]^d)\to L^1([0,1]^d),
\qquad
C_{k,d}=C(k,d)<\infty,
\qquad
C_q=C_q(k,d)<\infty,
\]
such that
\[
Q_{k,l}g(x)
=
Q_{k,l}\!\left(
g\,\mathbf 1_{\{u\in[0,1]^d:\|u-x\|_\infty\le C_{k,d}/l\}}
\right)(x),
\qquad
g\in L^1([0,1]^d),\ x\in[0,1]^d,
\]
and
\begin{equation}
\label{eq:qi-Linfty-gradient-stability}
\|Q_{k,l}g\|_{L^\infty}
\le
C_q\|g\|_{L^\infty},
\qquad
\|\nabla Q_{k,l}g\|_{L^\infty}
\le
C_q\|\nabla g\|_{L^\infty},
\qquad
g,\partial_1g,\ldots,\partial_dg\in L^\infty([0,1]^d).
\end{equation}
If \(s\in\mathbb N\), \(0\le r\le s\), and \(k\ge s\), then, for
\(C=C(k,s,r,d)\),
\begin{equation}
\label{eq:bspline-Hr-approx-general}
\|Q_{k,l}g-g\|_{H^r}
\le
C l^{-(s-r)}
\|g\|_{H^s},
\qquad
g\in H^s([0,1]^d).
\end{equation}
\end{lemma}
\begin{proof}
We construct \(Q_{k,l}\) explicitly and then verify the displayed bounds.
The construction has three ingredients. First, we take tensor products of
univariate order-\(k\) B-splines on the uniform grid with repeated endpoint
knots. Second, we choose local coefficient functionals so that the resulting
operator reproduces tensor-product polynomials of coordinatewise degree at
most \(k-1\). Third, compact support, uniformly bounded overlap, and the
moment conditions give the \(L^\infty\) and gradient stability estimates,
while polynomial reproduction and locality give the Sobolev approximation
bound by the Bramble--Hilbert lemma\cite{schumaker2007spline}.

Let \(t_i^{(l)}=i/l\) for \(0\le i\le l\), and repeat the endpoint knots by
\[
t_{-k+1}^{(l)}=\cdots=t_{-1}^{(l)}=t_0^{(l)}=0,
\qquad
t_l^{(l)}=t_{l+1}^{(l)}=\cdots=t_{l+k-1}^{(l)}=1.
\]
For \(i=-k+1,\dots,l-1\), define the order-\(k\) univariate B-spline by
\[
N_{l,i}^{(k)}(x):=
(-1)^k
\bigl(t_{i+k}^{(l)}-t_i^{(l)}\bigr)
\bigl[t_i^{(l)},t_{i+1}^{(l)},\dots,t_{i+k}^{(l)}\bigr]
\{(x-t)_+^{k-1}\},
\qquad x\in[0,1].
\]
Here the divided difference is defined recursively as follows. For distinct
knots \(a_0,\dots,a_m\),
\[
[a_0]F:=F(a_0),
\qquad
[a_0,\dots,a_m]F
:=
\frac{[a_1,\dots,a_m]F-[a_0,\dots,a_{m-1}]F}{a_m-a_0},
\]
with the repeated-knot value understood by continuity, equivalently by the
corresponding derivative limit. For interior indices \(0\le i\le l-k\),
this gives the explicit formula
\[
N_{l,i}^{(k)}(x)
=
\frac{l^{k-1}}{(k-1)!}
\sum_{j=0}^{k}
(-1)^j
\binom{k}{j}
\left(x-\frac{i+j}{l}\right)_+^{k-1},
\]
and the boundary splines are determined by the repeated endpoint knots.
For \(\mathbf i=(i_1,\dots,i_d)\in\{-k+1,\dots,l-1\}^d\), set
\[
N_{l,\mathbf i}^{(k)}(x)
:=
\prod_{m=1}^d N_{l,i_m}^{(k)}(x_m),
\qquad x=(x_1,\dots,x_d)\in[0,1]^d.
\]

We now define the local coefficient functionals. For each univariate index
\(i\), let
\[
\omega_{l,i}:=[t_i^{(l)},t_{i+k}^{(l)}]\cap[0,1].
\]
Choose \(\psi_{l,i}^{(k)}\in L^\infty([0,1])\), supported on
\(\omega_{l,i}\), such that
\[
\int_{\omega_{l,i}} u^a\psi_{l,i}^{(k)}(u)\,du
=
\gamma_{i,a},
\qquad
0\le a\le k-1,
\]
where \(\gamma_{i,a}\) is the B-spline coefficient of the monomial \(u^a\),
namely
\[
u^a
=
\sum_{j=-k+1}^{l-1}\gamma_{j,a}N_{l,j}^{(k)}(u),
\qquad u\in[0,1].
\]
Such \(\psi_{l,i}^{(k)}\) exists because this is a finite-dimensional moment
system on \(\omega_{l,i}\). By scaling from the reference knot
configuration, it may be chosen so that
\[
\|\psi_{l,i}^{(k)}\|_{L^1([0,1])}
\le C_k
\]
uniformly in \(l\) and \(i\). For a multi-index \(\mathbf i\), define
\[
\Lambda_{\mathbf i}(g)
:=
\int_{[0,1]^d}
g(u)
\prod_{m=1}^d\psi_{l,i_m}^{(k)}(u_m)\,du .
\]
Thus \(\Lambda_{\mathbf i}(g)\) depends only on \(g\) restricted to
\[
\omega_{l,\mathbf i}:=
\prod_{m=1}^d\omega_{l,i_m},
\]
and
\begin{equation}
\label{eq:qi-coefficient-local-stability}
\operatorname{diam}(\omega_{l,\mathbf i})\le C_{k,d}l^{-1},
\qquad
\Lambda_{\mathbf i}(g)=\Lambda_{\mathbf i}(g\mathbf 1_{\omega_{l,\mathbf i}}),
\qquad
|\Lambda_{\mathbf i}(g)|
\le
C_{k,d}\operatorname*{ess\,sup}_{u\in\omega_{l,\mathbf i}}|g(u)|.
\end{equation}
Define
\[
Q_{k,l}g
:=
\sum_{\mathbf i\in\{-k+1,\dots,l-1\}^d}
\Lambda_{\mathbf i}(g)N_{l,\mathbf i}^{(k)}.
\]
The moment conditions imply
\begin{equation}
\label{eq:qi-polynomial-reproduction}
Q_{k,l}p=p,
\qquad
p\in\operatorname{span}
\left\{
x_1^{a_1}\cdots x_d^{a_d}:0\le a_1,\dots,a_d\le k-1
\right\}.
\end{equation}

The B-splines form a nonnegative partition of unity, have supports of
diameter \(\mathcal O(l^{-1})\), and have uniformly bounded overlap:
\begin{equation}
\label{eq:bspline-local-overlap}
\operatorname{diam}\bigl(\operatorname{supp}N_{l,\mathbf i}^{(k)}\bigr)
\le C_{k,d}l^{-1},
\qquad
\sup_{x\in[0,1]^d}
\sum_{\mathbf i}
\mathbf 1_{\operatorname{supp}N_{l,\mathbf i}^{(k)}}(x)
\le C_{k,d}.
\end{equation}
The coefficient bound above therefore gives
\[
\|Q_{k,l}g\|_{L^\infty}
\le
C_q\|g\|_{L^\infty}.
\]
For the derivative estimate, differentiating the B-spline basis produces a
factor of order \(l\). Since \(Q_{k,l}\) reproduces constants and is local,
neighboring coefficients satisfy
\[
|\Lambda_{\mathbf i}(g)-\Lambda_{\mathbf j}(g)|
\le
C_{k,d}l^{-1}\|\nabla g\|_{L^\infty},
\qquad
|\mathbf i-\mathbf j|_\infty\le1.
\]
This factor \(l^{-1}\) cancels the derivative scale of the basis functions,
and the bounded overlap gives
\[
\|\nabla Q_{k,l}g\|_{L^\infty}
\le
C_q\|\nabla g\|_{L^\infty}.
\]

Finally, for \(k\ge s\), \eqref{eq:qi-polynomial-reproduction} includes all
local polynomials of total degree at most \(s-1\). Together with the locality,
coefficient stability, and coefficient-patch diameter in
\eqref{eq:qi-coefficient-local-stability}, and the spline-support diameter
and finite-overlap bounds in \eqref{eq:bspline-local-overlap}, the hypotheses
of the standard Bramble--Hilbert estimate for local spline
quasi-interpolation are satisfied. Applying this estimate on each patch and
summing by finite overlap \cite{schumaker2007spline} proves
\eqref{eq:bspline-Hr-approx-general}.
\end{proof}

\begin{proposition}[Approximation by the clipped and gradient-bounded sieve]
\label{prop: approximate of clipped NN}
Let \(\Omega\subset\mathbb R^d\) be a bounded domain with sufficiently smooth
boundary, and let \(\mu\) be the spatial mixture measure corresponding to a pair
\((p,q)\) satisfying Assumption~\ref{assm: density}. Let \(s\in\mathbb N\)
with \(1\le s\le4\). Suppose
$
f^\star\in H^s(\mu)
$
and
\[
\max\left\{
\|f^\star\|_{L^\infty(\mu)},
\|\nabla f^\star\|_{L^\infty(\mu)}
\right\}
\le M^\star .
\]
Then there exists \(C_\Omega\ge1\), depending only on \(\Omega\), the density
bounds in Assumption~\ref{assm: density}, and the fixed dimension \(d\), such
that the following holds. Choose \(M>C_\Omega M^\star\), and choose \(T_M\)
so that \(T_M(u)=u\) for all \(|u|\le C_\Omega M^\star\). For
\[
L=\mathcal O(1),
\qquad
W=\mathcal O(N),
\qquad
S=\mathcal O(N),
\qquad
B=\mathcal O(N),
\]
there exists
\[
\bar f_N\in\mathcal F_M(L,W,S,B)
\]
such that
\begin{equation}
\label{eq:clipped-approx-H1-final}
\|\bar f_N-f^\star\|_{H^1(\mu)}^2
\lesssim
N^{-\frac{2(s-1)}{d}}
\|f^\star\|_{H^s(\mu)}^2 .
\end{equation}
\end{proposition}

\begin{proof}
Under Assumption~\ref{assm: density}, the weighted Sobolev norm induced by
\(\mu\) is equivalent to the usual Sobolev norm on \(\Omega\). Hence it
suffices to prove the approximation result in the usual \(H^1(\Omega)\) norm.

Let
\[
\mathcal C_\Omega
=
\prod_{j=1}^d[\alpha_j,\alpha_j+\rho_\Omega]
\]
be a closed axis-aligned cube containing \(\overline\Omega\), with
\(\rho_\Omega>0\). Since \(\Omega\) has sufficiently smooth boundary, choose a
bounded linear extension operator
\[
\mathcal E:H^s(\Omega)\to H^s(\mathcal C_\Omega)
\]
such that, for a constant \(C_{\mathcal E}\ge1\),
\[
\|\mathcal E f\|_{H^s(\mathcal C_\Omega)}
\le
C_{\mathcal E}\|f\|_{H^s(\Omega)}
\]
and
\[
\max\left\{
\|\mathcal E f\|_{L^\infty(\mathcal C_\Omega)},
\|\nabla(\mathcal E f)\|_{L^\infty(\mathcal C_\Omega)}
\right\}
\le
C_{\mathcal E}
\max\left\{
\|f\|_{L^\infty(\Omega)},
\|\nabla f\|_{L^\infty(\Omega)}
\right\}.
\]
Define the affine map
\[
a_\Omega(x)
:=
\left(
\frac{x_1-\alpha_1}{\rho_\Omega},
\ldots,
\frac{x_d-\alpha_d}{\rho_\Omega}
\right),
\qquad x\in \mathcal C_\Omega .
\]
Then \(a_\Omega\) is a bijection from \(\mathcal C_\Omega\) to \([0,1]^d\)
with linear part \(R_\Omega=\rho_\Omega^{-1}I_d\), and hence \(R_\Omega\) is
invertible. Moreover,
\[
a_\Omega^{-1}(z)
=
(\alpha_1+\rho_\Omega z_1,\ldots,\alpha_d+\rho_\Omega z_d),
\qquad z\in[0,1]^d .
\]
Let \(C_q\) be the stability constant in
Lemma~\ref{lem:spline-stability} for cubic splines, and set
\[
C_\Omega
:=
\max\left\{
1,\,
C_qC_{\mathcal E},\,
\|R_\Omega\|_{\mathrm{op}}C_q
\|R_\Omega^{-1}\|_{\mathrm{op}}C_{\mathcal E}
\right\}.
\]
All constants below may depend on \(\Omega\), the density bounds in
Assumption~\ref{assm: density}, the extension operator, the affine map
\(a_\Omega\), and the fixed dimension \(d\), but not on \(N\).

Let
\[
l:=\left\lceil N^{1/d}\right\rceil .
\]
Define the spline approximant on \(\Omega\) by
\begin{equation}\label{eq:def of F_N}
    F_N(x)
:=
\left(
Q_{4,l}\bigl((\mathcal E f^\star)\circ a_\Omega^{-1}\bigr)
\right)(a_\Omega(x)),
\qquad x\in\Omega .
\end{equation}
Here \((\mathcal E f^\star)\circ a_\Omega^{-1}\) is a function on
\([0,1]^d\), so \(Q_{4,l}\) is applied on the unit cube.

We first prove the approximation bound. Since \(a_\Omega\) is fixed, Sobolev
norms before and after the affine change of variables are equivalent.
Therefore
\[
\|(\mathcal E f^\star)\circ a_\Omega^{-1}\|_{H^s([0,1]^d)}
\lesssim
\|\mathcal E f^\star\|_{H^s(\mathcal C_\Omega)}
\le
C_{\mathcal E}\|f^\star\|_{H^s(\Omega)}
\lesssim
\|f^\star\|_{H^s(\mu)}.
\]
By Lemma~\ref{lem:spline-stability}, with \(k=4\) and \(r=1\),
\[
\left\|
Q_{4,l}\bigl((\mathcal E f^\star)\circ a_\Omega^{-1}\bigr)
-
(\mathcal E f^\star)\circ a_\Omega^{-1}
\right\|_{H^1([0,1]^d)}
\lesssim
l^{-(s-1)}
\|(\mathcal E f^\star)\circ a_\Omega^{-1}\|_{H^s([0,1]^d)}.
\]
Since \(a_\Omega\) is fixed and \(\mathcal E f^\star=f^\star\) on
\(\Omega\), by the definition of $F_N$ \eqref{eq:def of F_N}, the same change of variables gives
\[
\|F_N-f^\star\|_{H^1(\Omega)}
\lesssim
\left\|
Q_{4,l}\bigl((\mathcal E f^\star)\circ a_\Omega^{-1}\bigr)
-
(\mathcal E f^\star)\circ a_\Omega^{-1}
\right\|_{H^1([0,1]^d)}.
\]
Combining the previous estimates,
\[
\|F_N-f^\star\|_{H^1(\Omega)}
\lesssim
l^{-(s-1)}
\|f^\star\|_{H^s(\mu)}.
\]
Since \(l\asymp N^{1/d}\),
\[
\|F_N-f^\star\|_{H^1(\Omega)}
\lesssim
N^{-\frac{s-1}{d}}
\|f^\star\|_{H^s(\mu)}.
\]
By the norm equivalence between \(H^1(\Omega)\) and \(H^1(\mu)\), this
also implies
\[
\|F_N-f^\star\|_{H^1(\mu)}
\lesssim
N^{-\frac{s-1}{d}}
\|f^\star\|_{H^s(\mu)}.
\]

We next verify the uniform envelope. By the boundedness of the extension
operator,
\[
\|(\mathcal E f^\star)\circ a_\Omega^{-1}\|_{L^\infty([0,1]^d)}
=
\|\mathcal E f^\star\|_{L^\infty(\mathcal C_\Omega)}
\le
C_{\mathcal E}M^\star .
\]
By the \(L^\infty\)-stability of \(Q_{4,l}\) in Lemma \ref{lem:spline-stability},
\[
\|F_N\|_{L^\infty(\Omega)}
\le
\left\|
Q_{4,l}\bigl((\mathcal E f^\star)\circ a_\Omega^{-1}\bigr)
\right\|_{L^\infty([0,1]^d)}
\le
C_qC_{\mathcal E}M^\star
\le
C_\Omega M^\star
<
M.
\]

For the gradient, the chain rule gives
\[
\nabla\bigl((\mathcal E f^\star)\circ a_\Omega^{-1}\bigr)(z)
=
R_\Omega^{-\top}
\nabla(\mathcal E f^\star)(a_\Omega^{-1}(z)).
\]
Hence
\[
\left\|
\nabla\bigl((\mathcal E f^\star)\circ a_\Omega^{-1}\bigr)
\right\|_{L^\infty([0,1]^d)}
\le
\|R_\Omega^{-1}\|_{\mathrm{op}}C_{\mathcal E}M^\star .
\]
Using the gradient stability of \(Q_{4,l}\) in Lemma \ref{lem:spline-stability},
\[
\left\|
\nabla Q_{4,l}\bigl((\mathcal E f^\star)\circ a_\Omega^{-1}\bigr)
\right\|_{L^\infty([0,1]^d)}
\le
C_q
\left\|
\nabla\bigl((\mathcal E f^\star)\circ a_\Omega^{-1}\bigr)
\right\|_{L^\infty([0,1]^d)}.
\]
Since
\[
F_N(x)
=
\left(
Q_{4,l}\bigl((\mathcal E f^\star)\circ a_\Omega^{-1}\bigr)
\right)(a_\Omega(x)),
\]
another application of the chain rule gives
\[
\|\nabla F_N\|_{L^\infty(\Omega)}
\le
\|R_\Omega\|_{\mathrm{op}}
\left\|
\nabla Q_{4,l}\bigl((\mathcal E f^\star)\circ a_\Omega^{-1}\bigr)
\right\|_{L^\infty([0,1]^d)}.
\]
Combining the last three inequalities,
\[
\|\nabla F_N\|_{L^\infty(\Omega)}
\le
\|R_\Omega\|_{\mathrm{op}}
C_q
\|R_\Omega^{-1}\|_{\mathrm{op}}
C_{\mathcal E}M^\star
\le
C_\Omega M^\star
<
M.
\]
Thus the definition of \(C_\Omega\) ensures that both the function value and
the gradient of \(F_N\) are controlled before clipping.

It remains to realize \(F_N\) as a sparse ReLU\(^3\) neural network. The
quasi-interpolant
\[
Q_{4,l}\bigl((\mathcal E f^\star)\circ a_\Omega^{-1}\bigr)
\]
is a linear combination of tensor-product cubic B-splines on \([0,1]^d\).
For interior indices,
\[
N_{l,i}^{(4)}(z)
=
\frac{1}{3!}
\sum_{j=0}^{4}
(-1)^j
\binom{4}{j}
\left(lz-(i+j)\right)_+^3 .
\]
The boundary splines have analogous formulas determined by the repeated
endpoint knots. Hence every univariate cubic B-spline can be represented
exactly by a ReLU\(^3\) subnetwork. Since the dimension \(d\) is fixed,
each tensor-product cubic B-spline can be represented by a fixed-depth
ReLU\(^3\) subnetwork. The number of tensor-product basis functions is
\[
|I_{l,4}|^d=\mathcal O(l^d)=\mathcal O(N).
\]
Composing with the fixed affine map \(a_\Omega\) only modifies the first layer.
Therefore \(F_N\) can be represented by a sparse ReLU\(^3\) network with
\[
L=\mathcal O(1),
\qquad
W=\mathcal O(N),
\qquad
S=\mathcal O(N),
\qquad
B=\mathcal O(N).
\]
The constants hidden in the \(\mathcal O(\cdot)\) notation may depend on
\(\Omega\), \(d\), the spline order, \(M^\star\), and fixed problem
parameters, but not on \(N\).

Finally define
\[
\bar f_N:=T_M\circ F_N .
\]
Since
\[
\|F_N\|_{L^\infty(\Omega)}\le C_\Omega M^\star
\]
and \(T_M(u)=u\) for all \(|u|\le C_\Omega M^\star\), the clipping is inactive
on \(\Omega\). Therefore
\[
\bar f_N=F_N
\qquad\text{on }\Omega .
\]
Consequently,
\[
\|\bar f_N\|_{L^\infty(\mu)}\le M,
\qquad
\|\nabla \bar f_N\|_{L^\infty(\mu)}\le M,
\]
and hence
\[
\bar f_N\in\mathcal F_M(L,W,S,B).
\]
Moreover,
\[
\|\bar f_N-f^\star\|_{H^1(\mu)}
=
\|F_N-f^\star\|_{H^1(\mu)}
\lesssim
N^{-\frac{s-1}{d}}
\|f^\star\|_{H^s(\mu)}.
\]
Squaring both sides gives
\[
\|\bar f_N-f^\star\|_{H^1(\mu)}^2
\lesssim
N^{-\frac{2(s-1)}{d}}
\|f^\star\|_{H^s(\mu)}^2 .
\]
This proves \eqref{eq:clipped-approx-H1-final}.
\end{proof}

\subsection{Proof of Lemma \ref{lem:J-strong-smooth}}

\begin{proof}
We first analyze the cross-entropy loss. Recall that in Section \ref{sec: Theory} we define
\[
  \eta(x):=\PP(Y=1\mid X=x)=\frac{q(x)}{p(x)+q(x)},
  \qquad
  c_{\min}:=\frac{1}{4\cosh^2(M/2)}.
\]
For fixed $x$, define the conditional risk as a function of the scalar logit
$u$:
\[
  G_x(u):=-\eta(x)\log\sigma(u)-(1-\eta(x))\log(1-\sigma(u)).
\]
Here $G_x'$ and $G_x''$ denote ordinary derivatives with respect to $u$.
A direct calculation yields
\[
G_x''(u) = \sigma(u)(1-\sigma(u)) = \frac{1}{4\cosh^2(u/2)}.
\]
Since any $f \in \cH_M$ satisfies $|f(x)| \le M$ almost everywhere, the curvature is uniformly bounded from below and above:
\begin{equation}\label{eq: bounded second derivative}
     c_{\min} \;\le\; G_x''(u) \;\le\; \frac{1}{4}, \quad \forall u \in [-M, M].
\end{equation}
Since $G_x(u)$ is smooth with bounded derivatives on this compact interval, standard arguments from the calculus of variations imply that $L_{\rm CE}$ is Fr\'echet differentiable with $D L_{\rm CE}(g)[h] = \E[G_X'(g(X))h(X)]$.
Applying Taylor's theorem with integral remainder to $G_x$, we obtain:
\[
  G_x(u) - G_x(v) - G_x'(v)(u-v)
  = (u-v)^2 \int_0^1 (1-t) G_x''\big(v+t(u-v)\big)\,\mathrm dt.
\]
Using the bounds on $G_x''$, this implies
\[
  \frac{c_{\min}}{2}(u-v)^2
  \;\le\;
  G_x(u) - G_x(v) - G_x'(v)(u-v)
  \;\le\;
  \frac{1}{8}(u-v)^2.
\]
Integrating with respect to $\mu$ (setting $u=f(x), v=g(x)$) yields the strong convexity and smoothness estimates for $L_{\rm CE}$:
\begin{equation}\label{eq:CE-bounds}
     \frac{c_{\min}}{2}\|f-g\|_{L^2(\mu)}^2
  \;\le\;
  L_{\rm CE}(f)-L_{\rm CE}(g) - D L_{\rm CE}(g)[f-g]
  \;\le\;
  \frac{1}{2}\|f-g\|_{L^2(\mu)}^2.
\end{equation}

We then analyze  the regularizer. The functional $\mathcal{R}(f)=\|\nabla f\|_{L^2(\mu)}^2$ is a standard quadratic form on the Hilbert space $H^1(\mu)$. It is strictly convex and Fr\'echet differentiable with $D\mathcal R(g)[h] = 2\langle \nabla g, \nabla h \rangle_{L^2(\mu)}$. The polarization identity yields the exact expansion:
\begin{equation}\label{eq:Omega-exact}
  \mathcal{R}(f) - \mathcal R(g) - D\mathcal R(g)[f-g] = \|\nabla(f-g)\|_{L^2(\mu)}^2.
\end{equation}

Finally, we put together above results.
Recall $J_\lambda(f) = L_{\rm CE}(f) + \lambda\mathcal{R}(f)$. By linearity, $DJ_\lambda = DL_{\rm CE} + \lambda D\mathcal R$. Combining the lower bound from \eqref{eq:CE-bounds} and the identity \eqref{eq:Omega-exact}, we obtain:
\[
  J_\lambda(f)-J_\lambda(g) - DJ_\lambda(g)[f-g]
  \;\ge\;
  \frac{c_{\min}}{2}\|f-g\|_{L^2(\mu)}^2 + \lambda\|\nabla(f-g)\|_{L^2(\mu)}^2,
\]
which proves \eqref{eq:J-strong-general}. Similarly, combining the upper bound from \eqref{eq:CE-bounds} with \eqref{eq:Omega-exact} yields \eqref{eq:J-smooth}.
\end{proof}

\subsection{Proof of Lemma \ref{lem:bias}}
\begin{proof}
We proceed in three steps.

\proofstep{Step 1: A pointwise calibration inequality along $f_\lambda-f^\star$.}
From \eqref{eq: bounded second derivative} we have
\[
  c_{\min}
  \;\le\;
  G_x''(u)
  \;\le\;
  \frac{1}{4}
  \qquad\text{for all }u\in[-M,M]\text{ and all }x.
\]

By the fundamental theorem of calculus,
\[
  G_x'(f(x)) - G_x'(f^\star(x))
  = \int_0^1 G_x''\big(f^\star(x)+t(f(x)-f^\star(x))\big)\,(f(x)-f^\star(x))\,\mathrm dt.
\]
Since $f^\star$ is the Bayes logit, we have
$G_x'(f^\star(x)) = \sigma(f^\star(x))-\eta(x) = 0$, so
\[
  G_x'(f(x))\,(f(x)-f^\star(x))
  = \Bigg(\int_0^1 G_x''\big(f^\star(x)+t(f(x)-f^\star(x))\big)\,\mathrm dt\Bigg)
    (f(x)-f^\star(x))^2.
\]
Define
\[
  w(x):= \int_0^1 G_x''\big(f^\star(x)+t(f(x)-f^\star(x))\big)\,\mathrm dt.
\]

By the bounds on $G_x''$, we have
$c_{\min}\le w(x)\le 1/4$ for all $x$.
Moreover $G_x'(f(x))=\sigma(f(x))-\eta(x)$, so
\[
  (\sigma(f(x))-\eta(x))\,(f(x)-f^\star(x))
  = w(x)\,(f(x)-f^\star(x))^2.
\]

Taking expectation with respect to $X\sim\mu$ gives
\[
  \E\big[(\sigma(f(X))-\eta(X))\,(f(X)-f^\star(X))\big]
  = \E\big[w(X)\,(f(X)-f^\star(X))^2\big],
\]
and hence
\begin{equation}
  c_{\min}\,\|f-f^\star\|_{L^2(\mu)}^2
  \;\le\;
  \E\big[(\sigma(f(X))-\eta(X))\,(f(X)-f^\star(X))\big]
  \;\le\;
  \frac{1}{4}\,\|f-f^\star\|_{L^2(\mu)}^2.
  \label{eq:graddistance-sandwich}
\end{equation}
This is the desired ``gradient-distance'' inequality along the
direction $f-f^\star$.

\proofstep{Step 2: First-order optimality of $f_\lambda$ and Green's identity.} Since $f_\lambda$ minimizes $J_\lambda(f)$ over the convex set $\mathcal{H}_M$, and noting that $f^\star \in \mathcal{H}_M$ (which holds by the assumption $M \ge M^\star$), $f_\lambda$ satisfies the first-order variational inequality:
\[
  DJ_\lambda(f_\lambda)[f^\star - f_\lambda] \ge 0.
\]
Equivalently, \(DJ_\lambda(f_\lambda)[f_\lambda-f^\star]\le0\), which expands to
\begin{equation}
  \E\big[(\sigma(f_\lambda)-\eta)(f_\lambda-f^\star)\big]
  + 2\lambda\,(\nabla f_\lambda,\nabla(f_\lambda-f^\star))_{L^2(\mu)}
  \;\le\; 0.
  \label{eq:FOC-h=Delta}
\end{equation}

Since \(\nabla f_\lambda=\nabla f^\star+\nabla(f_\lambda-f^\star)\),
\[
  (\nabla f_\lambda,\nabla(f_\lambda-f^\star))_{L^2(\mu)}
  = \|\nabla(f_\lambda-f^\star)\|_{L^2(\mu)}^2
    + (\nabla f^\star,\nabla(f_\lambda-f^\star))_{L^2(\mu)}.
\]
By the Neumann boundary condition in Assumption~\ref{assm: density} and Green's Identity, we have
\[
  -(\nabla(f_\lambda-f^\star),\nabla f^\star)_{L^2(\mu)}
  = (\Delta f^\star + \nabla f^\star\cdot\nabla\log \rho,\;
     f_\lambda-f^\star)_{L^2(\mu)}.
\]
By the definition
\(\mathcal K h:=\Delta h+\nabla h\cdot\nabla\log\rho\),
\(\mathcal K f^\star\in L^2(\mu)\).
Then
\[
  (\nabla f^\star,\nabla(f_\lambda-f^\star))_{L^2(\mu)}
  = - (\mathcal K f^\star,f_\lambda-f^\star)_{L^2(\mu)}.
\]

Substituting into \eqref{eq:FOC-h=Delta} yields
\[
  \E\big[(\sigma(f_\lambda)-\eta)(f_\lambda-f^\star)\big]
  + 2\lambda\|\nabla(f_\lambda-f^\star)\|_{L^2(\mu)}^2
  - 2\lambda(\mathcal K f^\star,f_\lambda-f^\star)_{L^2(\mu)}
  \;\le\; 0,
\]
or equivalently
\begin{equation}
  \E\big[(\sigma(f_\lambda)-\eta)(f_\lambda-f^\star)\big]
  + 2\lambda\|\nabla(f_\lambda-f^\star)\|_{L^2(\mu)}^2
  \;\le\; 2\lambda(\mathcal K f^\star,f_\lambda-f^\star)_{L^2(\mu)}.
  \label{eq:key-equality}
\end{equation}

\proofstep{Step 3: Combining the sandwich inequality and Cauchy--Schwarz.}
Applying the lower bound in
\eqref{eq:graddistance-sandwich} with $f=f_\lambda$,
we obtain
\[
  c_{\min}\,\|f_\lambda-f^\star\|_{L^2(\mu)}^2
  \;\le\;
  \E\big[(\sigma(f_\lambda)-\eta)(f_\lambda-f^\star)\big].
\]
Combining this with \eqref{eq:key-equality}, we get
\begin{equation}
\begin{aligned}
    c_{\min}\,\|f_\lambda-f^\star\|_{L^2(\mu)}^2
  + 2\lambda\|\nabla(f_\lambda-f^\star)\|_{L^2(\mu)}^2
  &\;\le\;
  2\lambda(\mathcal K f^\star,f_\lambda-f^\star)_{L^2(\mu)}\\
  &\le
  2\lambda \,\|\mathcal K f^\star\|_{L^2(\mu)}\,\|f_\lambda-f^\star\|_{L^2(\mu)}.
    \label{eq:bias-ineq-final}
\end{aligned}
\end{equation}
We now extract the desired bounds from
\eqref{eq:bias-ineq-final}.

\smallskip\noindent
\emph{(i) Value bias.}
Dropping the nonnegative gradient term on the left-hand side gives
\[
  c_{\min}\,\|f_\lambda-f^\star\|_{L^2(\mu)}^2
  \le
  2\lambda \,\|\mathcal K f^\star\|_{L^2(\mu)}\,\|f_\lambda-f^\star\|_{L^2(\mu)}.
\]
If \(f_\lambda=f^\star\), the conclusion is trivial.
Otherwise, dividing both sides by \(\|f_\lambda-f^\star\|_{L^2(\mu)}\),
\[
  \|f_\lambda-f^\star\|_{L^2(\mu)}
  \le \frac{2\,\|\mathcal K f^\star\|_{L^2(\mu)}}{c_{\min}}\,\lambda.
\]
Hence
\begin{equation}
  \|f_\lambda-f^\star\|_{L^2(\mu)}^2
  \le \Big(\frac{2\,\|\mathcal K f^\star\|_{L^2(\mu)}}{c_{\min}}\Big)^2\,\lambda^2.
  \label{eq:bias-value}
\end{equation}

\smallskip\noindent
\emph{(ii) Gradient bias.}
Substituting the bound
$\|f_\lambda-f^\star\|_{L^2(\mu)}
   \le \big(2\,\|\mathcal K f^\star\|_{L^2(\mu)}/c_{\min}\big)\lambda$
back into \eqref{eq:bias-ineq-final}, we obtain
\begin{eqnarray*}
     2\lambda\|\nabla(f_\lambda-f^\star)\|_{L^2(\mu)}^2
  &\le&
  2\lambda \,\|\mathcal K f^\star\|_{L^2(\mu)}\,\|f_\lambda-f^\star\|_{L^2(\mu)}\\
 & \le&
  2\lambda \,\|\mathcal K f^\star\|_{L^2(\mu)}\cdot
  \frac{2\,\|\mathcal K f^\star\|_{L^2(\mu)}}{c_{\min}}\,\lambda
  = \frac{4\,\|\mathcal K f^\star\|_{L^2(\mu)}^2}{c_{\min}}\,\lambda^2.
\end{eqnarray*}
Dividing by $2\lambda$ (recall $\lambda>0$) yields
\begin{equation}
  \|\nabla(f_\lambda-f^\star)\|_{L^2(\mu)}^2
  \le \frac{2\,\|\mathcal K f^\star\|_{L^2(\mu)}^2}{c_{\min}}\,\lambda.
  \label{eq:bias-gradient}
\end{equation}

Using the explicit value \(c_{\min}=1/(4\cosh^2(M/2))\), we can write
\[
  \Big(\frac{2\,\|\mathcal K f^\star\|_{L^2(\mu)}}{c_{\min}}\Big)^2
  = 64\,\|\mathcal K f^\star\|_{L^2(\mu)}^2\,\cosh^4(M/2),
  \qquad
  \frac{2\,\|\mathcal K f^\star\|_{L^2(\mu)}^2}{c_{\min}}
  = 8\,\|\mathcal K f^\star\|_{L^2(\mu)}^2\,\cosh^2(M/2).
\]
Thus we may take
$\beta = 64\,\|\mathcal K f^\star\|_{L^2(\mu)}^2\,\cosh^4\!\big(\tfrac{M}{2}\big)$ so that
\[
  \max\Big\{
        64\,\|\mathcal K f^\star\|_{L^2(\mu)}^2\,\cosh^4\!\big(\tfrac{M}{2}\big),\;
        8\,\|\mathcal K f^\star\|_{L^2(\mu)}^2\,\cosh^2\!\big(\tfrac{M}{2}\big)
      \Big\}
  \;\le\beta.
\]

\end{proof}
\subsection{Proof of Theorem~\ref{thm: generalized error}}
\label{sec: proof of generalzied error}

Let \(\{\sigma_i^p\}_{i=1}^n\) and \(\{\sigma_i^q\}_{i=1}^n\) be independent
i.i.d. Rademacher variables, i.e.,
\[
\PP(\sigma_i^p=1)=\PP(\sigma_i^p=-1)
=\PP(\sigma_i^q=1)=\PP(\sigma_i^q=-1)=\frac12,
\qquad i=1,\ldots,n.
\]
For the fixed source--target design, define the
empirical Rademacher complexity of a function class \(\mathcal G\) with respect to
\(\mathcal D=\{(X_i^p,0)\}_{i=1}^n\cup\{(X_i^q,1)\}_{i=1}^n\) by
\[
\mathfrak R_n(\mathcal G;\mathcal D):=
\mathbb E_\sigma
\left[
\sup_{g\in\mathcal G}
\left\{
\frac{1}{2n}\sum_{i=1}^{n}\sigma_i^p g(X_i^p,0)
+
\frac{1}{2n}\sum_{i=1}^{n}\sigma_i^q g(X_i^q,1)
\right\}
\,\middle|\,\mathcal D
\right].
\]
Here \(\mathbb E_\sigma\) denotes expectation with respect to the Rademacher
variables only, conditional on \(\mathcal D\). Similarly,
\(\mathbb E_{\sigma^p}\) and \(\mathbb E_{\sigma^q}\) denote expectations
over the two corresponding Rademacher blocks.

\begin{proof}
Fix \(t>0\).
Define the shifted excess-loss function
\begin{equation}
\label{eq:shifted-excess-loss}
\ell_f(x,y):=
\ell_{\rm CE}(y,f(x))
-
\ell_{\rm CE}(y,f_0(x))
+
\lambda
\left(
\|\nabla f(x)\|_2^2
-
\|\nabla f_0(x)\|_2^2
\right).
\end{equation}

Then, with \(\overline P\) and \(\overline P_n\) defined above,
\[
\overline P\ell_f
=
J_\lambda(f)-J_\lambda(f_0),
\]
and
\[
\overline P_n\ell_f
=
\widehat J_{\lambda,\mathcal D}(f)
-
\widehat J_{\lambda,\mathcal D}(f_0).
\]

\proofstep{Step 1: Local Rademacher complexity bound.}
For \(r>0\), define the localized set
\[
\mathcal U_r:=
\left\{
f\in\mathcal F:
J_\lambda(f)-J_\lambda(f_\lambda)
+
J_\lambda(f_0)-J_\lambda(f_\lambda)
\le r
\right\},
\]
and define the shifted localized loss class
\[
\mathcal G_r:=
\left\{
\ell_f:\Omega\times\{0,1\}\to\mathbb R
\ \middle|\
f\in \mathcal U_r
\right\}.
\]
Here \(\ell_f\) is the shifted excess loss in
\eqref{eq:shifted-excess-loss}.
We claim that for \(r\gtrsim n^{-2}\),
\begin{equation}
\label{eq:shifted-local-rad}
\mathfrak R_n(\mathcal G_r;\mathcal D)
\le
\phi(r):=
C_0\left[
\frac{1}{n}
+
\sqrt{
\frac{S\,3^L r}{n}\log(BWn)
}
\right],
\end{equation}
and \(\phi(4r)\le2\phi(r)\).

We first prove the Lipschitz property of the shifted excess loss. Since
\[
\partial_u\ell_{\rm CE}(y,u)=\sigma(u)-y,
\]
the cross-entropy loss is \(1\)-Lipschitz in \(u\). For any
\(f_1,f_2\in \mathcal U_r\), using the fixed clipped-gradient
envelope
\[
\|f_j\|_{L^\infty(\mu)}\le M,\qquad
\|\nabla f_j\|_{L^\infty(\mu)}\le M,
\qquad j=1,2,
\]
we obtain
\[
\begin{aligned}
|g_{f_1}(x,y)-g_{f_2}(x,y)|
&=
\left|
\ell_{\rm CE}(y,f_1(x))
-
\ell_{\rm CE}(y,f_2(x))
+
\lambda\left(
\|\nabla f_1(x)\|_2^2
-
\|\nabla f_2(x)\|_2^2
\right)
\right|\\
&\le
|f_1(x)-f_2(x)|
+
\lambda
\left(
\|\nabla f_1(x)\|_2+\|\nabla f_2(x)\|_2
\right)
\|\nabla f_1(x)-\nabla f_2(x)\|_2\\
&\le
|f_1(x)-f_2(x)|
+
2M\lambda
\|\nabla f_1(x)-\nabla f_2(x)\|_2.
\end{aligned}
\]
Thus \(\ell_f\) is Lipschitz in \((f,\lambda\nabla f)\), with Lipschitz constant
depending only on \(M\).

Next we relate the shifted localization
\[
J_\lambda(f)-J_\lambda(f_\lambda)
+
J_\lambda(f_0)-J_\lambda(f_\lambda)
\le r
\]
to a radius constraint around \(f_0\). By Lemma~\ref{lem:J-strong-smooth}, for all
\(f\in\mathcal H_M\),
\[
J_\lambda(f)-J_\lambda(f_\lambda)
\ge
\frac{c_{\min}}{2}
\|f-f_\lambda\|_{L^2(\mu)}^2
+
\lambda
\|\nabla(f-f_\lambda)\|_{L^2(\mu)}^2.
\]
Similarly,
\[
J_\lambda(f_0)-J_\lambda(f_\lambda)
\ge
\frac{c_{\min}}{2}
\|f_0-f_\lambda\|_{L^2(\mu)}^2
+
\lambda
\|\nabla(f_0-f_\lambda)\|_{L^2(\mu)}^2.
\]
Therefore, if \(f\in \mathcal U_r\), then
\[
\|f-f_0\|_{L^2(\mu)}^2
\le
2\|f-f_\lambda\|_{L^2(\mu)}^2
+
2\|f_0-f_\lambda\|_{L^2(\mu)}^2
\le
\frac{4}{c_{\min}}r,
\]
and
\[
\|\lambda\nabla(f-f_0)\|_{L^2(\mu)}^2
\le
2\lambda^2\|\nabla(f-f_\lambda)\|_{L^2(\mu)}^2
+
2\lambda^2\|\nabla(f_0-f_\lambda)\|_{L^2(\mu)}^2
\le
2\lambda r
\le
2r,
\]
where we used \(0<\lambda<1\).

Applying the Ledoux--Talagrand contraction lemma
\citep[Theorem~4.12]{ledoux1991probability}  separately to the
\(p\)-sample block and the \(q\)-sample block, and using
\(\sup_f(a_f+b_f)\le \sup_f a_f+\sup_f b_f\), we first have
\[
\begin{aligned}
\mathfrak R_n(\mathcal G_r;\mathcal D)
&=
\mathbb E_\sigma
\sup_{f\in \mathcal U_r}
\left\{
\frac{1}{2n}\sum_{i=1}^n\sigma_i^p \ell_f(X_i^p,0)
+
\frac{1}{2n}\sum_{i=1}^n\sigma_i^q \ell_f(X_i^q,1)
\right\} \\
&\le
\mathbb E_{\sigma^p}
\sup_{f\in \mathcal U_r}
\frac{1}{2n}\sum_{i=1}^n\sigma_i^p \ell_f(X_i^p,0)
+
\mathbb E_{\sigma^q}
\sup_{f\in \mathcal U_r}
\frac{1}{2n}\sum_{i=1}^n\sigma_i^q \ell_f(X_i^q,1).
\end{aligned}
\]
The same
derivative-augmented covering-number argument as in Lemma A.26 of
\cite{lu2021machine} gives, for all \(r\gtrsim n^{-2}\),
\[
\begin{aligned}
\mathfrak R_n(\mathcal G_r;\mathcal D)
&\lesssim
\mathfrak R_n\!\left(
\{f-f_0:f\in \mathcal U_r\};\mathcal D
\right) \\
&\qquad\qquad+
\mathfrak R_n\!\left(
\{\lambda\nabla f-\lambda\nabla f_0:
f\in \mathcal U_r\};\mathcal D
\right)\\
&\lesssim
\frac1n
+
\frac1{\sqrt n}
\int_{1/n}^{C\sqrt r}
\sqrt{
S\Big[
\log(\delta^{-1})+3^L\log(WB)
\Big]
}
\,d\delta\\
&\lesssim
\frac1n
+
\sqrt{
\frac{S r}{n}
\Big[
\log n+3^L\log(WB)
\Big]
}\\
&\lesssim
\frac1n
+
\sqrt{
\frac{S\,3^L r}{n}\log(BWn)
}.
\end{aligned}
\]

The first inequality follows from the Ledoux--Talagrand contraction lemma
applied to the pointwise Lipschitz bound
\[
|g_{f_1}(x,y)-g_{f_2}(x,y)|
\le
|f_1(x)-f_2(x)|
+
2M\lambda\|\nabla f_1(x)-\nabla f_2(x)\|_2.
\]
Indeed, the logistic cross-entropy loss is \(1\)-Lipschitz in the logit
argument, and
\[
\begin{aligned}
\left|
\|\nabla f_1(x)\|_2^2-\|\nabla f_2(x)\|_2^2
\right|
&\le
\left(\|\nabla f_1(x)\|_2+\|\nabla f_2(x)\|_2\right)\\
&\qquad\cdot
\|\nabla f_1(x)-\nabla f_2(x)\|_2\\
&\le
2M\|\nabla f_1(x)-\nabla f_2(x)\|_2.
\end{aligned}
\]
Thus the Rademacher complexity of the shifted loss class is reduced to the
complexities of the localized function class and the localized
\(\lambda\)-scaled gradient class.

The second inequality is the Dudley entropy integral after using the shifted
localization
\[
J_\lambda(f)-J_\lambda(f_\lambda)
+
J_\lambda(f_0)-J_\lambda(f_\lambda)
\le r.
\]
By the strong convexity of \(J_\lambda\), for \(0<\lambda\le 1\),
\[
\|f-f_0\|_{L^2(\mu)}^2
+
\|\lambda\nabla(f-f_0)\|_{L^2(\mu)}^2
\lesssim r.
\]
Hence both localized classes have \(L^2(\mu)\)-radius of order \(\sqrt r\).

Since \(\mu=(p+q)/2\), for every measurable \(h\),
\[
\|h\|_{L^2(p)}^2\le 2\|h\|_{L^2(\mu)}^2,
\qquad
\|h\|_{L^2(q)}^2\le 2\|h\|_{L^2(\mu)}^2.
\]
Therefore the same \(C\sqrt r\) localized radius is valid for the two
empirical blocks \(\{X_i^p\}_{i=1}^n\) and \(\{X_i^q\}_{i=1}^n\).

Combining this radius bound with the derivative-augmented covering-number
estimate for sparse ReLU\(^3\) networks gives
\[
\log \mathcal N
\left(
\delta,\mathcal F,\|\cdot\|_\infty
\right)
\vee
\log \mathcal N
\left(
\delta,\nabla\mathcal F,\|\cdot\|_\infty
\right)
\lesssim
S\Big[
\log(\delta^{-1})+3^L\log(WB)
\Big].
\]
Since the empirical \(L^2\)-metric is dominated by the sup-norm, Dudley's
integral yields the displayed entropy integral with upper limit \(C\sqrt r\).

Finally, evaluating the integral gives
\[
\frac1{\sqrt n}
\int_{1/n}^{C\sqrt r}
\sqrt{
S\Big[
\log(\delta^{-1})+3^L\log(WB)
\Big]
}
\,d\delta
\lesssim
\sqrt{
\frac{S r}{n}
\Big[
\log n+3^L\log(WB)
\Big]
},
\]
which is further bounded by
\[
\sqrt{
\frac{S\,3^L r}{n}\log(BWn)
}.
\]
This is exactly the sub-root function \(\phi\) stated in
Theorem~\ref{thm: generalized error}:
\[
\phi(r):=
C_0\left[
\frac{1}{n}
+
\sqrt{
\frac{S\,3^L r}{n}\log(BWn)
}
\right].
\]
This function is sub-root: it is nonnegative, nondecreasing, and
\(\phi(r)/\sqrt r\) is nonincreasing on \(r>0\). Consequently, its critical
radius satisfies
\[
r^\star
\lesssim
\frac1n+
\frac{S\,3^L\log(BWn)}{n}.
\]
In particular, when \(L=\mathcal{O}(1)\), \(S=\mathcal{O}(N)\), and \(B,W=\mathcal{O}(N)\), we obtain
\[
r^\star
\lesssim
\frac{N(\log N+\log n)}{n}.
\]

\proofstep{Step 2: Peeling and the normalized empirical process.}
The local Rademacher bound in Step 1 controls the class only on each localized shell
\[
\mathcal G_s
=
\left\{
\ell_f:
J_\lambda(f)-J_\lambda(f_\lambda)
+
J_\lambda(f_0)-J_\lambda(f_\lambda)
\le s
\right\}.
\]
Here and below, \(\ell_f\) denotes the shifted loss defined in
\eqref{eq:shifted-excess-loss}.
To turn these shell-wise bounds into a uniform bound over the whole class
\(\mathcal F\), we introduce the normalized class
\[
\overline{\mathcal G}_r:=
\left\{
\widehat{\ell}_f:
\widehat{\ell}_f(x,y)
=
\frac{\ell_f(x,y)}
{J_\lambda(f)-J_\lambda(f_\lambda)
+
J_\lambda(f_0)-J_\lambda(f_\lambda)+r},
\ f\in\mathcal F
\right\}.
\]
The numerator uses the same shifted loss \(\ell_f\) from
\eqref{eq:shifted-excess-loss}.
The denominator is positive because both excess risks are nonnegative and \(r>0\).
The role of the denominator is to rescale each function according to its own
local excess radius. Thus, functions lying farther from the population
minimizer are placed in larger shells and are normalized more strongly.

Applying the Peeling Lemma (Lemma A.7 in \cite{lu2021machine}), with the
localization functional
\[
f\mapsto
J_\lambda(f)-J_\lambda(f_\lambda)
+
J_\lambda(f_0)-J_\lambda(f_\lambda),
\]
and using the sub-root bound
\[
\mathfrak R_n(\mathcal G_s;\mathcal D)\le \phi(s),
\]
we obtain
\[
\mathfrak R_n(\overline{\mathcal G}_r;\mathcal D)
\le
\frac{4\phi(r)}{r}.
\]
In words, peeling converts the local Rademacher bounds on the shells
\[
\left\{
J_\lambda(f)-J_\lambda(f_\lambda)
+
J_\lambda(f_0)-J_\lambda(f_\lambda)
\le s
\right\}
\]
into a single Rademacher bound for the normalized global class.

We now pass from the normalized loss class to the centered normalized empirical
process. Define
\[
\widetilde{\mathcal G}_r:=
\left\{
\widetilde{\ell}_f:\Omega\times\{0,1\}\to\mathbb R
\ \middle|\
\widetilde{\ell}_f(x,y)
=
\frac{\overline P\ell_f-\ell_f(x,y)}
{J_\lambda(f)-J_\lambda(f_\lambda)
+
J_\lambda(f_0)-J_\lambda(f_\lambda)+r},
\ f\in\mathcal F
\right\}.
\]
This is the centered version of the normalized shifted loss in
\eqref{eq:shifted-excess-loss}.
Let \(\{X_i^{p\prime}\}_{i=1}^n\) and
\(\{X_i^{q\prime}\}_{i=1}^n\) be auxiliary independent copies, independent of
the data, with \(X_i^{p\prime}\sim p\) and \(X_i^{q\prime}\sim q\).
By the Symmetrization Lemma (Lemma A.3 in \cite{lu2021machine})
 applied separately to the \(p\)-sample
and \(q\)-sample blocks and then summed,
\[
\begin{aligned}
\sup_{\widetilde g\in\widetilde{\mathcal G}_r}
\mathbb E
\left[
\frac1{2n}\sum_{i=1}^{n}
\widetilde g(X_i^{p\prime},0)
+
\frac1{2n}\sum_{i=1}^{n}
\widetilde g(X_i^{q\prime},1)
\right]
&\le
2\mathfrak R_n(\overline{\mathcal G}_r;\mathcal D)  \\
&\le
\frac{8\phi(r)}{r}.
\end{aligned}
\]

Therefore,
\begin{equation}
\label{eq:shifted-expectation-bound}
\sup_{\widetilde g\in\widetilde{\mathcal G}_r}
\mathbb E
\left[
\frac1{2n}\sum_{i=1}^{n}
\widetilde g(X_i^{p\prime},0)
+
\frac1{2n}\sum_{i=1}^{n}
\widetilde g(X_i^{q\prime},1)
\right]
\le
\frac{8\phi(r)}{r}.
\end{equation}

\proofstep{Step 3: Verifying the Talagrand conditions.}
For any \(f\in\mathcal F\), since both \(f\) and \(f_0\) lie in the clipped
and gradient-bounded sieve,
\[
\|f\|_\infty,\|f_0\|_\infty,
\|\nabla f\|_\infty,\|\nabla f_0\|_\infty
\le M.
\]
Hence, for every \((x,y)\),
\[
\begin{aligned}
|\ell_f(x,y)|
&\le
|\ell_{\rm CE}(y,f(x))|
+
|\ell_{\rm CE}(y,f_0(x))|
+
\lambda
\left(
\|\nabla f(x)\|_2^2+\|\nabla f_0(x)\|_2^2
\right)\\
&\le
2(\log(1+e^M)+M)+2\lambda M^2\\
&\le
2(\log(1+e^M)+M)+2M^2
=:M_\infty.
\end{aligned}
\]
Therefore
\begin{equation}
\label{eq:shifted-infty-bound}
\|\widetilde{\ell}_f\|_\infty
\le
\frac{2M_\infty}{r}
=:\beta_{\rm Tal}.
\end{equation}

We next bound the second moment. From the Lipschitz estimate in Step 1,
\[
|\ell_f(x,y)|
\le
|f(x)-f_0(x)|
+
2M\lambda
\|\nabla f(x)-\nabla f_0(x)\|_2.
\]
Thus
\[
\overline P\ell_f^2
\le
C_0
\left(
\|f-f_0\|_{L^2(\mu)}^2
+
\|\lambda\nabla(f-f_0)\|_{L^2(\mu)}^2
\right).
\]
Using the strong-convexity bounds from Step 1,
\[
\overline P\ell_f^2
\le
C_0
\left(
J_\lambda(f)-J_\lambda(f_\lambda)
+
J_\lambda(f_0)-J_\lambda(f_\lambda)
\right).
\]
Consequently,
\begin{eqnarray*}
    \overline P\widetilde{\ell}_f^2
&=&
\frac{\overline P[(\ell_f-\overline P\ell_f)^2]}
{\left(
J_\lambda(f)-J_\lambda(f_\lambda)
+
J_\lambda(f_0)-J_\lambda(f_\lambda)+r
\right)^2}\\
&\le&
\frac{\overline P\ell_f^2}
{\left(
J_\lambda(f)-J_\lambda(f_\lambda)
+
J_\lambda(f_0)-J_\lambda(f_\lambda)+r
\right)^2}
\le
\frac{C_0}{r}
=:\sigma_{\rm Tal}^2.
\end{eqnarray*}
Moreover,
\[
\overline P\widetilde{\ell}_f=0.
\]

\proofstep{Step 4: Talagrand concentration and choosing the radius.}
Apply Talagrand's inequality to the independent product variables
\((X_i^p,X_i^q)\), with the normalized centered function
\[
(x^p,x^q)
\mapsto
\frac12\widetilde g(x^p,0)+\frac12\widetilde g(x^q,1).
\]
With probability at least \(1-e^{-t}\),
\[
\begin{aligned}
\sup_{\widetilde g\in\widetilde{\mathcal G}_r}
\overline P_n\widetilde g
&\le
2
\sup_{\widetilde g\in\widetilde{\mathcal G}_r}
\mathbb E
\left[
\frac1{2n}\sum_{i=1}^{n}\widetilde g(X_i^{p\prime},0)
+
\frac1{2n}\sum_{i=1}^{n}\widetilde g(X_i^{q\prime},1)
\right]
+
\sqrt{\frac{2t\sigma_{\rm Tal}^2}{n}}
+
\frac{2t\beta_{\rm Tal}}{n}\\
&\le
\frac{16\phi(r)}{r}
+
C_0\sqrt{\frac{t}{nr}}
+
\frac{C_0t}{nr}
=:\psi(r).
\end{aligned}
\]
Choose
\[
r_0:=
C_0'
\max\left\{
r^\star,\frac{t}{n}
\right\},
\]

where \(C_0'\) is chosen so that
\[
C_0'
\ge
\max\left\{128^2,\,36C_0^2,\,6C_0\right\}.
\]
Since \(\phi\) is sub-root and
\(r^\star\) is its critical radius, \(r\mapsto \phi(r)/\sqrt r\) is nonincreasing and
\(\phi(r^\star)=r^\star\). Hence \(r_0\ge C_0'r^\star\) implies
\[
\frac{\phi(r_0)}{r_0}
\le
\left(\frac{r^\star}{r_0}\right)^{1/2}
\le
\frac{1}{\sqrt{C_0'}}.
\]
Thus the first term satisfies
\[
\frac{16\phi(r_0)}{r_0}\le\frac18.
\]
Moreover, \(r_0\ge C_0't/n\), and therefore
\[
{C_0}\sqrt{\frac{t}{nr_0}}\le\frac16,
\qquad
\frac{{C_0}t}{nr_0}\le\frac16,
\]
where the two inequalities follow respectively from
\(C_0'\ge36C_0^2\) and
\(C_0'\ge6C_0\).

Therefore,
\[
\psi(r_0)
\le
\frac18+\frac16+\frac16
<
\frac12.
\]

\proofstep{Step 5: Concluding the shifted oracle bound.}
Pick \(r=r_0\). On the event above, for every \(f\in\mathcal F\),
\[
\frac{
\left[
J_\lambda(f)-\widehat J_{\lambda,\mathcal D}(f)
\right]
-
\left[
J_\lambda(f_0)-\widehat J_{\lambda,\mathcal D}(f_0)
\right]
}{
J_\lambda(f)-J_\lambda(f_\lambda)
+
J_\lambda(f_0)-J_\lambda(f_\lambda)+r_0
}
=
\overline P_n\widetilde{\ell}_f
\le
\frac12,
\]
which implies
\[
\left[
J_\lambda(f)-\widehat J_{\lambda,\mathcal D}(f)
\right]
-
\left[
J_\lambda(f_0)-\widehat J_{\lambda,\mathcal D}(f_0)
\right]
\le
\frac12\left[J_\lambda(f)-J_\lambda(f_\lambda)\right]
+
\frac12\left[J_\lambda(f_0)-J_\lambda(f_\lambda)\right]
+
\frac12r_0.
\]

Now take \(f=\widehat f_{\lambda,\mathcal F}\). By empirical optimality,
\[
\widehat J_{\lambda,\mathcal D}(\widehat f_{\lambda,\mathcal F})
\le
\widehat J_{\lambda,\mathcal D}(f_0).
\]
Therefore
\[
\begin{aligned}
J_\lambda(\widehat f_{\lambda,\mathcal F})
-
J_\lambda(f_\lambda)
&\le
\left[
J_\lambda(\widehat f_{\lambda,\mathcal F})
-
\widehat J_{\lambda,\mathcal D}(\widehat f_{\lambda,\mathcal F})
\right]
-
\left[
J_\lambda(f_0)
-
\widehat J_{\lambda,\mathcal D}(f_0)
\right]
+
J_\lambda(f_0)-J_\lambda(f_\lambda)\\
&\le
\frac12\left[
J_\lambda(\widehat f_{\lambda,\mathcal F})-J_\lambda(f_\lambda)
\right]
+
\frac32\left[
J_\lambda(f_0)-J_\lambda(f_\lambda)
\right]
+
\frac12r_0.
\end{aligned}
\]
Rearranging gives
\[
J_\lambda(\widehat f_{\lambda,\mathcal F})-J_\lambda(f_\lambda)
\le
3\left[J_\lambda(f_0)-J_\lambda(f_\lambda)\right]+r_0.
\]
Since \(r_0\lesssim \max\{r^\star,t/n\}\), we conclude that
\[
J_\lambda(\widehat f_{\lambda,\mathcal F})
-
J_\lambda(f_\lambda)
\lesssim
J_\lambda(f_0)-J_\lambda(f_\lambda)+r^\star+\frac{t}{n}.
\]
This proves the theorem.
\end{proof}
\subsection{Proof of Theorem \ref{thm:classification-oracle}}\label{sec: proof of upper bound}
\begin{proof}
Let
\[
\mathcal F:=\mathcal F_M(L_n,W_n,S_n,B_n).
\]
By Proposition~\ref{prop: approximate of clipped NN}, there exists
\(\bar f_N\in\mathcal F\) and a constant \(C_1\), independent of
\(n,N,\lambda\), such that
\begin{equation}
\label{eq:fbarN-approx-fstar}
\|\bar f_N-f^\star\|_{H^1(\mu)}^2
\le
C_1
N^{-\frac{2(s-1)}{d}}.
\end{equation}
Define
\[
\delta_N:=N^{-\frac{s-1}{d}}.
\]
Then
\[
\|\bar f_N-f^\star\|_{L^2(\mu)}
\le C_1^{1/2}\delta_N,
\qquad
\|\nabla(\bar f_N-f^\star)\|_{L^2(\mu)}
\le C_1^{1/2}\delta_N.
\]

We first bound the comparator excess \(J_\lambda(\bar f_N)-J_\lambda(f_\lambda)\).
Since \(f_\lambda\) minimizes \(J_\lambda\), this quantity is nonnegative.
Moreover,
\[
J_\lambda(\bar f_N)-J_\lambda(f_\lambda)
\le
J_\lambda(\bar f_N)-J_\lambda(f^\star)
+
J_\lambda(f^\star)-J_\lambda(f_\lambda).
\]

\proofstep{Step 1: Bound \(J_\lambda(\bar f_N)-J_\lambda(f^\star)\).}
Since \(f^\star\) minimizes the population cross-entropy risk and
\[
\partial_u^2 \ell_{\rm CE}(y,u)
=
\sigma(u)(1-\sigma(u))
\le \frac14,
\]
we have
\[
L_{\rm CE}(\bar f_N)-L_{\rm CE}(f^\star)
\le
\frac18\|\bar f_N-f^\star\|_{L^2(\mu)}^2
\le
\frac18 C_1\delta_N^2.
\]
For the Sobolev penalty, using
\[
\|\nabla \bar f_N\|_{L^\infty(\mu)}\le M,
\qquad
\|\nabla f^\star\|_{L^\infty(\mu)}\le M,
\]
we obtain
\[
\begin{aligned}
\lambda
\left|
\|\nabla \bar f_N\|_{L^2(\mu)}^2
-
\|\nabla f^\star\|_{L^2(\mu)}^2
\right|
&\le
\lambda
\int
\left(
\|\nabla \bar f_N\|_2+\|\nabla f^\star\|_2
\right)
\|\nabla(\bar f_N-f^\star)\|_2\,d\mu\\
&\le
2M\lambda
\|\nabla(\bar f_N-f^\star)\|_{L^2(\mu)}\\
&\le
2M C_1^{1/2}\lambda\delta_N.
\end{aligned}
\]
Therefore
\begin{equation}
\label{eq:J-fbarN-minus-fstar}
J_\lambda(\bar f_N)-J_\lambda(f^\star)
\le
\frac18 C_1\delta_N^2
+
2M C_1^{1/2}\lambda\delta_N.
\end{equation}

\proofstep{Step 2: Bound \(J_\lambda(f^\star)-J_\lambda(f_\lambda)\).}
Since \(f^\star\) minimizes \(L_{\rm CE}\),
\[
L_{\rm CE}(f^\star)-L_{\rm CE}(f_\lambda)\le 0.
\]
Thus
\[
J_\lambda(f^\star)-J_\lambda(f_\lambda)
\le
\lambda
\left(
\|\nabla f^\star\|_{L^2(\mu)}^2
-
\|\nabla f_\lambda\|_{L^2(\mu)}^2
\right).
\]
Then
\[
\|\nabla f^\star\|_{L^2(\mu)}^2
-
\|\nabla f_\lambda\|_{L^2(\mu)}^2
=
-2\langle\nabla f^\star,\nabla(f_\lambda-f^\star)\rangle_{L^2(\mu)}
-
\|\nabla(f_\lambda-f^\star)\|_{L^2(\mu)}^2.
\]
Hence
\[
J_\lambda(f^\star)-J_\lambda(f_\lambda)
\le
2\lambda
\left|
\langle\nabla f^\star,\nabla(f_\lambda-f^\star)\rangle_{L^2(\mu)}
\right|.
\]
Using the same Green identity as in Lemma~\ref{lem:bias},
\[
\left|
\langle\nabla f^\star,\nabla(f_\lambda-f^\star)\rangle_{L^2(\mu)}
\right|
=
\left|
\langle \mathcal K f^\star,f_\lambda-f^\star\rangle_{L^2(\mu)}
\right|
\le
\|\mathcal K f^\star\|_{L^2(\mu)}
\|f_\lambda-f^\star\|_{L^2(\mu)}.
\]
By Lemma~\ref{lem:bias},
\[
\|f_\lambda-f^\star\|_{L^2(\mu)}
\le
\sqrt{\beta}\lambda.
\]
Therefore
\begin{equation}
\label{eq:J-fstar-minus-flambda}
J_\lambda(f^\star)-J_\lambda(f_\lambda)
\le
2\|\mathcal K f^\star\|_{L^2(\mu)}\sqrt{\beta}\,\lambda^2.
\end{equation}

Combining \eqref{eq:J-fbarN-minus-fstar} and
\eqref{eq:J-fstar-minus-flambda}, we obtain
\begin{equation}
\label{eq:comparator-excess-bound}
J_\lambda(\bar f_N)-J_\lambda(f_\lambda)
\le
\frac18 C_1\delta_N^2
+
2M C_1^{1/2}\lambda\delta_N
+
2\|\mathcal K f^\star\|_{L^2(\mu)}\sqrt{\beta}\lambda^2.
\end{equation}

\proofstep{Step 3: Apply the shifted oracle inequality.}
Apply Theorem~\ref{thm: generalized error} with comparator \(\bar f_N\)
and \(t=2\log n\). There exists a constant \(C_2\), independent of
\(n,N,\lambda\), such that with probability at least \(1-n^{-2}\),
\[
J_\lambda(\widehat f_{\lambda,\mathcal F})
-
J_\lambda(f_\lambda)
\le
C_2
\left(
J_\lambda(\bar f_N)-J_\lambda(f_\lambda)+r^\star+\frac{\log n}{n}
\right).
\]
Using \eqref{eq:comparator-excess-bound}, we get
\begin{equation}
\label{eq:excess-after-clean-oracle}
J_\lambda(\widehat f_{\lambda,\mathcal F})
-
J_\lambda(f_\lambda)
\le
C_2
\left[
\frac18 C_1\delta_N^2
+
2M C_1^{1/2}\lambda\delta_N
+
2\|\mathcal K f^\star\|_{L^2(\mu)}\sqrt{\beta}\lambda^2
+
r^\star+\frac{\log n}{n}
\right].
\end{equation}

\proofstep{Step 4: Convert excess risk to \(L^2\) and gradient errors.}
By Lemma~\ref{lem:J-strong-smooth}, for any \(f,g\in\mathcal H_M\),
\[
J_\lambda(f)-J_\lambda(g)-DJ_\lambda(g)[f-g]
\ge
\frac{c_{\min}}{2}
\|f-g\|_{L^2(\mu)}^2
+
\lambda
\|\nabla(f-g)\|_{L^2(\mu)}^2.
\]
Since \(f_\lambda\) is the population minimizer, the first-order optimality
condition gives
\[
DJ_\lambda(f_\lambda)
[
\widehat f_{\lambda,\mathcal F}-f_\lambda
]
\ge 0.
\]
Therefore
\[
\frac{c_{\min}}{2}
\|\widehat f_{\lambda,\mathcal F}-f_\lambda\|_{L^2(\mu)}^2
+
\lambda
\|\nabla(\widehat f_{\lambda,\mathcal F}-f_\lambda)\|_{L^2(\mu)}^2
\le
J_\lambda(\widehat f_{\lambda,\mathcal F})
-
J_\lambda(f_\lambda).
\]
Combining this with \eqref{eq:excess-after-clean-oracle}, we obtain
\begin{equation}
\label{eq:value-error-flambda-explicit}
\|\widehat f_{\lambda,\mathcal F}-f_\lambda\|_{L^2(\mu)}^2
\le
\frac{2C_2}{c_{\min}}
\left[
\frac18 C_1\delta_N^2
+
2M C_1^{1/2}\lambda\delta_N
+
2\|\mathcal K f^\star\|_{L^2(\mu)}\sqrt{\beta}\lambda^2
+
r^\star+\frac{\log n}{n}
\right],
\end{equation}
and
\begin{equation}
\label{eq:grad-error-flambda-explicit}
\|\nabla(\widehat f_{\lambda,\mathcal F}-f_\lambda)\|_{L^2(\mu)}^2
\le
\frac{C_2}{\lambda}
\left[
\frac18 C_1\delta_N^2
+
2M C_1^{1/2}\lambda\delta_N
+
2\|\mathcal K f^\star\|_{L^2(\mu)}\sqrt{\beta}\lambda^2
+
r^\star+\frac{\log n}{n}
\right].
\end{equation}

\proofstep{Step 5: Add the regularization bias.}
By Lemma~\ref{lem:bias},
\[
\|f_\lambda-f^\star\|_{L^2(\mu)}^2
\le
\beta\lambda^2,
\qquad
\|\nabla(f_\lambda-f^\star)\|_{L^2(\mu)}^2
\le
\beta\lambda.
\]
Using the triangle inequality,
\[
\|\widehat f_{\lambda,\mathcal F}-f^\star\|_{L^2(\mu)}^2
\le
2
\|\widehat f_{\lambda,\mathcal F}-f_\lambda\|_{L^2(\mu)}^2
+
2
\|f_\lambda-f^\star\|_{L^2(\mu)}^2,
\]
and
\[
\|\nabla(\widehat f_{\lambda,\mathcal F}-f^\star)\|_{L^2(\mu)}^2
\le
2
\|\nabla(\widehat f_{\lambda,\mathcal F}-f_\lambda)\|_{L^2(\mu)}^2
+
2
\|\nabla(f_\lambda-f^\star)\|_{L^2(\mu)}^2.
\]
Therefore
\begin{equation}
\label{eq:H1-final-before-choice-explicit}
\begin{aligned}
&\|\widehat f_{\lambda,\mathcal F}-f^\star\|_{H^1(\mu)}^2\\
&\le
\left(
\frac{4C_2}{c_{\min}}
+
\frac{2C_2}{\lambda}
\right)
\left[
\frac18 C_1\delta_N^2
+
2M C_1^{1/2}\lambda\delta_N
+
2\|\mathcal K f^\star\|_{L^2(\mu)}\sqrt{\beta}\lambda^2
+
r^\star+\frac{\log n}{n}
\right] \\
&\quad
+
2\beta\lambda^2
+
2\beta\lambda.
\end{aligned}
\end{equation}

\proofstep{Step 6: Choose \(\lambda\) and \(N\).}
The critical radius satisfies
\[
r^\star
\le
C_3
\frac{N(\log N+\log n)}{n},
\]
where \(C_3\) depends only on the fixed sieve envelope \(M\) and
fixed depth \(L=\mathcal{O}(1)\).

Choose
\[
\lambda\asymp\delta_N=N^{-\frac{s-1}{d}}.
\]
Since \(c_{\min}\), \(M\), \(\|\mathcal K f^\star\|_{L^2(\mu)}\), \(\beta\),
\(C_1\), \(C_2\), and \(C_3\) are independent of \(n\), the leading terms in
\eqref{eq:H1-final-before-choice-explicit} are bounded by
\[
C
\left[
\delta_N
+
\frac{r^\star}{\delta_N}
+
\frac{\log n}{n\delta_N}
\right].
\]
Since \(r^\star\gtrsim N/n\) up to logarithmic factors, the term
\[
\frac{\log n}{n\delta_N}
\]
is dominated by \(r^\star/\delta_N\). Hence
\[
\|\widehat f_{\lambda,\mathcal F}-f^\star\|_{H^1(\mu)}^2
\le
C
\left[
N^{-\frac{s-1}{d}}
+
\frac{
N^{1+\frac{s-1}{d}}(\log N+\log n)
}{n}
\right].
\]
Balancing the two terms gives
\[
N^{-\frac{s-1}{d}}
\asymp
\frac{N^{1+\frac{s-1}{d}}}{n},
\]
and therefore
\[
N\asymp n^{\frac{d}{d+2s-2}}.
\]
Consequently,
\[
\lambda
\asymp
N^{-\frac{s-1}{d}}
\asymp
n^{-\frac{s-1}{d+2s-2}}.
\]
Substituting this choice yields
\[
\|\widehat f_{\lambda,\mathcal F}-f^\star\|_{H^1(\mu)}^2
\le
C
n^{-\frac{s-1}{d+2s-2}}\log n.
\]
This proves the theorem.
\end{proof}

\section{Proof of Theorem~\ref{thm: minimax lower bound}: Minimax lower bound}
\label{sec:proof_minimax}
\begin{proof}
\proofstep{Step 1: Local packing on \(\mathcal C_{\rm pair}\) with enough separation.}
We construct a finite subset of \(\mathcal C_{\rm pair}\) whose log-density ratios have
non-trivial \(H^1\)-separation.

Let \(\eta:\R\to\R\) be the one-dimensional \(C^\infty\) bump
\[
  \eta(t):=
  \begin{cases}
    \exp\!\big(-\tfrac{1}{t(1-t)}\big), & t\in(0,1),\\[4pt]
    0, & t\notin(0,1),
  \end{cases}
\]
and define
\[
  \varphi(x):=\prod_{i=1}^d \eta(x_i),
  \qquad x=(x_1,\dots,x_d)\in\R^d.
\]
Then \(\varphi\in C^\infty(\R^d)\), \(\varphi\ge0\),
\(\nabla\varphi\not\equiv0\), and
\(\supp(\varphi)\subset[0,1]^d\).

For an integer \(m\ge1\), choose points
\(\{x_j\}_{j\in\{1,\dots,m\}^d}\subset\R^d\) such that the cubes
\(x_j+[0,(3m)^{-1}]^d\) are disjoint and contained in \((0,1)^d\). Define
\[
  \varphi_{m,j}(x):=\varphi\big(3m(x-x_j)\big).
\]
By the change of variables \(u=3m(x-x_j)\) and standard Sobolev scaling,
\begin{equation}\label{eq:phi-scaling}
  \|\varphi_{m,j}\|_{L^2(\mu_0)}^2 \asymp m^{-d},\qquad
  \|\nabla\varphi_{m,j}\|_{L^2(\mu_0)}^2 \asymp m^{2-d},\qquad
  \|\varphi_{m,j}\|_{H^s(\mu_0)}^2 \asymp m^{2s-d},
\end{equation}
where the implicit constants depend only on \(\varphi\) and the density of
\(\mu_0\).

By the Varshamov--Gilbert bound~\cite{takezawa2005introduction}, there exist
\(\tau^{(1)},\dots,\tau^{(N)}\in\{0,1\}^{m^d}\) such that
\[
  N\ge 2^{m^d/8},
  \qquad
  \|\tau^{(v)}-\tau^{(v')}\|_2^2\ge \frac{m^d}{8}
  \quad\text{for all }v\neq v'.
\]

Fix an amplitude \(a_m>0\). For each \(v\), choose the unique scalar \(c_v\)
such that
\[
  \int_\Omega
  \sigma\!\left(
    a_m\sum_{j\in\{1,\dots,m\}^d}\tau_j^{(v)}\varphi_{m,j}(x)-c_v
  \right)
  d\mu_0(x)
  =
  \frac12,
\]
and set
\[
  f_v(x):=
  a_m\sum_{j\in\{1,\dots,m\}^d}\tau_j^{(v)}\varphi_{m,j}(x)-c_v.
\]
The scalar \(c_v\) exists and is unique because the left-hand side is a
continuous strictly decreasing function of \(c_v\), with limits \(1\) and
\(0\) as \(c_v\to-\infty\) and \(c_v\to+\infty\). Since the bump sum is
nonnegative,
\[
  0\le c_v\le
  a_m\sup_{x\in\Omega}
  \sum_{j\in\{1,\dots,m\}^d}\tau_j^{(v)}\varphi_{m,j}(x)
  \lesssim a_m.
\]

Writing the constant density of \(\mu_0\) as \(\rho_0=|\Omega|^{-1}\), define
\[
  q_v(x):=2\sigma(f_v(x))\rho_0,
  \qquad
  p_v(x):=2(1-\sigma(f_v(x)))\rho_0.
\]
Then
\[
  \int_\Omega q_v(x)\,dx=1,\qquad
  \int_\Omega p_v(x)\,dx=1,\qquad
  \frac{p_v(x)+q_v(x)}2=\rho_0,
\]
and
\[
  \log\frac{q_v(x)}{p_v(x)}=f_v(x).
\]

We next verify that \((p_v,q_v)\in\mathcal C_{\rm pair}\). Since the supports of
\(\varphi_{m,j}\) are disjoint,
\begin{align}
  \sup_{x\in\Omega}|f_v(x)|
  &\lesssim a_m, \label{eq:Linf-fk}\\
  \|f_v\|_{H^s(\mu_0)}^2
  &\lesssim
  a_m^2\sum_{j\in\{1,\dots,m\}^d}
  \|\varphi_{m,j}\|_{H^s(\mu_0)}^2+c_v^2
  \lesssim a_m^2m^{2s}+a_m^2. \label{eq:Hs-fk}
\end{align}
Choosing \(a_m\asymp m^{-s}\) with a sufficiently small implicit constant
ensures that \(\|f_v\|_{H^s(\mu_0)}<\infty\) and
\(\sup_{x\in\Omega}|f_v(x)|\le M\). Since \(f_v\) is bounded, smooth, and
constant in a neighborhood of \(\partial\Omega\), it satisfies condition 3 in
Assumption~\ref{assm: density}. Moreover, \(p_v\) and \(q_v\) are strictly
positive and \(C^2\), and \(\partial_n f_v=0\) on \(\partial\Omega\). Since
\((p_v+q_v)/2=\rho_0\), the Neumann condition also holds.

For \(v\neq v'\), disjoint supports and \eqref{eq:phi-scaling} give
\begin{align*}
  \|\nabla f_v-\nabla f_{v'}\|_{L^2(\mu_0)}^2
  &=
  a_m^2
  \sum_{j\in\{1,\dots,m\}^d}
  \big(\tau_j^{(v)}-\tau_j^{(v')}\big)^2
  \|\nabla\varphi_{m,j}\|_{L^2(\mu_0)}^2 \\
  &\gtrsim
  a_m^2 m^d m^{2-d}
  \asymp m^{-2(s-1)}.
\end{align*}
Thus, for some constant \(c>0\),
\begin{equation}\label{eq:H1-separation-correct}
  \|f_v-f_{v'}\|_{H^1(\mu_0)}^2
  \ge c\,m^{-2(s-1)},
  \qquad v\neq v'.
\end{equation}

\proofstep{Step 2: KL upper bound for the fixed source--target observation law.}
Under the fixed source--target design, candidate \(v\) induces
\(p_v^n\otimes q_v^n\).

For two candidates \(v\) and \(v'\),
\[
\begin{aligned}
  \KL(p_v^n\otimes q_v^n\|p_{v'}^n\otimes q_{v'}^n)
  &=
  n\KL(p_v\|p_{v'})
  +
  n\KL(q_v\|q_{v'})\\
  &=
  2n\int_\Omega
  \KL\!\left(
    {\rm Bern}(\sigma(f_v(x)))
    \,\middle\|\,
    {\rm Bern}(\sigma(f_{v'}(x)))
  \right)d\mu_0(x).
\end{aligned}
\]
For each fixed \(x\),
\[
\begin{aligned}
&\KL\!\left(
    {\rm Bern}(\sigma(f_v(x)))
    \,\middle\|\,
    {\rm Bern}(\sigma(f_{v'}(x)))
  \right)\\
&\quad =
\E_{Y\sim{\rm Bern}(\sigma(f_v(x)))}
\!\left[
  \ell_{\rm CE}(Y,f_{v'}(x))-\ell_{\rm CE}(Y,f_v(x))
\right].
\end{aligned}
\]
The derivative of
\(u\mapsto
\E_{Y\sim{\rm Bern}(\sigma(f_v(x)))}[\ell_{\rm CE}(Y,u)]\)
vanishes at \(u=f_v(x)\), and its second derivative is
\(\sigma(u)(1-\sigma(u))\). Since \(\|f_v\|_\infty\le M\) uniformly in \(v\),
Taylor's theorem gives
\begin{equation}\label{eq:KL-upper}
  \KL(p_v^n\otimes q_v^n\|p_{v'}^n\otimes q_{v'}^n)
  \lesssim
  n\,\|f_v-f_{v'}\|_{L^2(\mu_0)}^2.
\end{equation}
Moreover,
\begin{equation}\label{eq:L2-bound-lower-packing}
\begin{aligned}
  \|f_v-f_{v'}\|_{L^2(\mu_0)}^2
  &\lesssim
  a_m^2
  \Big(
    \sum_{j\in\{1,\dots,m\}^d}
    |\tau_j^{(v)}-\tau_j^{(v')}|^2
  \Big)m^{-d}
  +
  |c_v-c_{v'}|^2\\
  &\lesssim a_m^2
  \asymp m^{-2s}.
\end{aligned}
\end{equation}
Combining \eqref{eq:KL-upper} and \eqref{eq:L2-bound-lower-packing},
\begin{equation}\label{eq:kl-final-bound}
  \KL(p_v^n\otimes q_v^n\|p_{v'}^n\otimes q_{v'}^n)
  \lesssim n\,m^{-2s}.
\end{equation}

\proofstep{Step 3: Local Fano reduction and minimax lower bound.}
Let \(V\) be uniformly distributed on \(\{1,\dots,N\}\). Conditional on
\(V=v\), draw
\(\{X_i^p\}_{i=1}^n\stackrel{\rm i.i.d.}{\sim}p_v\) and
\(\{X_i^q\}_{i=1}^n\stackrel{\rm i.i.d.}{\sim}q_v\), independently. Any
estimator of the form
\[
\widehat f
:=
\psi\!\left(
\{X_i^p\}_{i=1}^n,
\{X_i^q\}_{i=1}^n
\right)
\]
induces the testing rule
\[
  \widehat V:=
  \argmin_{1\le v\le N}
  \big\|
    \widehat f-f_v
  \big\|_{H^1(\mu_0)}.
\]
By \eqref{eq:H1-separation-correct}, if \(\widehat V\neq V\), then
\[
  \|\widehat f-f_V\|_{H^1(\mu_0)}
  \ge
  \frac12\|f_V-f_{\widehat V}\|_{H^1(\mu_0)}
  \gtrsim
  m^{-(s-1)}.
\]
Therefore
\begin{equation}\label{eq:risk-vs-error}
  \sup_{1\le v\le N}
  \E_{p_v^n\otimes q_v^n}
  \big[\|\widehat f-f_v\|_{H^1(\mu_0)}^2\big]
  \gtrsim
  m^{-2(s-1)}\PP(\widehat V\neq V).
\end{equation}

It remains to lower bound \(\PP(\widehat V\neq V)\). Choose
\(m\asymp n^{1/(2s+d)}\), with a sufficiently large implicit constant. Then
\eqref{eq:kl-final-bound} gives
\[
  \KL(p_v^n\otimes q_v^n\|p_{v'}^n\otimes q_{v'}^n)
  \lesssim m^d.
\]
Since \(\log N\asymp m^d\), the local Fano inequality
\cite{takezawa2005introduction} gives
\[
\begin{aligned}
  \PP(\widehat V\neq V)
  &\ge
  1-
  \frac{
    I(V;\{X_i^p\}_{i=1}^n,\{X_i^q\}_{i=1}^n)+\log 2
  }{\log N}\\
  &\ge
  1-
  \frac{
    N^{-2}\sum_{v,v'}
    \KL(p_v^n\otimes q_v^n\|p_{v'}^n\otimes q_{v'}^n)
    +\log 2
  }{\log N}
  \ge \frac12,
\end{aligned}
\]
where the last inequality holds by taking the implicit constant in
\(a_m\asymp m^{-s}\) sufficiently small. Combining this bound with
\eqref{eq:risk-vs-error} and taking the infimum over all estimators \(\psi\),
\[
  \inf_{\psi}
  \sup_{1\le v\le N}
  \E_{p_v^n\otimes q_v^n}
  \big[
    \|\psi(\{X_i^p\}_{i=1}^n,\{X_i^q\}_{i=1}^n)-f_v\|_{H^1(\mu_0)}^2
  \big]
  \gtrsim
  n^{-\frac{2(s-1)}{2s+d}}.
\]
Since all \((p_v,q_v)\) belong to \(\mathcal C_{\rm pair}\), the same lower bound holds
for the minimax risk over \(\mathcal C_{\rm pair}\).
\end{proof}

\section{Proof for the extension to diffusion models}
\label{sec:proof_unbounded}
\paragraph{Appendix \ref{sec:proof_unbounded} roadmap.}
The diffusion-model proof is organized as follows. We first fix the truncated
time-space notation and the projection-rescaling map on
\([t_0,T]\times B_{2R}\). Lemma~\ref{lem:global-linear-growth} collects the
uniform VP-convolution bounds, including sub-Gaussian tails, linear growth of
\(f^\star\), and polynomial bounds on time-space derivatives. The global
envelope choice then ensures that \(f^\star\) and the pulled-back network
class fit inside the clipped Sobolev class. Lemma~\ref{lem:C2-existence-uniqueness}
proves existence, uniqueness, and clipping of the truncated population
minimizer, and Lemma~\ref{lem:C3-bias} controls its regularization bias on
the truncated cylinder. Lemma~\ref{lem:C4-rescaled-approx} gives the
\(R\)-dependent approximation comparator, Lemma~\ref{lem:C6-vector-entropy}
bounds the entropy of the pulled-back vector class, and
Lemma~\ref{lem:C3-aniso-local-complexity} gives the comparator-centered
anisotropic oracle inequality. Finally, the proof of
Theorem~\ref{thm: unbounded_convergence_nohop} combines the truncation event,
the oracle bound, the comparator gap, and the tail estimate, then chooses
\(N,\lambda,R\) to obtain the stated rate.

Throughout this appendix, \(\alpha_t\) denotes the VP coefficient, and we write
\[
\kappa:=\frac{s-1}{d+1+2s-2}.
\]
for the nonparametric rate exponent. In this section we prove Theorem~\ref{thm: unbounded_convergence_nohop}.
We work on the bounded time-space cylinder \([t_0,T]\times B_{2R}\), where
\[
B_{2R}:=\{x\in\mathbb R^d:\|x\|\le 2R\},
\]
and use a projection-rescaling argument so that the neural network is trained
on the fixed compact spatial domain \(B_1\).

\subsection{Notation}
Throughout Appendix~\ref{sec:proof_unbounded}, for each truncation radius \(R\ge1\), we use the
\(R\)-dependent clipped-gradient time-space class
\[
\mathcal F_R:=\mathcal F^{(d+1)}_{M_R}(L,W,S,B),
\qquad
M_R:=C_M(1+R),
\]
on the fixed input domain
\[
[t_0,T]\times B_1\subset\mathbb R^{d+1}.
\]
The input variable is \((t,\bar x)\), and time is treated as an ordinary
network input. By definition of \(\mathcal F_R\),
\[
\|f\|_{L^\infty([t_0,T]\times B_1)}\le M_R,
\qquad
\|\nabla_{\bar x}f\|_{L^\infty([t_0,T]\times B_1)}\le M_R,
\qquad f\in\mathcal F_R.
\]

\paragraph{Time-dependent label-augmented model and truncated population risk.}
For each \(t\in[t_0,T]\), define a joint distribution
\(\mu_{p,q,t}\) on \(\mathbb R^d\times\{0,1\}\) by
\begin{equation}
\label{eq:mu-pqt-def-unbounded}
Y\sim{\rm Bernoulli}(1/2),
\qquad
X\mid(Y=1)\sim q_t,
\qquad
X\mid(Y=0)\sim p_t.
\end{equation}
We write
\[
\eta_t(x):=\mathbb P(Y=1\mid X=x)
=
\frac{q_t(x)}{p_t(x)+q_t(x)}.
\]

Recall the pointwise \(H^1\)-loss density in \eqref{eq:norm}. For any
function \(h:[t_0,T]\times\mathbb R^d\to\mathbb R\), define the truncated
population functional on \([t_0,T]\times B_{2R}\) by
\begin{equation}
\label{eq:pop_trunc_J_2R_new}
J_{\lambda,2R}(h):=
\int_{t_0}^T
\mathbb E_{(X,Y)\sim\mu_{p,q,t}}
\Big[
\mathbf 1\{X\in B_{2R}\}
\Big(
\ell_{\rm CE}(Y,h(t,X))+\lambda\|\nabla_x h(t,X)\|^2
\Big)
\Big]\,dt.
\end{equation}

\paragraph{Projection-rescaling and pulled-back class.}
Recall that the projected-rescaled input is
\(\bar x=\operatorname{Proj}_{2R}(x)/(2R)\in B_1\).  For a function
\(f:[t_0,T]\times B_1\to\mathbb R\), define its pull-back to
\([t_0,T]\times B_{2R}\) by
\begin{equation}
\label{eq:pullback-operator}
(\pi_R f)(t,x):=
f\!\left(t,\frac{x}{2R}\right),
\qquad (t,x)\in [t_0,T]\times B_{2R}.
\end{equation}
For a function \(h:[t_0,T]\times B_{2R}\to\mathbb R\), define
\[
(\pi_R^{-1}h)(t,\bar x):=h(t,2R\bar x),
\qquad (t,\bar x)\in [t_0,T]\times B_1.
\]

Then
\[
\nabla_x(\pi_R f)(t,x)
=
\frac{1}{2R}
\nabla_{\bar x}f\!\left(t,\frac{x}{2R}\right).
\]
Similarly, for \((t,x)\in [t_0,T]\times B_{2R}\),
\[
f^\star(t,x)
=
(\pi_R^{-1}f^\star)\!\left(t,\frac{x}{2R}\right),
\qquad
\nabla_x f^\star(t,x)
=
\frac{1}{2R}
\nabla_{\bar x}(\pi_R^{-1}f^\star)\!\left(t,\frac{x}{2R}\right).
\]

Define the pulled-back network class
\[
\mathcal F_{2R}:=
\{\pi_R f:f\in\mathcal F_R\}.
\]
By the envelope definition of \(\mathcal F_R\), every
\(h=\pi_R f\in\mathcal F_{2R}\) satisfies
\[
\|h\|_{L^\infty([t_0,T]\times B_{2R})}\le M_R,
\qquad
\|\nabla_x h\|_{L^\infty([t_0,T]\times B_{2R})}
\le
\frac{M_R}{2R}
\le C_M,
\qquad R\ge1.
\]

\paragraph{Norms.}

On \([t_0,T]\times B_{2R}\), unless otherwise stated, we use
the \(\mu_t\)-weighted spatial energy norm
\begin{equation}
\label{eq:weighted-energy-norm-2R}
\|g\|_{L^2([t_0,T];H^1(B_{2R},\mu_t))}^2:=
\int_{t_0}^T\int_{B_{2R}}
\Big(|g(t,x)|^2+\|\nabla_xg(t,x)\|^2\Big)\,d\mu_t(x)\,dt.
\end{equation}
Here \(H^1(B_{2R},\mu_t)\) is only in the spatial variable \(x\).

We also use its \(L^2\)-part
\[
\|g\|_{L^2([t_0,T]\times B_{2R},\mu_t)}^2:=
\int_{t_0}^T\int_{B_{2R}} |g(t,x)|^2\,d\mu_t(x)\,dt.
\]
Since \(p_t\) and \(q_t\) are Gaussian convolutions,
\(\rho_t=(p_t+q_t)/2\) is smooth and strictly positive on the compact
cylinder \([t_0,T]\times B_{2R}\); hence
\(\|\cdot\|_{L^2([t_0,T];H^1(B_{2R},\mu_t))}\) is equivalent, for
fixed \(R\), to the corresponding Lebesgue spatial energy norm on
\([t_0,T]\times B_{2R}\).
On the fixed cylinder \([t_0,T]\times B_1\), we use the Lebesgue norms
\[
\|u\|_{L^2([t_0,T];H^1(B_1))}^2:=
\int_{t_0}^T\int_{B_1}
\Big(
|u(t,\bar x)|^2+\|\nabla_{\bar x}u(t,\bar x)\|^2
\Big)\,d\bar x\,dt,
\]
and, for integer \(s\ge1\),
\[
\|u\|_{H^s([t_0,T]\times B_1)}^2:=
\sum_{a+|\gamma|\le s}
\int_{t_0}^T\int_{B_1}
\big|
\partial_t^aD_{\bar x}^{\gamma}u(t,\bar x)
\big|^2\,d\bar x\,dt.
\]
Here \(L^2([t_0,T];H^1(B_1))\) is spatial Sobolev only in
\(\bar x\), while \(H^s([t_0,T]\times B_1)\) is the full
time-space Sobolev norm of order \(s\) on the same fixed cylinder.

\paragraph{Bounded Sobolev classes on truncated cylinders.}
For \(R>0\) and \(M>0\), define the bounded Sobolev class on \([t_0,T]\times B_{2R}\) by
\begin{eqnarray}
\label{eq:Htilde-M-2R-def}
\widetilde{\mathcal H}_{M,2R}&:=&
\Big\{
h:[t_0,T]\times B_{2R}\to\mathbb R
\ \Big|\
\|h\|_{L^2([t_0,T];H^1(B_{2R},\mu_t))}<\infty,\nonumber\\
&& \qquad \ |h(t,x)|\le M\ \text{for a.e. }(t,x)\in [t_0,T]\times B_{2R}
\Big\}.
\end{eqnarray}

\subsection{Uniform bounds on the pulled-back class and auxiliary lemmas}
\label{subsec:uniform-bound}

We first collect the
 Gaussian-convolution bounds.
\begin{lemma}[Basic consequences of the regular VP assumption]
\label{lem:global-linear-growth}
Suppose Assumption~\ref{assm:unbounded_diffusion} holds.
Set
$
R_1:=\omega_0R_0.
$
Then, if \(X_0\sim p_0\) and \(Y_0\sim q_0\), the laws of
\(\alpha_tX_0\) and \(\alpha_tY_0\) are supported in \(B_{R_1}\).

Moreover, there exist constants \(C,c,K<\infty\), depending only on
\[
d,s,t_0,T,R_0,\underline\sigma,\overline\sigma,\omega_0,
\]
such that the following bounds hold uniformly over \(t\in[t_0,T]\).

First, the mixture measure has sub-Gaussian tails: for every \(u\ge0\),
\begin{equation}
\label{eq:vp_tail_bound}
\mu_t(\{x:\|x\|>u\})
\le
C\exp\{-c(u-R_1)_+^2\}.
\end{equation}
More generally, for every \(m\ge0\),
\begin{equation}
\label{eq:vp_tail_moment_bound}
\int_{\{\|x\|>u\}}(1+\|x\|)^m\,d\mu_t(x)
\le
C_m(1+u)^m\exp\{-c_m(u-R_1)_+^2\}.
\end{equation}

Second, the log-density ratio has at most linear growth:
\begin{equation}
\label{eq:vp_fstar_growth}
|f^\star(t,x)|
\le
C(1+\|x\|).
\end{equation}
Its spatial gradient is uniformly bounded:
\begin{equation}
\label{eq:vp_fstar_grad_bound}
\|\nabla_x f^\star(t,x)\|
\le
C.
\end{equation}

Finally, for every integer \(a\ge0\) and multi-index
\(b\in\mathbb N^d\) satisfying \(a+|b|\le s\)
\begin{equation}
\label{eq:vp_time_space_derivative_growth}
\big|\partial_t^aD_x^b f^\star(t,x)\big|
\le
C(1+\|x\|)^K.
\end{equation}
\end{lemma}
\begin{proof}
The support statement follows directly from the VP representation. If
\(X_0\sim p_0\) and \(Y_0\sim q_0\), then
\(X_0,Y_0\in B_{R_0}\) almost surely. Since
\[
|\alpha_t|\le \omega_0,
\]
we have
\[
\|\alpha_tX_0\|\le \omega_0R_0,
\qquad
\|\alpha_tY_0\|\le \omega_0R_0
\quad\text{a.s.},
\]
so both laws are supported in
$
B_{\omega_0R_0}
=
B_{R_1}.
$

We first prove the tail bounds. If \(Z_0\) is supported in
\(B_{R_0}\), then
\[
Z_t=\alpha_tZ_0+\sigma_t\xi
\]
satisfies
\[
\|Z_t\|
\le
R_1+\overline\sigma\|\xi\|.
\]
Therefore
\[
\mathbb P(\|Z_t\|>u)
\le
\mathbb P\left(
\|\xi\|>\frac{(u-R_1)_+}{\overline\sigma}
\right)
\le
C\exp\{-c(u-R_1)_+^2\}.
\]
This holds for both \(p_t\) and \(q_t\), hence for their mixture \(\mu_t\).
The polynomial tail-moment bound
\eqref{eq:vp_tail_moment_bound} follows by integration by parts. Indeed, for
any nonnegative random variable \(W\),
\[
\mathbb E\big[(1+W)^m\mathbf 1\{W>u\}\big]
\le
(1+u)^m\mathbb P(W>u)
+
 m\int_u^\infty (1+r)^{m-1}\mathbb P(W>r)\,dr,
\]
and applying the preceding sub-Gaussian tail bound gives
\eqref{eq:vp_tail_moment_bound}, after adjusting constants.

Next we study the density ratio. Since \(p_t\) is
the law of \(\alpha_tX_0+\sigma_t\xi\) with \(X_0\sim p_0\) and
\(\xi\sim\mathcal N(0,I)\), we have the Gaussian mixture representation
\[
p_t(x)
=
\int
(2\pi\sigma_t^2)^{-d/2}
\exp\left(
-\frac{\|x-\alpha_t z\|^2}{2\sigma_t^2}
\right)
\,dp_0(z),
\]
and similarly \(q_t\) is represented by replacing \(p_0\) with \(q_0\).
Expanding the square gives
\[
p_t(x)
=
(2\pi\sigma_t^2)^{-d/2}
\exp\left(-\frac{\|x\|^2}{2\sigma_t^2}\right)
I_p(t,x),
\]
where
\[
I_p(t,x):=
\int
\exp\left(
\frac{\alpha_t\langle x,z\rangle}{\sigma_t^2}
-
\frac{\alpha_t^2\|z\|^2}{2\sigma_t^2}
\right)\,dp_0(z).
\]
Similarly,
\[
q_t(x)
=
(2\pi\sigma_t^2)^{-d/2}
\exp\left(-\frac{\|x\|^2}{2\sigma_t^2}\right)
I_q(t,x),
\]
where \(I_q\) is defined with \(q_0\). Hence the common Gaussian factor
cancels in the log-density ratio:
\[
f^\star(t,x)
=
\log I_q(t,x)-\log I_p(t,x).
\]

Since \(z\in B_{R_0}\), \(\alpha_t\) is bounded, and
\(\sigma_t\ge\underline\sigma\), the exponent inside \(I_p\) is bounded above
and below by affine functions of \(\|x\|\). Therefore
\[
|\log I_p(t,x)|\le C(1+\|x\|),
\qquad
|\log I_q(t,x)|\le C(1+\|x\|),
\]
which proves \eqref{eq:vp_fstar_growth}.

For the gradient, differentiating under the integral gives
\[
\nabla_x\log I_p(t,x)
=
\frac{\alpha_t}{\sigma_t^2}
\frac{
\int
z
\exp\left(
\frac{\alpha_t\langle x,z\rangle}{\sigma_t^2}
-
\frac{\alpha_t^2\|z\|^2}{2\sigma_t^2}
\right)\,dp_0(z)
}{
I_p(t,x)
}.
\]
Since the weight in the numerator is nonnegative and \(p_0\) is supported on
\(B_{R_0}\), the ratio above has norm at most \(R_0\). Thus
\[
\|\nabla_x\log I_p(t,x)\|\le C.
\]
The same bound holds for \(I_q\). Therefore
\[
\nabla_x f^\star(t,x)
=
\nabla_x\log I_q(t,x)-\nabla_x\log I_p(t,x)
\]
is uniformly bounded, proving \eqref{eq:vp_fstar_grad_bound}.

It remains to show \eqref{eq:vp_time_space_derivative_growth}. We prove it for \(\log I_p\);
the argument for \(\log I_q\) is identical. Define
\[
\Phi(t,x,z):=
\frac{\alpha_t\langle x,z\rangle}{\sigma_t^2}
-
\frac{\alpha_t^2\|z\|^2}{2\sigma_t^2}.
\]
For every integer \(a\ge0\) and multi-index \(b\in\mathbb N^d\) with
\(a+|b|\le s\), Assumption~\ref{assm:unbounded_diffusion}, compactness of
\([t_0,T]\), and \(\|z\|\le R_0\) imply
\[
\big|\partial_t^aD_x^b\Phi(t,x,z)\big|
\le C_{a,b}(1+\|x\|),
\]
uniformly in \(t\in[t_0,T]\) and \(z\in B_{R_0}\).
We now pass the derivative bound from \(\Phi\) to \(\log I_p\).
Let \(\nu=(a,b)\) be a time--space multi-index and write
\(D^\nu=\partial_t^aD_x^b\). Since \(p_0\) is supported on
\(B_{R_0}\), the bound
\[
|D^\eta\Phi(t,x,z)|\le C_\eta(1+\|x\|),
\qquad 1\le |\eta|\le s,
\]
holds uniformly over \(t\in[t_0,T]\) and \(z\in B_{R_0}\).
Moreover, on every compact set in \(x\), the derivatives
\(D^\eta e^{\Phi(t,x,z)}\), \(|\eta|\le s\), are uniformly bounded on
\([t_0,T]\times K\times B_{R_0}\). Hence differentiation under
the integral is justified by dominated convergence.

For every \(1\le |\nu|\le s\), the multivariate Faà di Bruno formula gives
\[
D^\nu I_p(t,x)
=
\int e^{\Phi(t,x,z)}
\mathfrak B_\nu\!\left(
D^\eta\Phi(t,x,z):0<\eta\le\nu
\right)
\,dp_0(z),
\]
where \(\eta\le\nu\) is componentwise and, for variables
\(\{u_\eta:0<\eta\le\nu\}\),
\[
\mathfrak B_\nu(u_\eta:0<\eta\le\nu)
:=
\nu!
\sum_{\substack{(m_\eta)_{0<\eta\le\nu},\ m_\eta\in\mathbb N_0\\
\sum_{0<\eta\le\nu}m_\eta\eta=\nu}}
\prod_{0<\eta\le\nu}
\frac{1}{m_\eta!}
\left(\frac{u_\eta}{\eta!}\right)^{m_\eta}.
\]
For each admissible family \((m_\eta)_{0<\eta\le\nu}\),
\(\sum_{0<\eta\le\nu}m_\eta|\eta|=|\nu|\), and hence
\(\sum_{0<\eta\le\nu}m_\eta\le |\nu|\). Since
\(|D^\eta\Phi(t,x,z)|\le C_\eta(1+\|x\|)\), each monomial in
\(\mathfrak B_\nu\) is bounded by \(C_\nu(1+\|x\|)^{|\nu|}\), and the
number of monomials depends only on \(\nu\). Therefore
\[
\left|
\mathfrak B_\nu\!\left(
D^\eta\Phi(t,x,z):0<\eta\le\nu
\right)
\right|
\le C_\nu(1+\|x\|)^{|\nu|}.
\]
Since \(I_p(t,x)=\int e^{\Phi(t,x,z)}\,dp_0(z)>0\), it follows directly that
\[
\left|
\frac{D^\nu I_p(t,x)}{I_p(t,x)}
\right|
\le
\frac{
\int e^{\Phi(t,x,z)} C_\nu(1+\|x\|)^{|\nu|}\,dp_0(z)
}{
\int e^{\Phi(t,x,z)}\,dp_0(z)
}
=
C_\nu(1+\|x\|)^{|\nu|}.
\]

Applying the multivariate chain rule to \(\log I_p\), for every
\(1\le |\nu|\le s\), \(D^\nu\log I_p\) is a finite sum of products of
terms of the form
\[
\frac{D^{\eta}I_p(t,x)}{I_p(t,x)},
\qquad 1\le |\eta|\le |\nu|.
\]
Therefore the ratio bound above yields
\[
|D^\nu\log I_p(t,x)|
\le C_\nu(1+\|x\|)^{K_\nu},
\qquad 1\le |\nu|\le s.
\]
Together with the zeroth-order bound for \(\log I_p\) established above,
we obtain
\[
|D^\nu\log I_p(t,x)|
\le C_\nu(1+\|x\|)^{K_\nu},
\qquad |\nu|\le s.
\]
The same argument applies to \(\log I_q\). Since
\(f^\star=\log I_q-\log I_p\), enlarging \(C\) and \(K\) over the
finitely many time--space multi-indices with order at most \(s\) proves
\[
\left|\partial_t^aD_x^b f^\star(t,x)\right|
\le C(1+\|x\|)^K,
\qquad a+|b|\le s.
\]

\end{proof}

\paragraph{Global choice of the envelope constant.}
Throughout Appendix~\ref{sec:proof_unbounded}, the constant \(C_M\) is a
deterministic constant fixed once and for all, independently of
\(R\), \(n\), \(N\), and \(\lambda\). We now specify its choice.

By Lemma~\ref{lem:global-linear-growth}, there exists a constant
\(C_\star<\infty\), independent of \(R\), such that, for all \(R\ge1\),
\begin{equation}
\label{eq:global_target_envelope_rescaled}
\sup_{(t,\bar x)\in [t_0,T]\times B_1}|(\pi_R^{-1}f^\star)(t,\bar x)|
+
\sup_{(t,\bar x)\in [t_0,T]\times B_1}
\|\nabla_{\bar x}(\pi_R^{-1}f^\star)(t,\bar x)\|
\le
C_\star(1+R).
\end{equation}
Indeed,
\((\pi_R^{-1}f^\star)(t,\bar x)=f^\star(t,2R\bar x)\), and
\[
\nabla_{\bar x}(\pi_R^{-1}f^\star)(t,\bar x)
=
2R\,\nabla_x f^\star(t,2R\bar x).
\]
Thus the linear-growth bound for \(f^\star\) and the uniform bound for
\(\nabla_x f^\star\) imply \eqref{eq:global_target_envelope_rescaled}.

Let \(C_q\) denote the stability constant in
\eqref{eq:qi-Linfty-gradient-stability} for the spline
quasi-interpolation operator used in
Proposition~\ref{prop: approximate of clipped NN}. Since the approximation
is performed on the fixed cylinder \([t_0,T]\times B_1\subset\mathbb R^{d+1}\),
\(C_q\) depends only on the dimension, the smoothness index, and the
spline order, and is independent of \(R\), \(n\), \(N\), and \(\lambda\).

We choose \(C_M\) large enough so that
\begin{equation}
\label{eq:CM_global_choice}
C_M
\ge
\max\big\{
2C_\star,\,
2C_qC_\star,\,
C_\star+1
\big\}.
\end{equation}
Set \(M_R=C_M(1+R)\). Then the choice
\eqref{eq:CM_global_choice} has the following consequences, uniformly for
all \(R\ge1\).

First, the rescaled target satisfies
\begin{equation}
\label{eq:target_envelope_by_MR}
\sup_{[t_0,T]\times B_1}|\pi_R^{-1}f^\star|
+
\sup_{[t_0,T]\times B_1}\|\nabla_{\bar x}(\pi_R^{-1}f^\star)\|
\le M_R.
\end{equation}
Equivalently, on the original cylinder \([t_0,T]\times B_{2R}\),
\[
\sup_{(t,x)\in [t_0,T]\times B_{2R}}|f^\star(t,x)|
\le M_R.
\]
Moreover, the choice \(C_M\ge C_\star+1\) provides the strict slack
\begin{equation}
\label{eq:target_strict_slack}
M_R-\sup_{(t,x)\in [t_0,T]\times B_{2R}}|f^\star(t,x)|
\ge 1.
\end{equation}
This slack ensures that \(f^\star\) lies in the interior of the clipping
envelope; in particular, clipping is inactive on \(f^\star\) and on any
approximant that is uniformly within distance \(1\) of \(f^\star\).

Second, the stability constant in \eqref{eq:qi-Linfty-gradient-stability} is
absorbed into the same envelope. Namely,
if \(Q(\pi_R^{-1}f^\star)\) is the spline quasi-interpolant used in
Proposition~\ref{prop: approximate of clipped NN}, then
\begin{equation}
\label{eq:spline_envelope_by_MR}
\|Q(\pi_R^{-1}f^\star)\|_{L^\infty([t_0,T]\times B_1)}
+
\|\nabla_{\bar x}\bigl(Q(\pi_R^{-1}f^\star)\bigr)\|_{L^\infty([t_0,T]\times B_1)}
\le M_R.
\end{equation}
Hence the clipped approximation construction on \([t_0,T]\times B_1\) is compatible with
the class \(\mathcal F_R=\mathcal F_{M_R}^{(d+1)}(L,W,S,B)\) for every
\(R\ge1\), without changing \(C_M\).

Consequently,
\begin{equation}
\label{eq:F2R-subset-Htilde-MR}
\mathcal F_{2R}\subseteq\widetilde{\mathcal H}_{M_R,2R},
\qquad
f^\star\in\widetilde{\mathcal H}_{M_R,2R}.
\end{equation}
Indeed, if \(h=\pi_R f\in\mathcal F_{2R}\), then
\[
\|h\|_{L^\infty([t_0,T]\times B_{2R})}
=
\sup_{(t,x)\in [t_0,T]\times B_{2R}}
\left|f\left(t,\frac{x}{2R}\right)\right|
\le M_R,
\]
because \(f\in\mathcal F_R\). The inclusion
\(f^\star\in\widetilde{\mathcal H}_{M_R,2R}\) follows from
\eqref{eq:target_envelope_by_MR}.

Finally, throughout the truncated-cylinder analysis we use the logistic
curvature lower bound
\begin{equation}
\label{eq:cmin_R_def}
c_{\min}(R):=
\frac{1}{4\cosh^2(M_R/2)}.
\end{equation}
Since \(M_R=C_M(1+R)\), there exists a constant \(C<\infty\), independent of
\(R\), such that
\begin{equation}
\label{eq:cmin_R_growth}
c_{\min}(R)^{-1}\lesssim e^{CR}.
\end{equation}

\begin{lemma}[Existence, uniqueness, and clipping on \(\lbrack t_0,T\rbrack\times B_{2R}\)]
\label{lem:C2-existence-uniqueness}
Suppose Assumption~\ref{assm:unbounded_diffusion} holds, and let
\(C_M\) be fixed by the global choice in
Appendix~\ref{subsec:uniform-bound}. For every \(R\ge1\) and every
\(\lambda>0\), the population functional \(J_{\lambda,2R}\) admits a unique
minimizer over \(\widetilde{\mathcal H}_{M_R,2R}\) with respect to the
finite-energy equivalence class, i.e. unique up to \(dt\,dx\)-a.e.
equivalence. We denote it by
\[
f_{\lambda,2R}:=
\arg\min_{h\in\widetilde{\mathcal H}_{M_R,2R}} J_{\lambda,2R}(h).
\]
Moreover,
\begin{equation}
\label{eq:C2_clipping_bound}
\|f_{\lambda,2R}\|_{L^\infty([t_0,T]\times B_{2R})}
\le
\|f^\star\|_{L^\infty([t_0,T]\times B_{2R})}
\le
M_R.
\end{equation}
\end{lemma}

\begin{proof}

For the direct-method argument, the finite-energy functions are viewed in
the Hilbert space \(L^2([t_0,T];H^1(B_{2R}))\), equipped with the equivalent
weighted energy norm
\(\|\cdot\|_{L^2([t_0,T];H^1(B_{2R},\mu_t))}\) defined in
\eqref{eq:weighted-energy-norm-2R}. For fixed
\(R\), the density \(\rho_t\) is bounded above and below by positive constants
on \([t_0,T]\times B_{2R}\), so weak compactness in this weighted norm is the same as in the
corresponding Lebesgue space.

The feasible set \(\widetilde{\mathcal H}_{M_R,2R}\) is convex. Let \(\{h_m\}_{m\ge1}\subset\widetilde{\mathcal H}_{M_R,2R}\) be a minimizing
sequence for \(J_{\lambda,2R}\). Since \(|h_m|\le M_R\) a.e. on \([t_0,T]\times B_{2R}\),
the \(L^2([t_0,T]\times B_{2R})\)-part is uniformly bounded. Moreover, the Sobolev
penalty in \(J_{\lambda,2R}\) controls
\(\|\nabla_x h_m\|_{L^2([t_0,T]\times B_{2R})}\). Hence \(\{h_m\}\) is bounded in
\(\|\cdot\|_{L^2([t_0,T];H^1(B_{2R},\mu_t))}\). By weak compactness, there exists a subsequence,
still denoted by \(\{h_m\}\), and a finite-energy limit \(h_\lambda\)
such that
\[
h_m\rightharpoonup h_\lambda
\quad\text{weakly in the finite-energy Hilbert space.}
\]
The constraint set is weakly closed. Indeed, weak convergence in the
finite-energy Hilbert space implies weak convergence in \(L^2([t_0,T]\times B_{2R})\), and
the set \(\{h\in L^2([t_0,T]\times B_{2R}): |h|\le M_R\text{ a.e.}\}\) is closed and convex
in \(L^2([t_0,T]\times B_{2R})\), hence weakly closed. Thus
\(h_\lambda\in\widetilde{\mathcal H}_{M_R,2R}\).

The cross-entropy part is convex and weakly lower semicontinuous on the
bounded logit interval \([-M_R,M_R]\), and the quadratic Sobolev penalty is
weakly lower semicontinuous. Therefore
\[
J_{\lambda,2R}(h_\lambda)
\le
\liminf_{m\to\infty}J_{\lambda,2R}(h_m),
\]
so the minimum is attained.

Uniqueness follows from the strict convexity of \(J_{\lambda,2R}\) on
\(\widetilde{\mathcal H}_{M_R,2R}\). Indeed, on the interval \([-M_R,M_R]\), the
conditional logistic risk has curvature bounded below by
\[
c_{\min}(R)
=
\frac{1}{4\cosh^2(M_R/2)}
>0.
\]

Since \(p_t\) and \(q_t\) are Gaussian convolutions, \(\rho_t>0\) on
\(B_{2R}\) for every \(t\in[t_0,T]\). Hence equality in the strict convexity
inequality forces two minimizers to agree \(dt\,dx\)-a.e. on \([t_0,T]\times B_{2R}\). Thus \(J_{\lambda,2R}\) has at most one minimizer. We denote the unique minimizer by \(f_{\lambda,2R}\).

It remains to prove the clipping property. Let
\[
M_R^\star:=\|f^\star\|_{L^\infty([t_0,T]\times B_{2R})}.
\]
By the global choice of \(C_M\), we have \(M_R^\star\le M_R\). For any
\(h\in\widetilde{\mathcal H}_{M_R,2R}\), define its clipped version by
\[
h^{\rm clip}(t,x):=
(-M_R^\star)\vee\bigl(h(t,x)\wedge M_R^\star\bigr).
\]
Then \(h^{\rm clip}\in\widetilde{\mathcal H}_{M_R,2R}\). Indeed,
\[
|h^{\rm clip}|\le M_R^\star\le M_R,
\]
and the chain rule for Lipschitz truncations gives
\[
\|\nabla_x h^{\rm clip}\|_{L^2([t_0,T]\times B_{2R})}
\le
\|\nabla_x h\|_{L^2([t_0,T]\times B_{2R})}.
\]

For fixed \((t,x)\), the conditional cross-entropy risk
\[
G_{t,x}(u):=-\eta_t(x)\log\sigma(u)-(1-\eta_t(x))\log(1-\sigma(u))
\]
is convex and is uniquely minimized at \(u=f^\star(t,x)\), because
\(\eta_t(x)=\sigma(f^\star(t,x))\). Since
\(|f^\star(t,x)|\le M_R^\star\), the projection of any \(u\) onto
\([-M_R^\star,M_R^\star]\) cannot increase \(G_{t,x}(u)\).
 Hence
\[
L_{{\rm CE},2R}(h^{\rm clip})
\le
L_{{\rm CE},2R}(h).
\]
Combining this with the contraction of the spatial-gradient seminorm yields
\[
J_{\lambda,2R}(h^{\rm clip})
\le
J_{\lambda,2R}(h).
\]

Applying this to \(h=f_{\lambda,2R}\), since $f_{\lambda,2R}$ is the minimizer, we find that
\(f_{\lambda,2R}^{\rm clip}\) is also a minimizer. By uniqueness,
\[
f_{\lambda,2R}^{\rm clip}=f_{\lambda,2R}.
\]
Therefore
\[
\|f_{\lambda,2R}\|_{L^\infty([t_0,T]\times B_{2R})}
\le
M_R^\star
=
\|f^\star\|_{L^\infty([t_0,T]\times B_{2R})}
\le
M_R,
\]
which proves \eqref{eq:C2_clipping_bound}.
\end{proof}

\begin{lemma}[Bias on \(\lbrack t_0,T\rbrack\times B_{2R}\)]
\label{lem:C3-bias}
Let \(f_{\lambda,2R}\) be the unique minimizer from
Lemma~\ref{lem:C2-existence-uniqueness}. Define the boundary remainder
\[
\mathfrak b_R:=
2M_R
\int_{t_0}^T
\int_{\partial B_{2R}}
\rho_t(x)|\partial_n f^\star(t,x)|\,dS(x)\,dt.
\]
Then there exists a quantity \(\beta(R)\), independent of \(\lambda\), such that
\begin{equation}
\label{eq:C3_bias_L2_beta}
\|f_{\lambda,2R}-f^\star\|_{L^2([t_0,T]\times B_{2R})}^2
\le
\beta(R)\lambda^2
+
C c_{\min}(R)^{-1}\lambda\mathfrak b_R,
\end{equation}
and
\begin{equation}
\label{eq:C3_bias_grad_beta}
\|\nabla_x(f_{\lambda,2R}-f^\star)\|_{L^2([t_0,T]\times B_{2R})}^2
\le
\beta(R)\lambda
+
C\mathfrak b_R.
\end{equation}
Moreover,
\begin{equation}
\label{eq:C3_beta_boundary_growth}
\mathfrak b_R\le C(1+R)^K e^{-cR^2},
\qquad
\beta(R)\le C e^{CR}(1+R)^K.
\end{equation}
\end{lemma}

\begin{proof}
Since \(f_{\lambda,2R}\) minimizes \(J_{\lambda,2R}\) over the convex set
\(\widetilde{\mathcal H}_{M_R,2R}\), and since
\(f^\star\in\widetilde{\mathcal H}_{M_R,2R}\) by
Lemma~\ref{lem:C2-existence-uniqueness}, convexity implies that, for every
\(\varepsilon\in[0,1]\),
\[
f_{\lambda,2R}
+
\varepsilon(f^\star-f_{\lambda,2R})
\in
\widetilde{\mathcal H}_{M_R,2R}.
\]
Therefore the one-sided directional derivative of \(J_{\lambda,2R}\) at
\(f_{\lambda,2R}\) along the feasible direction
\(f^\star-f_{\lambda,2R}\) is nonnegative:
\[
DJ_{\lambda,2R}(f_{\lambda,2R})
[f^\star-f_{\lambda,2R}]
\ge 0.
\]
Equivalently,
\begin{equation}
\label{eq:C3_first_order}
\int_{t_0}^T\int_{B_{2R}}
(\sigma(f_{\lambda,2R})-\eta_t)(f_{\lambda,2R}-f^\star)\,d\mu_t\,dt
+
2\lambda
\int_{t_0}^T\int_{B_{2R}}
\nabla_x f_{\lambda,2R}\cdot\nabla_x(f_{\lambda,2R}-f^\star)\,d\mu_t\,dt
\le0.
\end{equation}

The logistic curvature lower bound on \([-M_R,M_R]\) gives
\begin{equation}
\label{eq:C3_logistic_curvature}
\int_{t_0}^T\int_{B_{2R}}
(\sigma(f_{\lambda,2R})-\eta_t)(f_{\lambda,2R}-f^\star)\,d\mu_t\,dt
\ge
c_{\min}(R)
\|f_{\lambda,2R}-f^\star\|_{L^2([t_0,T]\times B_{2R})}^2.
\end{equation}
Using
\(\nabla_x f_{\lambda,2R}
=\nabla_x f^\star+\nabla_x(f_{\lambda,2R}-f^\star)\), we obtain from
\eqref{eq:C3_first_order} and \eqref{eq:C3_logistic_curvature}
\begin{eqnarray}
\label{eq:C3_basic_energy}
&&c_{\min}(R)\|f_{\lambda,2R}-f^\star\|_{L^2([t_0,T]\times B_{2R})}^2
+
2\lambda\|\nabla_x(f_{\lambda,2R}-f^\star)\|_{L^2([t_0,T]\times B_{2R})}^2\nonumber\\
&\le&
-2\lambda
\int_{t_0}^T\int_{B_{2R}}
\nabla_x f^\star\cdot\nabla_x(f_{\lambda,2R}-f^\star)\,d\mu_t\,dt.
\end{eqnarray}

For each fixed \(t\), define
\[
\mathcal K_t h(x):=
\Delta_x h(x)
+
\nabla_x h(x)\cdot\nabla_x\log\rho_t(x).
\]
Green's identity on \(B_{2R}\) yields
\begin{equation}
\label{eq:C3_green_identity}
\begin{aligned}
    -\int_{B_{2R}}
\nabla_x f^\star\cdot\nabla_x(f_{\lambda,2R}-f^\star)\,\rho_t\,dx
&=
\int_{B_{2R}}
\bigl(\mathcal K_t f^\star(t,\cdot)\bigr)(f_{\lambda,2R}-f^\star)\,d\mu_t\\
&-
\int_{\partial B_{2R}}
\rho_t(f_{\lambda,2R}-f^\star)\partial_n f^\star\,dS.
\end{aligned}
\end{equation}
If \(f_{\lambda,2R}-f^\star\) is smooth, this is the usual
integration-by-parts formula. For general
\(f_{\lambda,2R}-f^\star\in H^1(B_{2R})\), the identity follows by
approximating \(f_{\lambda,2R}-f^\star\) in \(H^1(B_{2R})\) by smooth functions and
using the trace theorem, since \(f^\star\) and \(\rho_t\) are smooth on
\(B_{2R}\).

Define
\[
\Theta(R):=
\int_{t_0}^T
\|\mathcal K_t f^\star(t,\cdot)\|_{L^2(B_{2R},\mu_t)}^2\,dt.
\]
Since both \(f_{\lambda,2R}\) and \(f^\star\) are bounded by \(M_R\) on
\([t_0,T]\times B_{2R}\), we have
\[
|f_{\lambda,2R}-f^\star|
\le
|f_{\lambda,2R}|+|f^\star|
\le
2M_R.
\]
The trace of \(f_{\lambda,2R}-f^\star\) on \(\partial B_{2R}\) satisfies the same
\(L^\infty\)-bound, because \(f_{\lambda,2R}-f^\star\in H^1(B_{2R})\cap
L^\infty(B_{2R})\). Hence the boundary contribution in
\eqref{eq:C3_green_identity} is bounded by \(\mathfrak b_R\). Combining
\eqref{eq:C3_basic_energy} with \eqref{eq:C3_green_identity} and applying
Cauchy--Schwarz gives
\begin{eqnarray}
\label{eq:C3_energy_with_boundary}
&&c_{\min}(R)\|f_{\lambda,2R}-f^\star\|_{L^2([t_0,T]\times B_{2R})}^2
+
2\lambda\|\nabla_x(f_{\lambda,2R}-f^\star)\|_{L^2([t_0,T]\times B_{2R})}^2\nonumber\\
&\le&
2\lambda\sqrt{\Theta(R)}
\|f_{\lambda,2R}-f^\star\|_{L^2([t_0,T]\times B_{2R})}
+
2\lambda\mathfrak b_R.
\end{eqnarray}
By Young's inequality,
\[
2\lambda\sqrt{\Theta(R)}
\|f_{\lambda,2R}-f^\star\|_{L^2([t_0,T]\times B_{2R})}
\le
\frac{c_{\min}(R)}{2}
\|f_{\lambda,2R}-f^\star\|_{L^2([t_0,T]\times B_{2R})}^2
+
\frac{C\lambda^2}{c_{\min}(R)}\Theta(R),
\]
Therefore,
\[
\frac{c_{\min}(R)}{2}
\|f_{\lambda,2R}-f^\star\|_{L^2([t_0,T]\times B_{2R})}^2
+
2\lambda\|\nabla_x(f_{\lambda,2R}-f^\star)\|_{L^2([t_0,T]\times B_{2R})}^2
\le
\frac{C\lambda^2}{c_{\min}(R)}\Theta(R)
+
2\lambda\mathfrak b_R.
\]
Dropping the gradient term gives
\[
\|f_{\lambda,2R}-f^\star\|_{L^2([t_0,T]\times B_{2R})}^2
\le
C c_{\min}(R)^{-2}\Theta(R)\lambda^2
+
C c_{\min}(R)^{-1}\lambda\mathfrak b_R.
\]
Dropping instead the \(L^2\)-term and dividing by \(2\lambda\) gives
\[
\|\nabla_x(f_{\lambda,2R}-f^\star)\|_{L^2([t_0,T]\times B_{2R})}^2
\le
C c_{\min}(R)^{-1}\Theta(R)\lambda
+
C\mathfrak b_R.
\]
Set
\[
\beta(R):=
C c_{\min}(R)^{-2}\Theta(R).
\]
Since \(c_{\min}(R)\le 1\), this single
choice of \(\beta(R)\) controls both estimates above. Thus
\eqref{eq:C3_bias_L2_beta} and \eqref{eq:C3_bias_grad_beta} follow.

It remains to prove the growth estimates in
\eqref{eq:C3_beta_boundary_growth}. First, we control \(\Theta(R)\). Since
\(p_t\) and \(q_t\) are the laws of
\(\alpha_tX_0+\sigma_t\xi\) and \(\alpha_tY_0+\sigma_t\xi'\), respectively,
with \(X_0\sim p_0\), \(Y_0\sim q_0\), and independent standard Gaussian
noises, the same Gaussian-factorization argument used in
Lemma~\ref{lem:global-linear-growth} gives
\[
\|\nabla_x\log p_t(x)\|
+
\|\nabla_x\log q_t(x)\|
\le
C(1+\|x\|),
\]
uniformly over \(t\in[t_0,T]\). Since
\[
\nabla_x\log\rho_t(x)
=
\frac{p_t(x)}{p_t(x)+q_t(x)}\nabla_x\log p_t(x)
+
\frac{q_t(x)}{p_t(x)+q_t(x)}\nabla_x\log q_t(x),
\]
we have
\[
\|\nabla_x\log\rho_t(x)\|
\le
C(1+\|x\|).
\]
Together with Lemma~\ref{lem:global-linear-growth}, which gives polynomial
growth of the derivatives of \(f^\star\) and a uniform bound on
\(\nabla_x f^\star\), this implies
\[
|(\mathcal K_t f^\star(t,\cdot))(x)|
\le
C(1+\|x\|)^K.
\]
Therefore,
\[
\Theta(R)
=
\int_{t_0}^T
\|\mathcal K_t f^\star(t,\cdot)\|_{L^2(B_{2R},\mu_t)}^2\,dt
\le
C(1+R)^K.
\]
Hence,
\[
\beta(R)
=
C c_{\min}(R)^{-2}\Theta(R)
\le
C e^{CR}(1+R)^K,
\]
where we used \(M_R=C_M(1+R)\) and
\[
c_{\min}(R)
=
\frac{1}{4\cosh^2(M_R/2)}.
\]

Finally, we bound the boundary remainder. By
Lemma~\ref{lem:global-linear-growth},
\[
|\partial_n f^\star(t,x)|
\le
\|\nabla_x f^\star(t,x)\|
\le C.
\]
Moreover, the compact-support VP representation implies the pointwise density
bound
\[
\rho_t(x)\le C\exp\{-c(\|x\|-R_1)_+^2\},
\]
uniformly over \(t\in[t_0,T]\). Hence, on \(\partial B_{2R}\),
\[
\rho_t(x)|\partial_n f^\star(t,x)|
\le
C e^{-cR^2}.
\]
Multiplying by the surface area of \(\partial B_{2R}\), the length of the
time interval, and \(M_R=C_M(1+R)\), we obtain
\[
\mathfrak b_R
\le
C(1+R)^K e^{-cR^2}.
\]
This completes the proof.
\end{proof}

\begin{lemma}[Approximation by the \(R\)-dependent clipped class]
\label{lem:C4-rescaled-approx}
Suppose Assumption~\ref{assm:unbounded_diffusion} holds and
\(s\le4\).
For every \(R\ge1\), there exists a deterministic comparator
\[
f_{0,R}\in\mathcal F_{2R}
\]
such that
\begin{equation}
\label{eq:C4_target_approx_poly}
\|f_{0,R}-f^\star\|_{L^2([t_0,T];H^1(B_{2R},\mu_t))}^2
\le
C(1+R)^K
N^{-\frac{2(s-1)}{d+1}}.
\end{equation}
The comparator \(f_{0,R}\) is fixed before observing the data.
\end{lemma}

\begin{proof}
By the choice of \(C_M\) in \eqref{eq:CM_global_choice}, Proposition~\ref{prop: approximate of clipped NN}
applies on the fixed compact time-space domain
\([t_0,T]\times B_1\subset\mathbb R^{d+1}\). Therefore, there exists
\[
f_R\in\mathcal F_R
\]
such that
\begin{equation}
\label{eq:C4_fixed_domain_clipped_approx}
\|f_R-\pi_R^{-1}f^\star\|_{H^1([t_0,T]\times B_1)}^2
\le
C
\|\pi_R^{-1}f^\star\|_{H^s([t_0,T]\times B_1)}^2
N^{-\frac{2(s-1)}{d+1}}.
\end{equation}
Define
\[
f_{0,R}:=\pi_R f_R\in\mathcal F_{2R}.
\]
Let
\[
e_R:=f_R-\pi_R^{-1}f^\star
\quad\text{on }[t_0,T]\times B_1,
\qquad
e:=f_{0,R}-f^\star
\quad\text{on }[t_0,T]\times B_{2R}.
\]
Then, for \(x\in B_{2R}\),
\[
e(t,x)
=
e_R\!\left(t,\frac{x}{2R}\right),
\qquad
\nabla_x e(t,x)
=
\frac1{2R}
\nabla_{\bar x}e_R\!\left(t,\frac{x}{2R}\right).
\]

By the Gaussian lower-noise bound in Assumption~\ref{assm:unbounded_diffusion},
there exists \(\rho_\star<\infty\), independent of \(R\), such that
\[
\sup_{t\in[t_0,T]}\sup_{x\in\mathbb R^d}\rho_t(x)\le \rho_\star.
\]
Hence
\[
\begin{aligned}
\|e\|_{L^2([t_0,T];H^1(B_{2R},\mu_t))}^2
&=
\int_{t_0}^T\int_{B_{2R}}
\big(|e(t,x)|^2+\|\nabla_xe(t,x)\|^2\big)\rho_t(x)\,dx\,dt  \\
&\le
\rho_\star
\int_{t_0}^T\int_{B_{2R}}
\big(|e(t,x)|^2+\|\nabla_xe(t,x)\|^2\big)\,dx\,dt.
\end{aligned}
\]
Changing variables \(x=2R\bar x\), we get
\[
\int_{B_{2R}}|e(t,x)|^2\,dx
=
(2R)^d\int_{B_1}|e_R(t,\bar x)|^2\,d\bar x,
\]
and
\[
\int_{B_{2R}}\|\nabla_xe(t,x)\|^2\,dx
=
(2R)^{d-2}
\int_{B_1}\|\nabla_{\bar x}e_R(t,\bar x)\|^2\,d\bar x.
\]
Since \(R\ge1\), \((2R)^{d-2}\le (2R)^d\). Therefore,
\[
\|e\|_{L^2([t_0,T];H^1(B_{2R},\mu_t))}^2
\le
C(2R)^d
\|e_R\|_{L^2([t_0,T];H^1(B_1))}^2
\le
C(2R)^d
\|e_R\|_{H^1([t_0,T]\times B_1)}^2.
\]
Combining this with \eqref{eq:C4_fixed_domain_clipped_approx} yields
\[
\|f_{0,R}-f^\star\|_{L^2([t_0,T];H^1(B_{2R},\mu_t))}^2
\le
C(2R)^d
\|\pi_R^{-1}f^\star\|_{H^s([t_0,T]\times B_1)}^2
N^{-\frac{2(s-1)}{d+1}}.
\]

It remains to bound
\(
(2R)^d
\|\pi_R^{-1}f^\star\|_{H^s([t_0,T]\times B_1)}^2.
\) For every \(a+|\gamma|\le s\),
\[
\partial_t^aD_{\bar x}^{\gamma}(\pi_R^{-1}f^\star)(t,\bar x)
=
(2R)^{|\gamma|}
\partial_t^aD_x^\gamma f^\star(t,2R\bar x).
\]
By the derivative-growth bound \eqref{eq:vp_time_space_derivative_growth} in
Lemma~\ref{lem:global-linear-growth},
\[
\big|
\partial_t^aD_x^\gamma f^\star(t,2R\bar x)
\big|
\le
C(1+R)^K,
\qquad
(t,\bar x)\in [t_0,T]\times B_1.
\]
Since \([t_0,T]\times B_1\) has finite Lebesgue volume and \(|\gamma|\le s\), we obtain,
after increasing \(K\) if necessary,
\[
(2R)^d
\|\pi_R^{-1}f^\star\|_{H^s([t_0,T]\times B_1)}^2
\le
C(1+R)^K.
\]
This gives
\eqref{eq:C4_target_approx_poly}.
\end{proof}

\begin{lemma}[Entropy of the pulled-back vector class]
\label{lem:C6-vector-entropy}
Let \(\mathcal F_R=\mathcal F^{(d+1)}_{M_R}(L,W,S,B)\)  be the
\(R\)-dependent clipped-gradient class on \([t_0,T]\times B_1\), with
input dimension \(d+1\), depth \(L=\mathcal O(1)\), sparsity
\(S=\mathcal O(N)\), width \(W\), and weight bound \(B\). Define
\begin{equation}\label{eq: localized-time-dependent-set}
    \mathcal V_{2R}:=
\left\{
\bigl(\pi_R f,\nabla_x(\pi_R f)\bigr): f\in\mathcal F_R
\right\},
\end{equation}
where, for \(x\in B_{2R}\),
\[
\bar x:=\frac{x}{2R}\in B_1,
\qquad
(\pi_R f)(t,x)=f(t,\bar x).
\]
Endow \(\mathcal V_{2R}\) with the norm
\[
\|(u,v)\|_\infty
:=
\|u\|_{L^\infty([t_0,T]\times B_{2R})}
+
\|v\|_{L^\infty([t_0,T]\times B_{2R})}.
\]
Then there exists a constant \(C<\infty\), independent of
\(R,N,n,\lambda\), such that for every \(0<\varepsilon<1\),
\begin{equation}
\label{eq:C6_vector_entropy}
\log\mathcal N
\bigl(
\varepsilon,\mathcal V_{2R},\|\cdot\|_\infty
\bigr)
\le
C N
\log\left(\frac{BWM_R}{\varepsilon}\right).
\end{equation}
\end{lemma}

\begin{proof}
The proof has two steps. First, the pull-back map sends
\([t_0,T]\times B_{2R}\) to the fixed cylinder \([t_0,T]\times B_1\), and
the spatial derivative gains the factor \(1/(2R)\le1\). Hence a cover of
the derivative-augmented class
\(\{(f,\nabla_{\bar x}f):f\in\mathcal F_R\}\) on the fixed cylinder induces
a cover of \(\mathcal V_{2R}\) on the larger cylinder. Second, on the fixed
cylinder, the output and first input derivatives of a fixed-depth sparse
ReLU\(^3\) network depend Lipschitzly on its active parameters. Discretizing
these \(\mathcal O(N)\) active parameters and then taking a union over
sparsity patterns gives the stated entropy bound.

We first reduce the covering problem on \([t_0,T]\times B_{2R}\) to one on
\([t_0,T]\times B_1\). For \(f\in\mathcal F_R\), write
\(\bar x=x/(2R)\). Then
\[
(\pi_R f)(t,x)
=
f(t,\bar x),
\qquad
\nabla_x(\pi_R f)(t,x)
=
\frac1{2R}
\nabla_{\bar x}f(t,\bar x).
\]
Since \(x\in B_{2R}\) implies \(\bar x\in B_1\), and since \(R\ge1\), any
\(\varepsilon\)-cover of
\[
\{(f,\nabla_{\bar x}f):f\in\mathcal F_R\}
\]
on \([t_0,T]\times B_1\) induces an \(\varepsilon\)-cover of
\(\mathcal V_{2R}\) on \([t_0,T]\times B_{2R}\). Indeed, if
\[
\|f-g\|_{L^\infty([t_0,T]\times B_1)}
+
\|\nabla_{\bar x}f-\nabla_{\bar x}g\|_{L^\infty([t_0,T]\times B_1)}
\le \varepsilon,
\]
then
\[
\|\pi_R f-\pi_R g\|_{L^\infty([t_0,T]\times B_{2R})}
+
\|\nabla_x(\pi_R f)-\nabla_x(\pi_R g)\|_{L^\infty([t_0,T]\times B_{2R})}
\le
\varepsilon.
\]
Thus it suffices to control the entropy of the derivative-augmented class
\[
\{(f,\nabla_{\bar x}f):f\in\mathcal F_R\}
\]
on \([t_0,T]\times B_1\).

The class \(\mathcal F_R\) is a clipped sparse ReLU\(^3\) network class with
\(\mathcal O(N)\) active parameters. Fix a sparsity pattern and write
\(f_\theta\) for the network determined by the active parameter vector
\(\theta\). Since ReLU\(^3\) networks are differentiable in the input, and
since their first input derivatives are piecewise-polynomial functions whose
coefficients depend polynomially on the network parameters, the map
\[
\theta\mapsto (f_\theta,\nabla_{\bar x}f_\theta)
\]
is Lipschitz from the active parameter set into the product sup-norm space
on \([t_0,T]\times B_1\). The imposed bounds on the weights, the output, and
the input gradient give a Lipschitz constant polynomial in \(B,W\), and
\(M_R\). Since \(L=\mathcal O(1)\), this polynomial dependence is absorbed
into the logarithmic factor below.

Therefore, on each fixed sparsity pattern, a standard grid discretization of
the \(\mathcal O(N)\) active parameters gives
\[
\log\mathcal N
\bigl(
\varepsilon,
\{(f,\nabla_{\bar x}f):f\in\mathcal F_R
\text{ has the fixed sparsity pattern}\},
\|\cdot\|_\infty
\bigr)
\le
C N
\log\left(\frac{BWM_R}{\varepsilon}\right),
\]
where the norm in this display is the product sup norm on
\([t_0,T]\times B_1\). The number of possible sparsity patterns contributes
at most an additional \(\mathcal O(S\log W)\) term. Since \(S=\mathcal O(N)\),
this term is absorbed into the same bound after increasing the constant.
Hence
\[
\log\mathcal N
\bigl(
\varepsilon,
\{(f,\nabla_{\bar x}f):f\in\mathcal F_R\},
\|\cdot\|_\infty
\bigr)
\le
C N
\log\left(\frac{BWM_R}{\varepsilon}\right),
\]
where again the norm is the product sup norm on \([t_0,T]\times B_1\).
Combining this estimate with the covering reduction above proves
\eqref{eq:C6_vector_entropy}.
\end{proof}

\begin{lemma}[Comparator-centered anisotropic oracle inequality on \(\lbrack t_0,T\rbrack\times B_{2R}\)]
\label{lem:C3-aniso-local-complexity}
Let \(h_0\in\mathcal F_{2R}\) be any deterministic comparator. Define the
scaled truncated empirical objective
\[
\begin{aligned}
\widehat J_{\lambda,\mathcal D_{\rm ext},2R}(h)
&=
\frac{T-t_0}{2n}\sum_{i=1}^n
\left[
\ell_{\rm CE}(0,h(t_i^p,X_i^p))
+
\lambda\|\nabla_xh(t_i^p,X_i^p)\|^2
\right]\mathbf 1\{X_i^p\in B_{2R}\}
\\
&\quad+
\frac{T-t_0}{2n}\sum_{i=1}^n
\left[
\ell_{\rm CE}(1,h(t_i^q,X_i^q))
+
\lambda\|\nabla_xh(t_i^q,X_i^q)\|^2
\right]\mathbf 1\{X_i^q\in B_{2R}\}.
\end{aligned}
\]
Let
\[
\widehat h_R
\in
\arg\min_{h\in\mathcal F_{2R}}
\widehat J_{\lambda,\mathcal D_{\rm ext},2R}(h).
\]
Then there exists a constant \(C<\infty\), independent of
\(R,N,n,\lambda,t\), such that, for every \(t>0\), with probability at least
\(1-e^{-t}\),
\begin{equation}
\label{eq:C6_comparator_oracle}
J_{\lambda,2R}(\widehat h_R)-J_{\lambda,2R}(f_{\lambda,2R})
\le
3\left[
J_{\lambda,2R}(h_0)-J_{\lambda,2R}(f_{\lambda,2R})
\right]
+
C e^{CR}(1+R)^K
\frac{N\log(BWM_R n)+t}{n}.
\end{equation}
\end{lemma}

\begin{proof}
The proof follows the oracle-inequality argument in
Appendix~\ref{sec: proof of generalzied error}, with three changes. First,
we center the empirical process at the fixed comparator \(h_0\), so that
the stochastic term is \(\overline P_{[t_0,T]}\psi_h-\overline P_{n,[t_0,T]}\psi_h\) and empirical optimality gives
\(\overline P_{n,[t_0,T]}\psi_{\widehat h_R}\le0\). Second, because the objective is evaluated
only on \([t_0,T]\times B_{2R}\), all losses carry the truncation indicator
and the fixed time-normalization factor \(T-t_0\). Third, the loss depends on
both \(h\) and \(\nabla_xh\), so the local Rademacher bound is anisotropic:
the cross-entropy curvature controls the value direction, while the Sobolev
penalty controls the spatial-gradient direction. After this localized
complexity bound is established, the normalized-process and Talagrand
concentration step is the same as in Appendix~\ref{sec: proof of generalzied error}.

Recall the bounded-case two-group averages \(\overline P\) and
\(\overline P_n\) introduced in Appendix~\ref{sec: proof of bounded case}. In the present time-dependent
proof, for measurable \(\psi\), define
\[
\overline P_{[t_0,T]}\psi =
\frac12\E_{\tau\sim{\rm Unif}([t_0,T]),\,X\sim p_\tau}[\psi(\tau,X,0)]
+
\frac12\E_{\tau\sim{\rm Unif}([t_0,T]),\,X\sim q_\tau}[\psi(\tau,X,1)],
\]
and
\[
\overline P_{n,[t_0,T]}\psi =
\frac1{2n}\sum_{i=1}^n\psi(t_i^p,X_i^p,0)
+
\frac1{2n}\sum_{i=1}^n\psi(t_i^q,X_i^q,1).
\]

For \(h\in\mathcal F_{2R}\), define the comparator-centered loss
\begin{eqnarray}\label{eq: time-dependent-comparator-centered loss}
    \psi_h(\tau,x,y) &:=&
(T-t_0)\mathbf 1\{x\in B_{2R}\}
\Bigg[
\ell_{\rm CE}(y,h(\tau,x))
-
\ell_{\rm CE}(y,h_0(\tau,x))\nonumber\\
&& \qquad\qquad\qquad \qquad\qquad+
\lambda\big(
\|\nabla_xh(\tau,x)\|^2
-
\|\nabla_xh_0(\tau,x)\|^2
\big)
\Bigg].
\end{eqnarray}
The expression inside the indicator is only relevant for \(x\in B_{2R}\),
so the value of any extension of \(h\) outside \([t_0,T]\times B_{2R}\) is irrelevant.
The factor \(T-t_0\) is inserted so that the empirical and population
objectives are normalized consistently with the unnormalized time integral
in \(J_{\lambda,2R}\). Thus
\[
\overline P_{[t_0,T]}\psi_h
=
J_{\lambda,2R}(h)-J_{\lambda,2R}(h_0),
\]
and
\[
\overline P_{n,[t_0,T]}\psi_h
=
\widehat J_{\lambda,\mathcal D_{\rm ext},2R}(h)
-
\widehat J_{\lambda,\mathcal D_{\rm ext},2R}(h_0).
\]
The fixed factor \(T-t_0\) is absorbed into constants below.

For every \(h=\pi_R f\in\mathcal F_{2R}\),
\[
\|h\|_{L^\infty([t_0,T]\times B_{2R})}\le M_R,
\qquad
\|\nabla_xh\|_{L^\infty([t_0,T]\times B_{2R})}\le \frac{M_R}{2R}.
\]
The same bounds hold for \(h_0\). Moreover, \(M_R/(2R)=\mathcal{O}(1)\) for \(R\ge1\).

For \(h_1,h_2\in\mathcal F_{2R}\), the logistic loss is \(1\)-Lipschitz in
its logit, and
$
\big|\|a\|^2-\|b\|^2\big|
\le
(\|a\|+\|b\|)\|a-b\|.
$
Therefore, pointwise on \(B_{2R}\),
\[
|\psi_{h_1}-\psi_{h_2}|
\le
C\left(
|h_1-h_2|
+
\lambda\frac{M_R}{R}
\|\nabla_xh_1-\nabla_xh_2\|
\right).
\]
For any real-valued class \(\mathcal G\) on
\([t_0,T]\times\mathbb R^d\times\{0,1\}\), recall that
\[
\mathfrak R_n(\mathcal G)
=
\mathbb E_\varepsilon
\sup_{g\in\mathcal G}
\left|
\frac1{2n}\sum_{i=1}^n\varepsilon_i^p g(t_i^p,X_i^p,0)
+
\frac1{2n}\sum_{i=1}^n\varepsilon_i^q g(t_i^q,X_i^q,1)
\right|,
\]
where \(\{\varepsilon_i^p,\varepsilon_i^q\}_{i=1}^n\) are independent
Rademacher variables conditionally on the sampled data. For vector-valued
classes, the same notation denotes the corresponding Rademacher average with
the Euclidean inner product in the summands.
Applying the vector-valued contraction inequality separately on the
\(p\)-sample and \(q\)-sample blocks and using subadditivity of the supremum,
\[
\begin{aligned}
&\mathfrak R_n\left(
\left\{
\psi_h:
\begin{array}{l}
J_{\lambda,2R}(h)-J_{\lambda,2R}(f_{\lambda,2R})\\
\quad+J_{\lambda,2R}(h_0)-J_{\lambda,2R}(f_{\lambda,2R})\le r
\end{array}
\right\}
\right)\\
&\qquad\le
C\mathfrak R_n\left(
\left\{
h-h_0:
\begin{array}{l}
J_{\lambda,2R}(h)-J_{\lambda,2R}(f_{\lambda,2R})\\
\quad+J_{\lambda,2R}(h_0)-J_{\lambda,2R}(f_{\lambda,2R})\le r
\end{array}
\right\}
\right)\\
&\qquad\quad+
C\lambda\frac{M_R}{2R}
\mathfrak R_n\left(
\left\{
\nabla_xh-\nabla_xh_0:
\begin{array}{l}
J_{\lambda,2R}(h)-J_{\lambda,2R}(f_{\lambda,2R})\\
\quad+J_{\lambda,2R}(h_0)-J_{\lambda,2R}(f_{\lambda,2R})\le r
\end{array}
\right\}
\right).
\end{aligned}
\]
No time derivative is involved because the loss depends only on
\((h,\nabla_xh)\).

We next convert the excess-risk localization into the \(L^2\)-radius needed
for Dudley's entropy integral. The two coordinates have different curvature:
the logistic part controls function values through \(c_{\min}(R)\), while the
Sobolev penalty controls spatial gradients through \(\lambda\). The
localization
\[
J_{\lambda,2R}(h)-J_{\lambda,2R}(f_{\lambda,2R})
+
J_{\lambda,2R}(h_0)-J_{\lambda,2R}(f_{\lambda,2R})
\le r
\]
therefore implies an anisotropic local radius for \(h-h_0\). Since
\(f_{\lambda,2R}\) minimizes \(J_{\lambda,2R}\) over the convex set
\(\widetilde{\mathcal H}_{M_R,2R}\), the variational inequality gives
\[
DJ_{\lambda,2R}(f_{\lambda,2R})[h-f_{\lambda,2R}]\ge0,
\qquad
h\in\mathcal F_{2R}\subseteq\widetilde{\mathcal H}_{M_R,2R}.
\]
Hence the strong convexity inequality on \([t_0,T]\times B_{2R}\) yields
\[
J_{\lambda,2R}(h)-J_{\lambda,2R}(f_{\lambda,2R})
\ge
\frac{c_{\min}(R)}{2}
\|h-f_{\lambda,2R}\|_{L^2([t_0,T]\times B_{2R})}^2
+
\lambda
\|\nabla_x(h-f_{\lambda,2R})\|_{L^2([t_0,T]\times B_{2R})}^2,
\]
and the same bound holds for \(h_0\). Combining the two inequalities with
the triangle inequality gives
\[
\|h-h_0\|_{L^2([t_0,T]\times B_{2R})}^2
\le
C c_{\min}(R)^{-1}r,
\qquad
\lambda\|\nabla_x(h-h_0)\|_{L^2([t_0,T]\times B_{2R})}^2
\le
Cr.
\]
Consequently,
\[
\|h-h_0\|_{L^2([t_0,T]\times B_{2R})}^2
+
\left\|
\lambda\frac{M_R}{2R}
(\nabla_xh-\nabla_xh_0)
\right\|_{L^2([t_0,T]\times B_{2R})}^2
\le
C(c_{\min}(R)^{-1}+1)r,
\]
where we used \(M_R/(2R)=\mathcal{O}(1)\) and \(0<\lambda<1\).

The preceding contraction bound shows that it remains to control the
following localized vector class:
\[
\mathcal W_R(r):=
\left\{
\left(h-h_0,\,
\lambda\frac{M_R}{2R}(\nabla_xh-\nabla_xh_0)\right):
\begin{array}{l}
h\in\mathcal F_{2R},\\
J_{\lambda,2R}(h)-J_{\lambda,2R}(f_{\lambda,2R})\\
\quad+J_{\lambda,2R}(h_0)-J_{\lambda,2R}(f_{\lambda,2R})\le r
\end{array}
\right\}.
\]
The last display shows that \(\mathcal W_R(r)\) is contained in an
\(L^2(\overline P_{[t_0,T]})\)-ball of radius
\[
C\sqrt{(c_{\min}(R)^{-1}+1)r}
\]
around the fixed center \((h_0,\lambda M_R(2R)^{-1}\nabla_xh_0)\). Translation
by this fixed center and the scaling of the second component only change
covering numbers by constants, because \(M_R/(2R)=\mathcal O(1)\) and
\(\lambda\le1\). Thus the entropy of \(\mathcal W_R(r)\) is controlled by
the entropy of \(\mathcal V_{2R}\) (defined on \eqref{eq: localized-time-dependent-set}). Since the empirical \(L^2(\overline P_{n,[t_0,T]})\)-metric
is bounded by the supremum metric, Lemma~\ref{lem:C6-vector-entropy} implies
\[
\log\mathcal N\bigl(\varepsilon,\mathcal W_R(r),L^2(\overline P_{n,[t_0,T]})\bigr)
\le
C N\log\left(\frac{BWM_R}{\varepsilon}\right).
\]
We now apply Dudley's entropy integral to the localized class
\(\mathcal W_R(r)\). Using the entropy bound above and truncating the
integral at \(1/n\), for \(r\gtrsim n^{-2}\), gives
\[
\begin{aligned}
&\mathfrak R_n\left(
\left\{
\psi_h:
\begin{array}{l}
J_{\lambda,2R}(h)-J_{\lambda,2R}(f_{\lambda,2R})\\
\quad+J_{\lambda,2R}(h_0)-J_{\lambda,2R}(f_{\lambda,2R})\le r
\end{array}
\right\}
\right)\\
&\qquad\le
\frac{C}{n}
+
\frac{C}{\sqrt n}
\int_{1/n}^{C\sqrt{(c_{\min}(R)^{-1}+1)r}}
\sqrt{
N\log\left(\frac{BWM_R}{\varepsilon}\right)
}
\,d\varepsilon .
\end{aligned}
\]
Since \(\varepsilon\ge 1/n\) on the integration range, the logarithmic factor
is bounded by \(C\log(BWM_R n)\). The condition \(r\gtrsim n^{-2}\) lets the
lower-order \(1/n\) term be absorbed into the integral bound. Therefore
\[
\mathfrak R_n\left(
\left\{
\psi_h:
\begin{array}{l}
J_{\lambda,2R}(h)-J_{\lambda,2R}(f_{\lambda,2R})\\
\quad+J_{\lambda,2R}(h_0)-J_{\lambda,2R}(f_{\lambda,2R})\le r
\end{array}
\right\}
\right)
\le
\phi_R(r),
\]
where
\[
\phi_R(r):=
C\sqrt{
\frac{(c_{\min}(R)^{-1}+1)rN\log(BWM_R n)}{n}
}.
\]
The function \(\phi_R\) is sub-root and its critical radius satisfies
\[
r_R^\star
\le
C(c_{\min}(R)^{-1}+1)
\frac{N\log(BWM_R n)}{n}.
\]

We next verify the boundedness and variance conditions for the normalized
process. The envelope satisfies
\[
|\psi_h|
\le
C\left(
M_R+
\lambda\left(\frac{M_R}{2R}\right)^2
\right).
\]
Moreover, the pointwise Lipschitz bound above gives
\[
\overline P_{[t_0,T]}\psi_h^2
\le
C\left[
\|h-h_0\|_{L^2([t_0,T]\times B_{2R})}^2
+
\lambda^2\left(\frac{M_R}{2R}\right)^2
\|\nabla_x(h-h_0)\|_{L^2([t_0,T]\times B_{2R})}^2
\right].
\]
Using the strong-convexity localization bounds, we obtain
\[
\overline P_{[t_0,T]}\psi_h^2
\le
C\left(
c_{\min}(R)^{-1}
+
\lambda\left(\frac{M_R}{2R}\right)^2
\right)
\left[
J_{\lambda,2R}(h)-J_{\lambda,2R}(f_{\lambda,2R})
+
J_{\lambda,2R}(h_0)-J_{\lambda,2R}(f_{\lambda,2R})
\right].
\]

Let
\[
\widetilde\psi_h:=
\frac{\overline P_{[t_0,T]}\psi_h-\psi_h}
{
J_{\lambda,2R}(h)-J_{\lambda,2R}(f_{\lambda,2R})
+
J_{\lambda,2R}(h_0)-J_{\lambda,2R}(f_{\lambda,2R})
+r
}.
\]
Then
\begin{equation}
\label{eq:C3-normalized-envelope}
\|\widetilde\psi_h\|_\infty
\le
\frac{
C\left(
M_R+\lambda\left(\frac{M_R}{2R}\right)^2
\right)
}{r},
\end{equation}
and
\begin{equation}
\label{eq:C3-normalized-variance}
\overline P_{[t_0,T]}\widetilde\psi_h^2
\le
\frac{
C\left(
c_{\min}(R)^{-1}
+
\lambda\left(\frac{M_R}{2R}\right)^2
\right)
}{r}.
\end{equation}
By the peeling lemma applied to the sub-root bound above,
\[
\mathfrak R_n\left(
\left\{
\frac{\psi_h}{
J_{\lambda,2R}(h)-J_{\lambda,2R}(f_{\lambda,2R})
+
J_{\lambda,2R}(h_0)-J_{\lambda,2R}(f_{\lambda,2R})
+r
}:h\in\mathcal F_{2R}
\right\}
\right)
\le
\frac{4\phi_R(r)}{r}.
\]
Applying the Symmetrization Lemma (Lemma A.3 in
\cite{lu2021machine}) to the independent \(p\)-sample and \(q\)-sample empirical
processes separately, and then adding the two bounds, gives
\begin{equation}
\label{eq:C3-symmetrization-normalized}
\mathbb E \left[\sup_{h\in\mathcal F_{2R}}
\frac{\overline P_{[t_0,T]}\psi_h-\overline P_{n,[t_0,T]}\psi_h}
{
J_{\lambda,2R}(h)-J_{\lambda,2R}(f_{\lambda,2R})
+
J_{\lambda,2R}(h_0)-J_{\lambda,2R}(f_{\lambda,2R})
+r
}\right]
\le
\frac{8\phi_R(r)}{r}.
\end{equation}
Using \eqref{eq:C3-symmetrization-normalized},
\eqref{eq:C3-normalized-envelope}, and \eqref{eq:C3-normalized-variance},
Talagrand's inequality implies that, with
probability at least \(1-e^{-t}\),
\begin{eqnarray*}
&&\sup_{h\in\mathcal F_{2R}}
\frac{\overline P_{[t_0,T]}\psi_h-\overline P_{n,[t_0,T]}\psi_h}
{
J_{\lambda,2R}(h)-J_{\lambda,2R}(f_{\lambda,2R})
+
J_{\lambda,2R}(h_0)-J_{\lambda,2R}(f_{\lambda,2R})
+r
}\\
&\le&
\frac{16\phi_R(r)}{r}\\
&\quad&+
C\sqrt{
\frac{
\left(
c_{\min}(R)^{-1}
+
\lambda\left(\frac{M_R}{2R}\right)^2
\right)t
}{nr}
}\\
&\quad&+
C\frac{
\left(
M_R+\lambda\left(\frac{M_R}{2R}\right)^2
\right)t
}{nr}.
\end{eqnarray*}

Choose
\[
r_0:=
C\left(c_{\min}(R)^{-1}
+
M_R
+
\lambda\left(\frac{M_R}{2R}\right)^2\right)\frac{N\log(BWM_R n)+t}{n}
,
\]
with \(C\) sufficiently large. Since \(c_{\min}(R)^{-1}\ge1\), this choice
makes the right-hand side at most \(1/2\). Hence, on the same event, for all
\(h\in\mathcal F_{2R}\),
\[
(\overline P_{[t_0,T]}-\overline P_{n,[t_0,T]})\psi_h
\le
\frac12
\left[
J_{\lambda,2R}(h)-J_{\lambda,2R}(f_{\lambda,2R})
\right]
+
\frac12
\left[
J_{\lambda,2R}(h_0)-J_{\lambda,2R}(f_{\lambda,2R})
\right]
+
\frac12 r_0.
\]
Taking \(h=\widehat h_R\) and using empirical optimality,
\[
\overline P_{n,[t_0,T]}\psi_{\widehat h_R}
=\widehat J_{\lambda,\mathcal D_{\rm ext},2R}(\widehat h_R)
-\widehat J_{\lambda,\mathcal D_{\rm ext},2R}(h_0)
\le0,
\]
we obtain
\[
\begin{aligned}
J_{\lambda,2R}(\widehat h_R)-J_{\lambda,2R}(h_0)
&=
\overline P_{[t_0,T]}\psi_{\widehat h_R}\\
&=
(\overline P_{[t_0,T]}-\overline P_{n,[t_0,T]})\psi_{\widehat h_R}
+
\overline P_{n,[t_0,T]}\psi_{\widehat h_R}\\
&\le
(\overline P_{[t_0,T]}-\overline P_{n,[t_0,T]})\psi_{\widehat h_R}\\
&\le
\frac12\left[
J_{\lambda,2R}(\widehat h_R)-J_{\lambda,2R}(f_{\lambda,2R})
\right]
+
\frac12\left[
J_{\lambda,2R}(h_0)-J_{\lambda,2R}(f_{\lambda,2R})
\right]
+
\frac12r_0.
\end{aligned}
\]
Rearranging yields
\[
J_{\lambda,2R}(\widehat h_R)-J_{\lambda,2R}(f_{\lambda,2R})
\le
3\left[
J_{\lambda,2R}(h_0)-J_{\lambda,2R}(f_{\lambda,2R})
\right]+r_0.
\]
It remains to bound \(r_0\).

Finally, \(M_R/(2R)=\mathcal{O}(1)\), \(M_R=C_M(1+R)\), and
\(c_{\min}(R)^{-1}\lesssim e^{CR}\). Thus
\[
c_{\min}(R)^{-1}
+
M_R
+
\lambda\left(\frac{M_R}{2R}\right)^2
\le
C e^{CR}(1+R)^K.
\]
Together with the definition of \(r_0\), this proves
\eqref{eq:C6_comparator_oracle}.
\end{proof}

\subsection{Proof of Theorem~\ref{thm: unbounded_convergence_nohop}}

Recall the rate exponent
\[
\kappa=
\frac{s-1}{(d+1)+2s-2}.
\]
Also recall the cutoff estimator \(\widetilde f^{(R)}\) from
\eqref{eq:global_cutoff_estimator}, the pointwise loss density
\(\mathcal L\) from \eqref{eq:norm}, and, for \(x\in B_{2R}\), the
pullback notation
\[
(\pi_R f)(t,x)=f\left(t,\frac{x}{2R}\right).
\]
Recall the global error decomposition from \eqref{eq:core_tail_decomp}; in
this proof we write
\begin{equation}
\label{eq:C3_global_error_decomposition}
\mathcal E(\widetilde f^{(R)})
=
\int_{t_0}^T
\|\widetilde f_t^{(R)}-f_t^\star\|_{H^1(\mu_t)}^2\,dt
=
{\rm Main}(R)+{\rm Tail}(R),
\end{equation}
where
\[
{\rm Main}(R)
=
\int_{t_0}^T\int_{B_R}
\mathcal L(\widetilde f^{(R)},f^\star)\,d\mu_t\,dt,
\]
and
\[
{\rm Tail}(R)
=
\int_{t_0}^T\int_{B_R^c}
\mathcal L(\widetilde f^{(R)},f^\star)\,d\mu_t\,dt.
\]

\proofstep{Step 1: Reduction of the main error to \([t_0,T]\times B_{2R}\).}
Recall that \(\chi_R\equiv1\) on \(B_R\), and by definition
\(\operatorname{Proj}_{2R}(x)=x\) for \(x\in B_R\). Hence
the definition of \(\widetilde f^{(R)}\) in
\eqref{eq:global_cutoff_estimator} gives
\[
\widetilde f^{(R)}=\pi_R\widehat f_R
\qquad\text{on }[t_0,T]\times B_R.
\]
Therefore,
\begin{equation}
\label{eq:C3_main_reduction}
{\rm Main}(R)
=
\int_{t_0}^T\int_{B_R}
\mathcal L(\pi_R\widehat f_R,f^\star)\,d\mu_t\,dt
\le
\int_{t_0}^T\int_{B_{2R}}
\mathcal L(\pi_R\widehat f_R,f^\star)\,d\mu_t\,dt.
\end{equation}
\proofstep{Step 2: The truncated ERM event.}
Define the truncation event
\[
\mathcal A_R
=
\left\{
\max_{1\le i\le n}\|X_i^p\|
\vee
\max_{1\le i\le n}\|X_i^q\|
\le 2R
\right\}.
\]
By Lemma~\ref{lem:global-linear-growth}, the marginal distributions
\(p_t\) and \(q_t\) have uniformly sub-Gaussian tails. Therefore there exist
constants \(C,c>0\), independent of \(n\) and \(R\), such that
\[
\sup_{t\in[t_0,T]}
\max\Big\{p_t(\{x:\|x\|>u\}),q_t(\{x:\|x\|>u\})\Big\}
\le
C\exp\Big(-c(u-R_1)_+^2\Big).
\]
Write \(t_{1:n}^p=(t_1^p,\ldots,t_n^p)\) and
\(t_{1:n}^q=(t_1^q,\ldots,t_n^q)\).
Conditioning on these sampled times and applying a union bound over the
\(p\)-sample and \(q\)-sample blocks,
\[
\mathbb P(\mathcal A_R^c\mid t_{1:n}^p,t_{1:n}^q)
\le
\sum_{i=1}^n
\mathbb P(\|X_i^p\|>2R\mid t_i^p)
+
\sum_{i=1}^n
\mathbb P(\|X_i^q\|>2R\mid t_i^q)
\le
2nC\exp\Big(-c(2R-R_1)_+^2\Big).
\]
For \(R\ge R_1\), \((2R-R_1)_+\ge R\). Hence
\begin{equation}
\label{eq:C3_truncation_event_tail_prob}
\mathbb P(\mathcal A_R^c)
\le
2nC e^{-cR^2}.
\end{equation}

On \(\mathcal A_R\), we have \(\operatorname{Proj}_{2R}(X_i^p)=X_i^p\) and
\(\operatorname{Proj}_{2R}(X_i^q)=X_i^q\).
\[
\frac{\lambda}{4R^2}
\|\nabla_{\bar x}f(t_i^p,X_i^p/(2R))\|^2
=\lambda\|\nabla_x(\pi_R f)(t_i^p,X_i^p)\|^2,
\]
and the same identity holds for the \(q\)-samples.
Recalling the projected empirical objective
\(\widehat J^{\rm proj}_{\lambda,\mathcal D_{\rm ext},2R}\) from
\eqref{eq:empirical_objective_proj}, minimizing it over
\(f\in\mathcal F_R\) is therefore equivalent, on \(\mathcal A_R\), to
minimizing the following truncated empirical objective over
\(h\in\mathcal F_{2R}\):
\begin{equation}
\label{eq:C3_truncated_empirical_objective}
\begin{aligned}
\widehat J_{\lambda,\mathcal D_{\rm ext},2R}(h)
&=
\frac{T-t_0}{2n}\sum_{i=1}^n
\left[
\ell_{\rm CE}(0,h(t_i^p,X_i^p))
+
\lambda\|\nabla_xh(t_i^p,X_i^p)\|^2
\right]\mathbf 1\{X_i^p\in B_{2R}\}
\\
&\quad+
\frac{T-t_0}{2n}\sum_{i=1}^n
\left[
\ell_{\rm CE}(1,h(t_i^q,X_i^q))
+
\lambda\|\nabla_xh(t_i^q,X_i^q)\|^2
\right]\mathbf 1\{X_i^q\in B_{2R}\}.
\end{aligned}
\end{equation}
Consequently, on \(\mathcal A_R\), \(\pi_R\widehat f_R\)
is an empirical minimizer of
\eqref{eq:C3_truncated_empirical_objective} over \(\mathcal F_{2R}\).

\proofstep{Step 3: Comparator-centered oracle inequality on \([t_0,T]\times B_{2R}\).}
Recall from Lemma~\ref{lem:C2-existence-uniqueness} that
\(f_{\lambda,2R}\) is the minimizer of \(J_{\lambda,2R}\) over
\(\widetilde{\mathcal H}_{M_R,2R}\).
By Lemma~\ref{lem:C4-rescaled-approx}, there exists a deterministic
comparator
$
f_{0,R}\in\mathcal F_{2R}
$
such that
\begin{equation}
\label{eq:C3_f0R_approx}
\|f_{0,R}-f^\star\|_{L^2([t_0,T];H^1(B_{2R},\mu_t))}^2
\le
C(1+R)^K N^{-\frac{2(s-1)}{d+1}}.
\end{equation}

Applying Lemma~\ref{lem:C3-aniso-local-complexity} with
\(h_0=f_{0,R}\) and \(t>0\), we obtain an event
\(\mathcal A_{R,t}\), with
\[
\mathbb P(\mathcal A_{R,t}^c)\le e^{-t},
\]
on which
\begin{equation}
\label{eq:C3_oracle_gap_new}
J_{\lambda,2R}(\pi_R\widehat f_R)
-
J_{\lambda,2R}(f_{\lambda,2R})
\le
3\left[
J_{\lambda,2R}(f_{0,R})
-
J_{\lambda,2R}(f_{\lambda,2R})
\right]
+
C e^{CR}(1+R)^K
\frac{N\log(BWM_Rn)+t}{n}.
\end{equation}

\proofstep{Step 4: Bound the comparator objective gap.}
We bound
\[
J_{\lambda,2R}(f_{0,R})
-
J_{\lambda,2R}(f_{\lambda,2R})
\]
by comparing both terms to \(f^\star\):
\[
\begin{aligned}
J_{\lambda,2R}(f_{0,R})
-
J_{\lambda,2R}(f_{\lambda,2R})
&\le
J_{\lambda,2R}(f_{0,R})-J_{\lambda,2R}(f^\star)\\
&\quad+
J_{\lambda,2R}(f^\star)-J_{\lambda,2R}(f_{\lambda,2R}).
\end{aligned}
\]

First, since \(f^\star\) is the pointwise Bayes logit, it minimizes the
population cross-entropy risk. The logistic loss is \(1/4\)-smooth in the
logit; hence
\[
L_{{\rm CE},2R}(f_{0,R})-L_{{\rm CE},2R}(f^\star)
\le
C\|f_{0,R}-f^\star\|_{L^2([t_0,T]\times B_{2R})}^2.
\]
For the Sobolev penalty, using the pulled-back gradient envelope
\[
\|\nabla_x f_{0,R}\|_{L^\infty([t_0,T]\times B_{2R})}
\le \frac{M_R}{2R}\lesssim 1
\]
and the uniform spatial-gradient bound for \(f^\star\) from
Lemma~\ref{lem:global-linear-growth}, we get
\[
\begin{aligned}
&\lambda
\left|
\|\nabla_x f_{0,R}\|_{L^2([t_0,T]\times B_{2R})}^2
-
\|\nabla_x f^\star\|_{L^2([t_0,T]\times B_{2R})}^2
\right|\\
&\qquad\le
C\lambda
\|\nabla_x(f_{0,R}-f^\star)\|_{L^2([t_0,T]\times B_{2R})}
\le
C\left[
\|f_{0,R}-f^\star\|_{L^2([t_0,T];H^1(B_{2R},\mu_t))}^2+\lambda^2
\right].
\end{aligned}
\]
Therefore, by \eqref{eq:C3_f0R_approx},
\begin{equation}
\label{eq:C3_J_f0_fstar}
J_{\lambda,2R}(f_{0,R})-J_{\lambda,2R}(f^\star)
\le
C(1+R)^KN^{-\frac{2(s-1)}{d+1}}
+
C\lambda^2.
\end{equation}

Second, since \(f^\star\) minimizes the cross-entropy risk,
\[
L_{{\rm CE},2R}(f^\star)
-
L_{{\rm CE},2R}(f_{\lambda,2R})
\le0.
\]
\[
J_{\lambda,2R}(f^\star)-J_{\lambda,2R}(f_{\lambda,2R})
\le
\lambda\left(
\|\nabla_x f^\star\|_{L^2([t_0,T]\times B_{2R})}^2
-
\|\nabla_x f_{\lambda,2R}\|_{L^2([t_0,T]\times B_{2R})}^2
\right).
\]
As in the proof of Lemma~\ref{lem:C3-bias}, Green's identity gives the
bound
\[
J_{\lambda,2R}(f^\star)-J_{\lambda,2R}(f_{\lambda,2R})
\le
C\left[
\beta(R)\lambda^2
+
c_{\min}(R)^{-1}\lambda\mathfrak b_R
+
\lambda\mathfrak b_R
\right].
\]
Combining this with \eqref{eq:C3_J_f0_fstar} yields
\begin{equation}
\label{eq:C3_comparator_gap_bound}
\begin{aligned}
&J_{\lambda,2R}(f_{0,R})
-
J_{\lambda,2R}(f_{\lambda,2R})\\
&\qquad\le
C\left[
(1+R)^KN^{-\frac{2(s-1)}{d+1}}
+
\lambda^2
+
\beta(R)\lambda^2
+
c_{\min}(R)^{-1}\lambda\mathfrak b_R
+
\lambda\mathfrak b_R
\right].
\end{aligned}
\end{equation}

\proofstep{Step 5: Convert the oracle bound to
\(L^2([t_0,T];H^1(B_{2R},\mu_t))\)-error.}
Here \(L^2([t_0,T];H^1(B_{2R},\mu_t))\) is the spatial energy norm defined in
\eqref{eq:weighted-energy-norm-2R}; it integrates over time but contains no
time derivative.
The same variational inequality for the constrained minimizer
\(f_{\lambda,2R}\), combined with strong convexity of \(J_{\lambda,2R}\) on
\(\widetilde{\mathcal H}_{M_R,2R}\), gives, for every
\(f\in\mathcal F_{2R}\),
\[
\begin{aligned}
J_{\lambda,2R}(f)-J_{\lambda,2R}(f_{\lambda,2R})
&\ge
\frac{c_{\min}(R)}{2}
\|f-f_{\lambda,2R}\|_{L^2([t_0,T];L^2(B_{2R},\mu_t))}^2 \\
&\quad+
\lambda
\|\nabla_x(f-f_{\lambda,2R})\|_{L^2([t_0,T];L^2(B_{2R},\mu_t))}^2 .
\end{aligned}
\]
Since
\[
\begin{aligned}
\|f-f_{\lambda,2R}\|_{L^2([t_0,T];H^1(B_{2R},\mu_t))}^2
&=
\|f-f_{\lambda,2R}\|_{L^2([t_0,T];L^2(B_{2R},\mu_t))}^2\\
&\quad+
\|\nabla_x(f-f_{\lambda,2R})\|_{L^2([t_0,T];L^2(B_{2R},\mu_t))}^2,
\end{aligned}
\]
we obtain
\[
\|f-f_{\lambda,2R}\|_{L^2([t_0,T];H^1(B_{2R},\mu_t))}^2
\le
C\bigl(c_{\min}(R)^{-1}\vee\lambda^{-1}\bigr)
\left[
J_{\lambda,2R}(f)-J_{\lambda,2R}(f_{\lambda,2R})
\right].
\]
For the final choice \(R=A\sqrt{\log n}\) and
\(\lambda\asymp n^{-\kappa}\), we have
\[
c_{\min}(R)^{-1}\lesssim e^{CR}=n^{o(1)},
\]
so \(\lambda^{-1}\) dominates \(c_{\min}(R)^{-1}\) for all sufficiently large \(n\).
Thus, on \(\mathcal A_{R,t}\cap\mathcal A_R\), using
\eqref{eq:C3_oracle_gap_new} and
\eqref{eq:C3_comparator_gap_bound},

\begin{eqnarray*}
\|\pi_R\widehat f_R-f_{\lambda,2R}\|_{L^2([t_0,T];H^1(B_{2R},\mu_t))}^2&\leq&
\frac{C}{\lambda}
\Big[
e^{CR}(1+R)^K\frac{N\log(BWM_Rn)+t}{n}
+
(1+R)^KN^{-\frac{2(s-1)}{d+1}}
\\
&&\qquad \qquad  +
\lambda^2+
\beta(R)\lambda^2
+
c_{\min}(R)^{-1}\lambda\mathfrak b_R
+
\lambda\mathfrak b_R
\Big].
\end{eqnarray*}

Using Lemma~\ref{lem:C3-bias} to pass from \(f_{\lambda,2R}\) to
\(f^\star\), we obtain
\begin{eqnarray}
\label{eq:C3_main_cylinder_bound_new}
\int_{t_0}^T\int_{B_{2R}}
\mathcal L(\pi_R\widehat f_R,f^\star)\,d\mu_t\,dt&\le&
C\Big[
\frac{e^{CR}(1+R)^K}{\lambda}
\frac{N\log(BWM_Rn)+t}{n}
+
\frac{(1+R)^K}{\lambda}
N^{-\frac{2(s-1)}{d+1}}\nonumber\\
&&\qquad +
\lambda
+
\beta(R)\lambda
+
(c_{\min}(R)^{-1}+1)\mathfrak b_R
\Big].
\end{eqnarray}
By \eqref{eq:C3_main_reduction}, the same bound holds for
\({\rm Main}(R)\).

\proofstep{Step 6: Tail error.}
On the annulus \(B_{2R}\setminus B_R\),
\[
\widetilde f^{(R)}(t,x)
=
\chi_R(x)(\pi_R\widehat f_R)(t,x),
\]
and therefore
\[
\nabla_x\widetilde f^{(R)}(t,x)
=
(\nabla_x\chi_R)(x)(\pi_R\widehat f_R)(t,x)
+
\chi_R(x)\nabla_x(\pi_R\widehat f_R)(t,x).
\]
Since
\[
|\chi_R|\le1,
\qquad
\|\nabla_x\chi_R\|_\infty\lesssim R^{-1},
\]
and
\[
\|\pi_R\widehat f_R\|_{L^\infty([t_0,T]\times B_{2R})}\le M_R,
\qquad
\|\nabla_x(\pi_R\widehat f_R)\|_{L^\infty([t_0,T]\times B_{2R})}
\le \frac{M_R}{2R}\lesssim1,
\]
we have
\[
|\widetilde f^{(R)}(t,x)|
+
\|\nabla_x\widetilde f^{(R)}(t,x)\|
\le
C(1+R),
\qquad
x\in B_{2R}\setminus B_R.
\]
Together with the value and gradient bounds for \(f^\star\) from
Lemma~\ref{lem:global-linear-growth}, this gives
\[
\mathcal L(\widetilde f^{(R)},f^\star)(t,x)
\le
C(1+R)^K,
\qquad
x\in B_{2R}\setminus B_R.
\]
Therefore, using the tail bound in Lemma~\ref{lem:global-linear-growth},
there exist constants \(C,K,c_1>0\) such that
\begin{equation}
\label{eq:C3_annulus_tail_bound}
\int_{t_0}^T\int_{B_{2R}\setminus B_R}
\mathcal L(\widetilde f^{(R)},f^\star)\,d\mu_t\,dt
\le
C(1+R)^K e^{-c_1(R-R_1)_+^2}.
\end{equation}

On \(B_{2R}^c\), \(\widetilde f^{(R)}=0\). Hence
\[
\mathcal L(\widetilde f^{(R)},f^\star)
=
|f^\star|^2+\|\nabla_xf^\star\|^2.
\]
By Lemma~\ref{lem:global-linear-growth}, this is bounded by
\(C(1+\|x\|)^K\). The tail-moment bound in
Lemma~\ref{lem:global-linear-growth} gives constants \(C,K,c_2>0\)
such that
\begin{equation}
\label{eq:C3_outer_tail_bound}
\int_{t_0}^T\int_{B_{2R}^c}
\mathcal L(\widetilde f^{(R)},f^\star)\,d\mu_t\,dt
\le
C(1+R)^K e^{-c_2(2R-R_1)_+^2}.
\end{equation}

Combining \eqref{eq:C3_annulus_tail_bound} and
\eqref{eq:C3_outer_tail_bound}, and taking \(R\ge2R_1\), we have
\[
(R-R_1)_+\ge R/2,
\qquad
(2R-R_1)_+\ge R.
\]
Thus, setting
\[
c_3=\min\{c_1/4,c_2\}>0,
\]
and enlarging \(C,K\) if necessary, we obtain
\begin{equation}
\label{eq:C3_tail_bound}
{\rm Tail}(R)
\le
C(1+R)^K e^{-c_3R^2}.
\end{equation}

\proofstep{Step 7: Choosing \(N,\lambda,R\).}
Choose
\begin{equation}
\label{eq:C3_N_lambda_choice}
N\asymp n^{\frac{d+1}{d+1+2s-2}},
\qquad
\lambda\asymp n^{-\kappa},
\qquad
\kappa=\frac{s-1}{d+1+2s-2}.
\end{equation}
Then
\begin{equation}
\label{eq:C3_rate_balancing}
\frac{N\log(BWM_Rn)}{n}
\lesssim
\lambda^2\log n,
\qquad
N^{-\frac{2(s-1)}{d+1}}
\lesssim
\lambda^2.
\end{equation}
By Lemmas~\ref{lem:C3-bias} and
\ref{lem:C3-aniso-local-complexity},
\[
\beta(R)\le C e^{CR}(1+R)^K,
\qquad
c_{\min}(R)^{-1}\le C e^{CR}.
\]
Substituting these estimates into
\eqref{eq:C3_main_cylinder_bound_new} and using
\eqref{eq:C3_rate_balancing}, with \(t=2\log n\), gives
\begin{equation}
\label{eq:C3_main_rate_after_tuning}
{\rm Main}(R)
\le
C e^{CR}(1+R)^K(\log n)n^{-\kappa}
+
C e^{CR}(1+R)^K\mathfrak b_R.
\end{equation}
Since
\[
\mathfrak b_R\le C(1+R)^K e^{-cR^2},
\]
the boundary term is exponentially small for the choice of \(R\) below.
Combining \eqref{eq:C3_main_rate_after_tuning} with
\eqref{eq:C3_tail_bound}, we obtain
\begin{equation}
\label{eq:C3_pre_final_rate}
\mathcal E(\widetilde f^{(R)})
\le
C e^{CR}(1+R)^K(\log n)n^{-\kappa}
+
C e^{CR}(1+R)^K e^{-cR^2}.
\end{equation}

Finally choose
\[
R=A\sqrt{\log n},
\]
with \(A>0\) sufficiently large so that \(R\ge2R_1\) for all sufficiently
large \(n\), the truncation-event failure probability in
\eqref{eq:C3_truncation_event_tail_prob} is at most \(n^{-2}\), and the exponential
remainder in \eqref{eq:C3_pre_final_rate} is negligible relative to the
main term. On the event
\[
\mathcal A_{R,2\log n}\cap\mathcal A_R,
\]
whose probability is at least \(1-2n^{-2}\), and hence at least
\(1-3n^{-2}\), we have
\begin{equation}
\label{eq:C3_final_n_o1_rate}
\mathcal E(\widetilde f^{(R)})
\le
C e^{C\sqrt{\log n}}(\log n)^K n^{-\kappa}.
\end{equation}
Since
\[
e^{C\sqrt{\log n}}(\log n)^K=n^{o(1)},
\]
it follows that for every \(\varepsilon>0\), there exists
\(C_\varepsilon<\infty\) such that
\begin{equation}
\label{eq:C3_final_eps_rate}
\mathcal E(\widetilde f^{(R)})
\le
C_\varepsilon n^{-\kappa+\varepsilon}.
\end{equation}
This completes the proof.
\qed

\section{Supplementary Experiment Results}
\label{sec:supplementary-experiments}
\begin{figure}[htbp]
    \centering
    \includegraphics[width=0.95\linewidth]{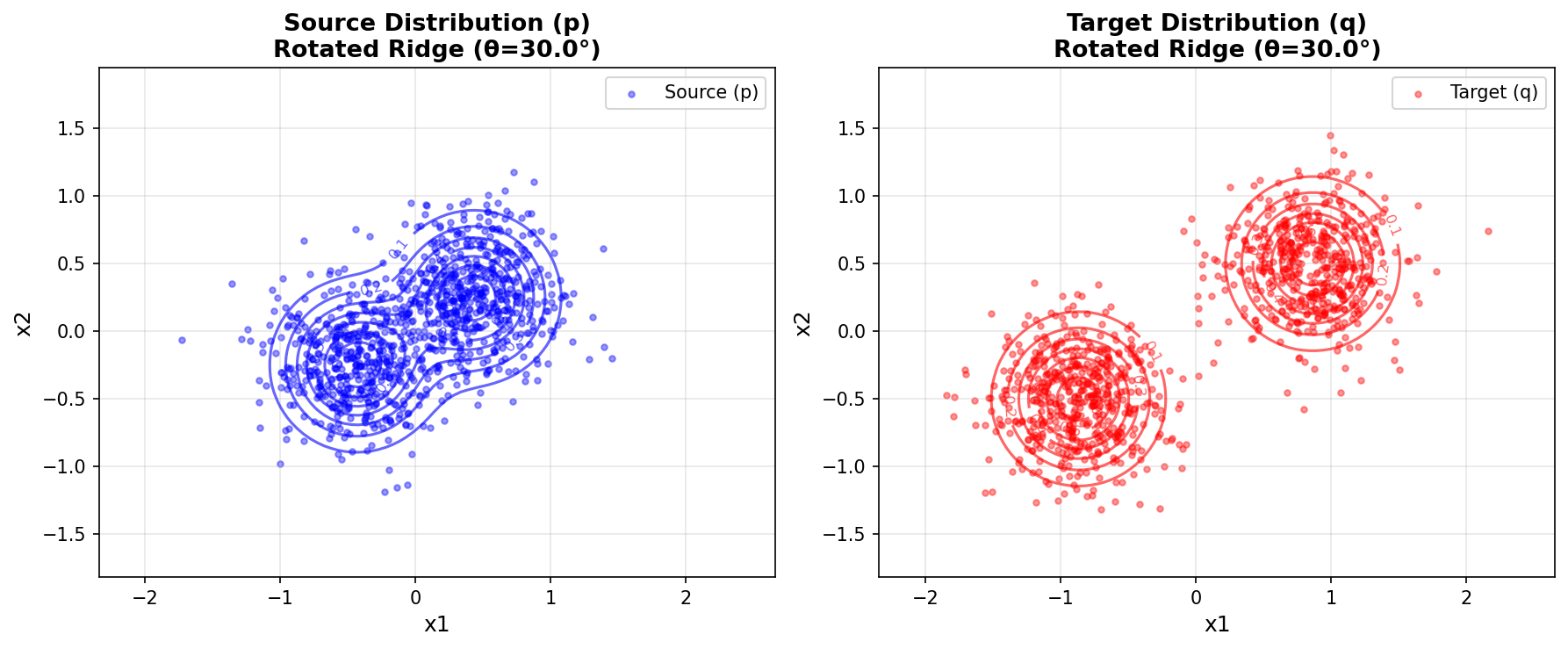}\vspace{-0.6em}

    \includegraphics[width=0.95\linewidth]{images/distributions_orthogonal_gmm.png}\vspace{-0.6em}

    \includegraphics[width=0.95\linewidth]{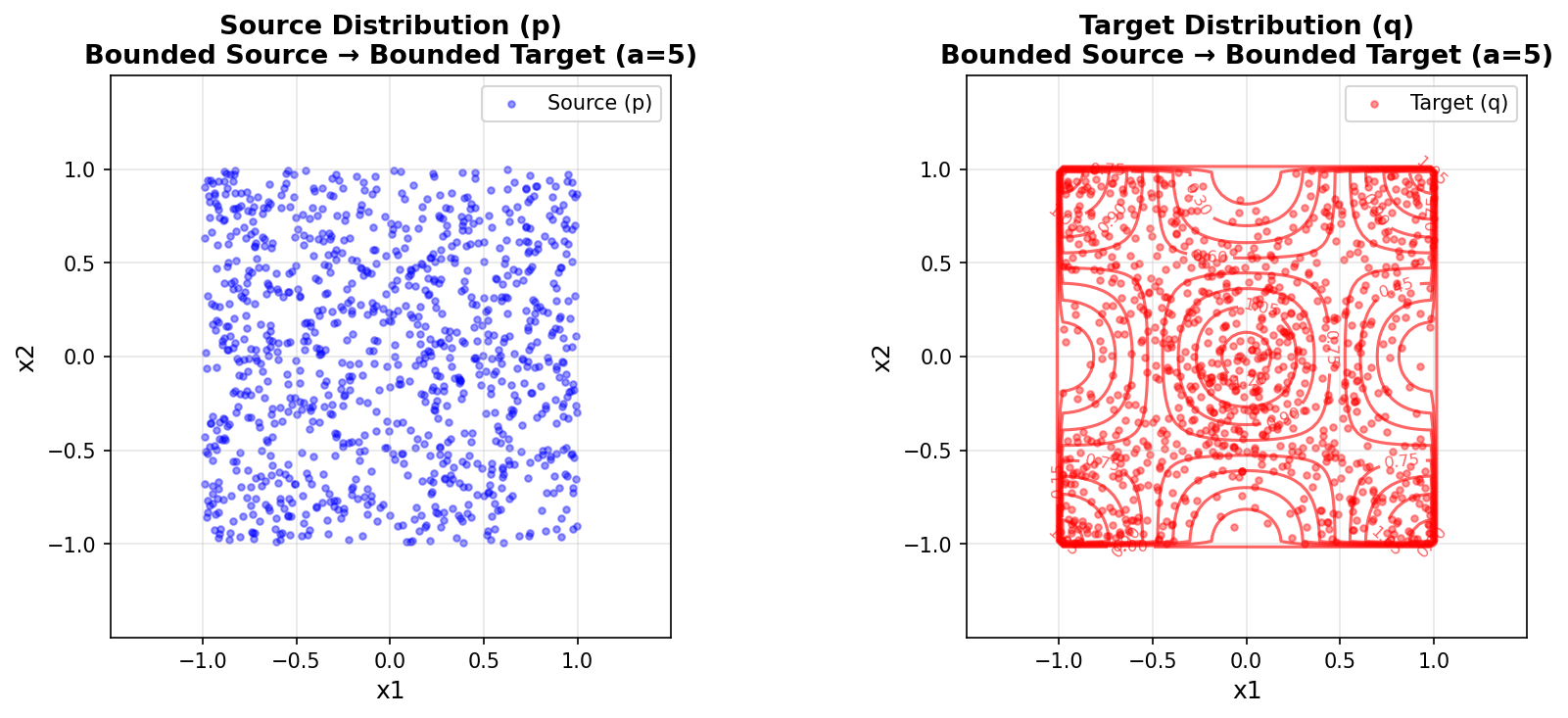}

    \caption{
    Three simulation distributions (left column: source $p$, right column: target $q$).
    \textbf{Top:} Same-line GMM with rotated ridge structure.
    \textbf{Middle:} Orthogonal GMM with misaligned source and target components.
    \textbf{Bottom:} Bounded source-to-target distribution with compact support.
    }
    \label{fig:three-subfig}
\end{figure}
\subsection{Synthetic Experiments}
We provide details for synthetic source--target pairs here.
\label{subsec:Synthetic source--target pairs}
\begin{itemize}
  \item \textbf{Rotated ridge.}
  In the \((t,s)\)-coordinates the source and target are mixtures of
  two Gaussians with identical covariance but different offsets along
  the \(t\)-axis:
  the source has components centered at \(\pm a\) and the target at
  \(\pm b\) (with \(b>a\)), while the \(s\)-coordinate is shared.
  The pair \((t,s)\) is then rotated by an angle \(\mathcal{P}=30^\circ\)
  into the observed \((x_1,x_2)\)-space.
  This creates an anisotropic ``ridge'' structure where \(p\) and \(q\)
  differ primarily along a rotated onedimensional direction.

  \item \textbf{Orthogonal GMM.}
  Both \(p\) and \(q\) are symmetric two-component Gaussian mixtures
  with spherical covariance.
  The source means are \(\pm \mu_S\), the target means are
  \(\pm \mu_T\), and we choose \(\mu_S^\top \mu_T = 0\),
  i.e.\ the source and target mixture axes are orthogonal.
  This case is adopted in \citet{ouyang2024transfer} and probes how well the estimator adapts when the main
  directions of variation in the source and target are misaligned.

  \item \textbf{Bounded source to bounded target.}
  The source \(p\) is the uniform distribution on the square
  \([-1,1]^2\).
  The target \(q\) is obtained by smoothly reweighting \(p\) with a
  bounded, oscillatory density ratio
  \(r(x) = \tfrac{q(x)}{p(x)}
      = c \bigl(1 + b \cos(\pi x_1)\cos(\pi x_2)\bigr)\),
  where the parameters are chosen so that \(r(x)\in[1/a,a]\) with
  \(a=5\).
  This yields a non-Gaussian, compactly supported setting with a
  spatially varying but uniformly bounded density ratio.
\end{itemize}
Together, these three examples cover rotated lowdimensional shifts,
orthogonal mixture structure, and bounded non-Gaussian modulation,
and hence provide a diverse testbed for gradient-of-ratio estimation.
\subsection{WGF experiments}
\paragraph{Setting details.}
When a large number of unlabeled target samples $x_q^{(k)}$ are available, \cite{courty2016optimal} proposed learning an optimal transport map to align the joint source and target distributions ($P_{XZ}$ and $Q_{XZ}$), and subsequently training a classifier on the aligned source samples. Building on this framework, \cite{liu2023minimizing} utilized WGF for sample alignment. Their method evolves particles $x_t$ to minimize $\text{KL}[p_t, q]$ between the evolving particle-label distribution $p_t$ (initialized with source samples $\mathcal{D}_{p_t}:=\{(x_t^{(i)}, y_p^{(i)})\}_{i=1}^{n_p}$) and the target density $q$. Upon convergence after $T$ iterations, a classifier is trained on the transported source samples $\{(x_T, y_q)\}$ to predict target labels.

Implementing WGF requires estimating the score difference. While \cite{liu2023minimizing} employed a local method known as ``local linear interpolation'' (LL) for this estimation, our proposed approach utilizes a classification-based method with Sobolev regularization to estimate the score difference globally.

\subsection{ECG experiments}
\label{subsec: ecg results detail}
\paragraph{Details for two datasets.}
PTB-XL contains 21,837 clinical 12-lead ECG recordings of 10-second duration from 18,885 patients, annotated with 71 diagnostic statements in a multi-label setting. The ICBEB2018 dataset consists of 6,877 12-lead ECG recordings with durations ranging from 6 to 60 seconds, each labeled into one of nine diagnostic classes, which form a subset of the PTB-XL label space. We randomly select $10\%$ of the ICBEB2018 samples as the limited target dataset via stratified sampling to preserve label proportions. All ECG signals are resampled to a frequency of 100 Hz.
\paragraph{Synthetic quality and training speed evaluation for diffusion models.}
We evaluate the quality of the generated samples of the target task using the standard Fréchet Inception Distance (FID) metric~\cite{heusel2017gans}. The FID score measures the Wasserstein-2 distance between the distributions of real and synthetic data within the feature space of an xresnet1d50 classifier~\cite{strodthoff2020deep} pre-trained on the target task.

The quantitative results are summarized in Table~\ref{tab:ecg_quality}. Our proposed TGDP-SOB not only achieves the best FID of 8.097 but also significantly reduces model parameters to 2.8M and training time to 30 minutes, offering a superior balance between generation quality and computational efficiency.

\begin{table}[htbp]
\centering
\caption{The effectiveness of Soblev Penalization on ECG benchmark under synthetic quality.}\label{tab:ecg_quality}
\begin{tabular}{l|c|c|c|c}
\hline
\textbf{Method}
& \textbf{FID}
& \textbf{Parameters}
& \textbf{Training Time} \\
\hline
Vanilla Diffusion       & 11.171 & 50.2M & 1h \\
Finetune Generator      & 8.415 & 50.2M & 40min \\
TGDP          & 8.100 & \textbf{2.8M} &
\textbf{30min} \\
\textbf{TGDP-SOB}  &\textbf{8.097} & \textbf{2.8M} &
\textbf{30min} \\
\hline
\end{tabular}
\end{table}
\end{document}